\documentclass{article}

\PassOptionsToPackage{numbers}{natbib}

\usepackage[preprint]{neurips_2026}

\usepackage[utf8]{inputenc} 
\usepackage[T1]{fontenc}    
\usepackage[colorlinks=true, linkcolor=blue, citecolor=blue, urlcolor=blue]{hyperref}
\usepackage{url}            
\usepackage{booktabs}       
\usepackage{amsfonts}       
\usepackage{nicefrac}       
\usepackage{microtype}      
\usepackage{xcolor}         

\usepackage{subcaption}
\usepackage{multirow}
\usepackage{amsmath}
\usepackage{amssymb}
\usepackage{mathtools}
\usepackage{amsthm}
\usepackage{thmtools}
\usepackage{thm-restate}
\usepackage[capitalize,noabbrev]{cleveref}
\usepackage{xfrac}

\theoremstyle{plain}
\newtheorem{theorem}{Theorem}[section]
\newtheorem{proposition}[theorem]{Proposition}

\newtheorem{corollary}[theorem]{Corollary}
\theoremstyle{definition}
\newtheorem{definition}[theorem]{Definition}

\theoremstyle{remark}
\newtheorem{remark}[theorem]{Remark}

\expandafter\newcommand\csname tinyx\endcsname[1]
  {{\fontsize{6.5pt}{8pt}\selectfont #1}}

\title{Decision-Aligned Evaluation of\\ Uncertainty Quantification}

\author{%
  Annika Schneider\textsuperscript{1,2,3,4}\
  Tommy Rochussen\textsuperscript{1,2,3}\
  Joshua Stiller\textsuperscript{3,5}\
  Vincent Fortuin\textsuperscript{2,3,4,6}\\
  \textsuperscript{1}Technical University of Munich\quad
  \textsuperscript{2}Helmholtz AI\quad
  \textsuperscript{3}MCML\\
  \textsuperscript{4}Konrad Zuse School of Excellence in Reliable AI\\
  \textsuperscript{5}LMU Munich\quad
  \textsuperscript{6}University of Technology Nuremberg
}

\begin{document}

\maketitle

\begin{abstract}
  Uncertainty estimates in machine learning are typically evaluated using generic metrics such as the negative log-likelihood and expected calibration error, yet good performance on such metrics does not necessarily imply high utility in downstream decisions.
  We introduce \emph{decision-alignment}, a criterion that reveals which evaluation metrics meaningfully align with downstream utilities.
  Applying this framework, we show that many widely used uncertainty metrics are either misaligned with common decision problems or encode pathological prior beliefs about the downstream task.
  We then propose \emph{prior-weighted utility metrics}, a special class of proper scoring rules that provides decision-aligned uncertainty evaluation.
  Across benchmark experiments and real-world case studies, our metrics consistently align with realized decision utility, while conventional metrics do not.
  Our results surface flaws in the current UQ evaluation protocol and offer a principled extension of existing metrics toward decision-relevant UQ evaluation.
\end{abstract}

\section{Introduction}\label{sect:intro}

    Uncertainty quantification (UQ) in machine learning (ML) is key for reliable decision-making.
    Probabilistic ML research aims to develop high-quality UQ methods, facilitating ML usage in~safety-critical domains.
    How UQ quality is measured directly shapes which methods are developed,~published, and deployed.
    Since the ``true uncertainty'' cannot be observed, UQ evaluation relies on surrogate metrics such as negative log-likelihood (NLL) or expected calibration error (ECE) \cite{ovadia_can_2019,vadera_ursabench_2022,basora_benchmark_2025}.
    However, good performance on such metrics does not necessarily imply high utility in downstream decisions \cite{vickers_decision_2006,dusenberry_analyzing_2020,ferrer_evaluating_2025}.
    We argue that \textbf{UQ evaluation should reflect downstream decision utility directly}.

    In this paper, we take UQ evaluation as our object of study.
    We focus on general-purpose evaluation where no specific downstream task is fixed.
    Benchmarks guide progress in ML~\cite{hardt_patterns_2021}, so by restricting ourselves to generic measures such as the NLL and ECE at this stage, we risk leaving behind methods that do not excel on such metrics but are actually more useful in downstream applications.

    In this work, \textbf{we aim to improve the understanding of common UQ evaluation metrics and align UQ evaluation with decision-making}.
    We make the following contributions:
    \begin{enumerate}
        \item We introduce \emph{decision-alignment} as a formal criterion to show that common UQ evaluation metrics either encode pathological priors or no coherent decision belief at all (\cref{sect:theory}).
        \item Based on these findings, we propose \emph{prior-weighted utility metrics} that capture a model's value for downstream utilities (\cref{sect:method}).
        \item We demonstrate in classification and regression experiments that only prior-weighted utility metrics reliably align with real downstream utilities (\cref{sect:experiments}).\footnote{
            Our code is available at \url{https://github.com/fortuinlab/prior-weighted-utilities}.
        }
    \end{enumerate}

\begin{table}[t]
    \caption{
        Summary of the main results from \cref{sect:theory} for classification \emph{(left)} and regression \emph{(right)}.
        If the metric is decision-aligned, we state the implicit prior $\pi$.
        Else, we write \emph{not aligned}. All UQ metrics we consider are either not decision-aligned or encode pathological prior beliefs.
    }
    \vspace{0.7em}
    \label{tab:theory_summary}
    \begin{subtable}{0.49\linewidth}
        \raggedright
        \begin{tabular}{lll}
            \toprule
            Metric & Binary decision & Top-$k$ select. \\
            \midrule
            NLL & $\pi(c) = \tfrac{1}{c(1-c)} \phantom{xx}$ & \multirow{2}{*}{not aligned} \\
            BS & $\pi(c) = 2$ &  \\
            \midrule
            \multirow{2}{*}{Acc} & \multirow{2}{*}{$\pi(c) = 2\delta_{0.5}(c)$} & $\pi(k) = \delta_{n}(k)$, \\
            & & \tiny{not aligned for $k < n$} \\
            \midrule
            ECE & \multirow{2}{*}{not aligned} & \multirow{2}{*}{not aligned} \\
            MCE &  &  \\
            \midrule
            R-AUC & \multirow{2}{*}{not aligned} & \multirow{2}{*}{not aligned} \\
            E-Det &  &  \\
            \bottomrule
        \end{tabular}
    \end{subtable}
    \hfill
    \begin{subtable}{0.49\linewidth}
        \raggedleft
        \begin{tabular}{lll}
            \toprule
            Metric & Selective pred. & Top-$k$ select. \\
            \midrule
            \multirow{2}{*}{NLL} & $\pi(\lambda) = \varepsilon\lambda^{-2}\phantom{\tfrac{1}{(1)}}$  & \multirow{2}{*}{not aligned} \\
            & \tiny{for $\lambda,\sigma^2>\varepsilon>0$} $\phantom{2}$ & \\
            \midrule
            \multirow{2}{*}{MSE} & $\pi(\lambda) = \delta_\infty(\lambda)$, & \multirow{2}{*}{not aligned} \\
            & \tiny{not aligned for $\lambda < \infty$} & \\
            \midrule
            ECE & \multirow{2}{*}{not aligned} & \multirow{2}{*}{not aligned} \\
            MCE &  &  \\
            \midrule
            R-AUC & \multirow{2}{*}{not aligned} & \multirow{2}{*}{not aligned} \\
            E-Det &  &  \\
            \bottomrule
            \end{tabular}
    \end{subtable}
\end{table}

\section{Related work}\label{sect:related_work}

    Here, we summarize the prior work most relevant to our contribution.
    In \cref{app:intro}, we comment in more detail on the scope of our work.
    Further studies that share similarities with our framework or goals are acknowledged in \cref{app:related_work}.
    
    \paragraph{Common UQ evaluation metrics in ML}

        Current UQ studies in ML primarily evaluate uncertainty through predictive accuracy and calibration metrics.
        Benchmarking studies such as those of \citet{ovadia_can_2019}, \citet{nado_uncertainty_2022}, and \citet{basora_benchmark_2025} typically assess models using NLL, Brier score (BS), ECE, and ranking-scores such as the area under the receiver operating characteristic curve (AUROC).
        Among these metrics, \emph{proper scoring rules (PSRs)} such as the NLL and BS are often considered favorable as they are minimized at the ground-truth predictive distribution, and hence encourage honest uncertainty estimates~\cite{gneiting_strictly_2007,brummer_measuring_2010}.
        In practice, however, the true distribution is unknown, such that we never know whether a model ever truly minimizes the score \cite{pic_proper_2025}, limiting the value of properness in UQ benchmarking.
        While the benchmarking tools \textit{URSABench}~\cite{vadera_ursabench_2022} and \textit{Neural Testbed}~\cite{osband_neural_2022} do incorporate decision-making and bandit-style evaluations respectively, these decision scenarios are artificial and simplified, limiting their transferability to the real-world.
        In contrast, our prior-weighted utility metrics are aligned with realistic downstream decisions in a principled manner.
        
    \paragraph{Connection of PSRs \texorpdfstring{\&}{&} decisions}
        
        There is a well-known connection between PSRs and decision-making---every decision problem induces a PSR \cite{dawid_theory_2014}.
        Further, in binary classification, any PSR can be expressed as a cost-weighted average over dichotomous decisions (under mild regularity conditions, see~\cite{savage_elicitation_1971,schervish_general_1989,buja_loss_2005,gneiting_strictly_2007,brummer_measuring_2010,reid_information_2011}); a result that inspires our decision-alignment framework.
        Some regression PSRs can also be expressed in terms of decisions, e.g., the continuous ranked probability score and BS can be expressed as averages over thresholded decisions \cite{gneiting_strictly_2007}, but this relationship is less general than in classification \cite{ehm_quantiles_2016}.
        These relationships give rise to the assumption that good performance on common PSRs relates to good performance in downstream decisions, but we refute this assumption through our decision-alignment criterion.

    \paragraph{Connection between calibration \texorpdfstring{\&}{&} decisions}
    
        It is unclear whether better performance on calibration metrics leads to better decision-making.
        While some studies observed that good calibration implies good decisions \cite{laves_well-calibrated_2020,zhang_role_2023}, others argue that metrics such as the ECE are flawed in terms of evaluating the usefulness of a model in downstream decisions \cite{gruber_better_2022,kleinberg_u-calibration_2023,ferrer_evaluating_2025}.
        There has been substantial recent work on more decision-centric notions of calibration which aim to provide approximate guarantees (up to some tolerance $\varepsilon$) for downstream decision utilities or restricted families of tasks \cite{hu_calibration_2024,qiao_truthfulness_2025,yang_cost-aware_2025,hegazy_scalable_2025}.
        These approaches are primarily motivated by training or post-hoc adjustment, and have limited use for benchmarking since $\varepsilon$-guarantees do not allow for consistent \emph{ranking} of two predictors---differences in downstream utility may lie within the approximation error.
        In contrast, our prior-weighted utilities are intentionally designed to provide consistent model rankings in terms of downstream utility and are therefore preferable in the context of benchmarking.
        
\section{Implicit decision beliefs of UQ metrics}\label{sect:theory}

    Standard UQ metrics are often justified axiomatically (e.g., through properness or calibration), but there is only a limited understanding of their implications for downstream utilities.
    We generalize the known connection between PSRs and decisions (see \cref{sect:related_work}) to \emph{decision-alignment}, allowing a systematic analysis of common UQ metrics in terms of decision utilities.

    \subsection{Decision-alignment of UQ metrics}

        We consider metrics $M$ and decision utilities $U_\theta$, parametrized by $\theta \in \Theta$,
        that evaluate a prediction-label pair $(\boldsymbol f,\boldsymbol y)$, where $\boldsymbol{y}=(y_1,\dots,y_n) \in \mathcal{Y}^n$ and $\boldsymbol{f}=(f_1,\dots,f_n) \in \mathcal{P}^n$, $n \geq 2$, and $\mathcal{Y}, \mathcal{P}$ depend on the task and prediction model.
        We consider probabilistic predictions, so each $f \in \mathcal{P}$ encodes a distributional report over $\mathcal{Y}$.
        For example, in probabilistic binary classification, $\mathcal{Y} = \{0,1\}$ and $\mathcal{P}=[0,1]$, that is, the prediction model maps to the probability that a given instance has label~$1$.
        We restrict ourselves to prediction spaces $\mathcal{P}$ and utilities $U_\theta$ such that for each $f \in \mathcal{P}$ a (not necessarily unique) \emph{Bayes act} exists for $U_\theta$, that is, a decision that maximizes expected utility under~$f$.
        Without loss of generality, we say that smaller values of $M$ and larger values of $U$ correspond to better performance.
        Since the labels $\boldsymbol{y}$ are typically fixed, we use the notation\footnote{
            Note that often, utility functions evaluate \emph{actions} given labels, and we evaluate \emph{predictions} given labels.
            To be precise, we define $U_{\theta,y}(f)$ as the utility of the Bayes act under $f$.
        }
        \[
            M_{\boldsymbol y}(\boldsymbol f) := M(\boldsymbol f,\boldsymbol y) \quad \text{and}\quad U_{\theta,\boldsymbol y}(\boldsymbol f) := U_\theta(\boldsymbol f,\boldsymbol y).
        \]
        
        \begin{definition}[Decision-alignment]\label{def:decision_alignment}
            A metric $M$ is \emph{decision-aligned} w.r.t.\ a decision family $\{U_\theta\}_{\theta \in \Theta}$ if there exists a prior $\pi$ s.t.\ for any fixed labels $\boldsymbol y$, there exists a strictly increasing function $h_{\boldsymbol y}:\mathrm{Im}(M_{\boldsymbol y})\to\mathbb R$ s.t.\
            \begin{equation}
                h_{\boldsymbol y}\bigl(M_{\boldsymbol y}(\boldsymbol f)\bigr)=\int_{\Theta}-U_{\theta,\boldsymbol y}(\boldsymbol f) \, \pi(\theta) \, \mathrm{d}\theta\quad \forall \boldsymbol f. \label{eq:decision_alignment}
            \end{equation}
        \end{definition}

        \begin{remark}[On the nature of $\pi$]
            In our framework, $\pi$  can be any nonnegative and nonzero measure on $\Theta$.
            It does not need to be finite and, in particular, does not need to be a probability measure.
            We only require that the right-hand side of \eqref{eq:decision_alignment} is well-defined and finite.
            To simplify notation, we write $\pi(\theta)\mathrm{d}\theta$ for integration w.r.t.\ the measure $\pi$, including absolutely continuous priors with density $\pi(\theta)$, as well as purely atomic priors such as Dirac measures~$\delta$.
            We use the term \emph{prior} to highlight that $\pi$ encodes the metric's implicit belief on the downstream utility's parameter, and one can think of the right-hand side in \eqref{eq:decision_alignment} as the (negative) expected utility under $\theta \sim \pi$.
            Thus, we introduce the notation
            \[
                EU_{\theta \sim \pi,\boldsymbol y}(\boldsymbol f) := \int_{\Theta}-U_{\theta,\boldsymbol y}(\boldsymbol f) \, \pi(\theta) \, \mathrm{d}\theta.
            \]
        \end{remark}

        \begin{remark}[On the choice of our framework]
            The reason why we defined our framework of the form \eqref{eq:decision_alignment} is two-fold:
            
            (1) \emph{Metric-analysis}:
            Expressing a metric $M$ through \eqref{eq:decision_alignment} reveals whether it implicitly anticipates some decision family $\{U_\theta\}_{\theta \in \Theta}$ under a prior $\pi$.
            This helps us decide when and whether to use the metric, and to interpret model rankings under $M$ appropriately.

            (2) \emph{Desirable metric property}:
            Decision-alignment is a necessary and sufficient condition for \emph{strict order- and tie-preservation} (see \cref{prop:preservation}), and thus a naturally desirable metric property.
            Intuitively, \textbf{a metric $\boldsymbol M$ is decision-aligned if and only if it ranks predictions exactly as expected utility under $\boldsymbol \pi$~would}.
        \end{remark}

        \begin{restatable}[Decision-alignment $\Leftrightarrow$ strict order- and tie-preservation]{proposition}{preservation}\label{prop:preservation}
        
            Fix a decision family $\{U_\theta\}_{\theta\in\Theta}$ and a prior $\pi$.
            Then, for any fixed $\boldsymbol y$, the following statements are equivalent:
           
            (I) $M$ is decision-aligned w.r.t.\ $\{U_\theta\}_{\theta\in\Theta}$ under $\pi$.
           
            (II) $M$ is strictly order-preserving and tie-preserving for $\{U_\theta\}_{\theta\in\Theta}$ under $\pi$, i.e., for all $\boldsymbol f_1,\boldsymbol f_2$,
            \begin{align*}
                M_{\boldsymbol y}(\boldsymbol f_1) < M_{\boldsymbol y}(\boldsymbol f_2)
                &\Rightarrow
                EU_{\theta \sim \pi, \boldsymbol y}(\boldsymbol f_1) < EU_{\theta \sim \pi, \boldsymbol y}(\boldsymbol f_2),\\
                M_{\boldsymbol y}(\boldsymbol f_1) = M_{\boldsymbol y}(\boldsymbol f_2)
                &\Rightarrow
                EU_{\theta \sim \pi,\boldsymbol y}(\boldsymbol f_1) = EU_{\theta \sim \pi,\boldsymbol y}(\boldsymbol f_2).
            \end{align*}
        \end{restatable}

        \begin{proof}[Proof sketch]
            \textbf{(I)$\Rightarrow$(II):} Follows through strict monotonicity of $h_{\boldsymbol y}$.
            \textbf{(II)$\Rightarrow$(I):} Follows by defining $h_{\boldsymbol y}(m)$ as $EU_{\theta \sim \pi,\boldsymbol y}(\boldsymbol f)$ for any $\boldsymbol f$ with $M_{\boldsymbol y}(\boldsymbol f)=m$.
            We state the full proof in \cref{app:prop_preservation}.
        \end{proof}

        In other words, decision-aligned metrics rank models according to their expected utility in the respective downstream decision family.
        Further, any metric that is strictly order- and tie-preserving is also decision-aligned, so \cref{def:decision_alignment} includes \emph{all} metrics that fulfill these criteria.
        In \cref{app:properties}, we discuss the relationship between decision-aligned metrics and PSRs.

        To understand what kind of metrics can be decision-aligned with respect to what types of utility families, we distinguish between \emph{pointwise-separable}, \emph{coordinate-independent}, and \emph{coordinate-dependent} metrics/utilities.

        \begin{definition}[Pointwise-separability \& coordinate-independence]
            A metric/utility $S$ is \emph{pointwise-separable} if
            there exists a measurable \emph{instance function} $s$ s.t.\footnote{
                This is essentially a stricter version of \emph{decomposable} and \emph{non-decomposable} losses \cite{hu_rank-based_2023}, requiring the same instance function $s$ for each component $i$.
            }
            \begin{equation}
                S(\boldsymbol f,\boldsymbol y) = \frac 1n \sum_{i=1}^n s(f_i,y_i) \label{eq:pointwise_separability}
            \end{equation}
            Otherwise, $S$ is \emph{non-separable}.
            $S$ is \emph{coordinate-independent} if the change in $S$ through altering one prediction $f_i$ does not depend on the other coordinates,
            that is, for $a,b \in \mathcal{P}$ and $\boldsymbol r,\boldsymbol r' \in \mathcal{P}^{n-1}$
            \[
                S_{\boldsymbol y}((a,\boldsymbol r))
                <
                S_{\boldsymbol y}((b,\boldsymbol r))
                \Leftrightarrow
                S_{\boldsymbol y}((a,\boldsymbol r'))
                <
                S_{\boldsymbol y}((b,\boldsymbol r')),
            \]
            and analogously for equality and opposite strict inequality.\footnote{
                We write $(f_i,\boldsymbol r)$ for the prediction vector with value $f_i$ at the $i^\text{th}$ position and values $\boldsymbol r$ at the remaining $n-1$ positions.
            }
            Otherwise, $S$ is \emph{coordinate-dependent}.
        \end{definition}
        
        Note that \textbf{each pointwise-separable metric/utility is also coordinate-independent}:
        For a pointwise-separable $S$, changing $f_i$ only affects $s(f_i,y_i)$, which is independent of the remaining predictions.
        Next, we show that coordinate-dependent metrics (e.g., calibration or ranking metrics) can immediately be ruled out as decision-aligned with respect to pointwise-separable utility families.
        
        \begin{restatable}[Separability barrier]{lemma}{separabilitybarrier}\label{lem:separability_barrier}
            If the utility $U_\theta$ is pointwise-separable for every $\theta\in\Theta$ and $\theta\mapsto U_\theta(\boldsymbol f,\boldsymbol y)$ is integrable for all $(\boldsymbol f,\boldsymbol y)$, then any metric $M$ that is decision-aligned w.r.t.\ $\{U_\theta\}_{\theta\in\Theta}$ (for some prior $\pi$) must be coordinate-independent.
        \end{restatable}

        \begin{proof}[Proof sketch]
            By exchanging the integral and pointwise sum on the right-hand side of \eqref{eq:decision_alignment}, we get that $h_{\boldsymbol y}\bigl(M_{\boldsymbol y}(\boldsymbol f)\bigr)$ is pointwise-separable, which requires $M$ to be coordinate-independent.
            We state the full proof in \cref{app:separability_barrier}.
        \end{proof}

        Note that the converse of \cref{lem:separability_barrier} is not true:
        A coordinate-independent metric $M$ can be
        decision-aligned w.r.t.\ $\{U_\theta\}_{\theta \in \Theta}$ even if every $U_\theta$ is non-separable, as long as the non-separable parts cancel out in the integral (see \cref{app:converse} for an example).
        
        We now analyze concrete examples of metric-decision pairings through the lens of decision-alignment.
        We consider decision tasks that arise naturally in applications and span both pointwise-separable and non-separable utilities---binary decisions (yes/no), selective prediction (only predict when sure enough), and top-$k$ selection (select a few items from a large candidate pool).
        
    \subsection{Classification}\label{sect:bc}

        In this section, we consider probabilistic binary\footnote{
            We generalize our analysis to multiclass classification in \cref{app:multiclass}.
        } classification models $f : \mathcal{X} \to [0,1]$, mapping from the input space $\mathcal{X}$ to the probability that the true label from the output space $\mathcal{Y}=\{0, 1\}$ is 1.
        As evaluation metrics, we consider the
        negative log-likelihood (NLL),
        Brier score (BS),
        accuracy (Acc),
        expected calibration error (ECE),
        maximum calibration error (MCE),
        area under the retention curve (R-AUC),
        and error detection (E-Det).
        We state the metric definitions in \cref{app:metrics}, where we also show that the \textbf{NLL, BS, and Acc are coordinate-independent}, and the \textbf{ECE, MCE, R-AUC, and E-Det are coordinate-dependent}.
        The results of this section are summarized on the left of \cref{tab:theory_summary}.

        \subsubsection{Binary decision}\label{sect:bc_bd}

            We first consider a binary decision problem where an action in $\{0, 1\}$ (e.g., $\{\text{do not treat, treat}\}$) is taken based on the model's predicted probability $f(x) = f \in [0,1]$, $x \in \mathcal{X}$ for the positive class  (e.g., disease present).
            We assume zero cost for correct predictions and costs $C_\text{FP}, C_\text{FN} > 0$ for false positive and false negative predictions.
            Without loss of generality, we normalize the costs as $c_\text{FP} = C_\text{FP}/(C_\text{FP} + C_\text{FN})$ (so $c_\text{FN} = 1-c_\text{FP}$).
            For ease of notation, we write $c = c_\text{FP}$ in the following.
            With this setup, the Bayes act is to select action~$1$ whenever $f > c$ \cite{elkan_foundations_2001}.
            This allows us to write the pointwise utility of a prediction $f$ with true label $y \in \mathcal{Y}$ for a fixed cost $c \in (0,1)$ as
            \[
                u_c(f,y) = - y (1-c) \boldsymbol{1}_{f \leq c} - (1-y) c \boldsymbol{1}_{f > c}.
            \]
            The dataset-level utility $U_c(\boldsymbol f,\boldsymbol y)$ is defined as the average pointwise utility.
            As a direct consequence, \textbf{$\boldsymbol{U_c}$ is pointwise-separable}.
            By \cref{lem:separability_barrier}, only the NLL, BS, and Acc can be decision-aligned for $\{U_c\}_{c \in (0,1)}$.
            In fact, for the specific downstream utility $u_c$, it is already known that PSRs $l$ admit an integral representation (under mild conditions, see~\citep{savage_elicitation_1971,schervish_general_1989,buja_loss_2005,gneiting_strictly_2007,brummer_measuring_2010,reid_information_2011}),
            \begin{equation}
                l(f, y) =  \int_{0}^{1} - u_c(f,y) \, \pi(c) \, \mathrm{d}c.
                \label{eq:psr_bd_integral}
            \end{equation}
            The expressions for $\pi(c)$ for the NLL, BS, and Acc\footnote{
                Note that the pointwise Acc is not a PSR, but can still be expressed through \eqref{eq:psr_bd_integral}, rather by coincidence than by principle.
            } have been stated before (see, e.g.,~\cite{buja_loss_2005,gneiting_strictly_2007}).

            \begin{proposition}[\citet{buja_loss_2005}]\label{prop:buja}
                The NLL, BS, and Acc are decision-aligned w.r.t.\ $\{U_c\}_{c \in (0,1)}$ with
                \[
                    \pi_\text{NLL}(c) = 1/(c(1-c)),\ \pi_\text{BS} = 2,\ \pi_\text{Acc}(c) = 2\delta_{0.5}(c)
                \]
                and $h_{\boldsymbol y}=\mathrm{id}$ for all $\boldsymbol y$.
            \end{proposition}

            
            While this decision-theoretic interpretation suggests that good performance on PSRs and Acc implies good decision-making, we would like to emphasize that \textbf{the implicit priors of common PSRs are pathological}.\footnote{
                We characterize \emph{pathological priors} as either (a) degenerate priors, (b) uninformative priors, or (c) priors that place unbounded mass on extreme regions where UQ is least needed.
            }
            In many binary decisions, the cost of a false negative (e.g., missed disease) is much higher than the cost of a false positive (e.g., unnecessary treatment \citep{elkan_foundations_2001}).
            $\pi_\text{NLL}(c)$ places unbounded mass on extreme regions, rewarding performance where one error type is essentially free and a trivial policy suffices.
            The BS weighs each $c \in (0,1)$ equally, corresponding to an \emph{uninformative prior}.
            $\pi_\text{Acc}(c)$ is degenerate, expressing the strong assumption that the costs are perfectly symmetric.
            We visualize $\pi_\text{NLL}, \pi_\text{BS}$, and $\pi_\text{Acc}$ on the left of \cref{fig:path_priors}.

            Since the ECE, MCE, R-AUC, and E-Det are coordinate-dependent, we can directly follow from \cref{lem:separability_barrier} that they are not decision-aligned w.r.t.\ $\{U_c\}_{c \in (0,1)}$.

            \begin{corollary}\label{cor:bc_bd_all_cor}
                The ECE, MCE, R-AUC, and E-Det are not decision-aligned w.r.t.\ $\{U_c\}_{c \in (0,1)}$.
            \end{corollary}

        \subsubsection{Top-\texorpdfstring{$\boldsymbol k$}{k} selection}\label{sect:bc_topk}

            Top-$k$ selection for $k \in [n-1] := \{1,\dots,n-1\}$ arises, for instance, in the context of recommender systems. 
            As downstream utility, we use the established \emph{precision@$k$} metric (in line with, e.g.,~\cite{tamm_quality_2021}),
            \[
              U_k(\boldsymbol{f},\boldsymbol{y}) = \frac{1}{k} \sum_{i \in \mathcal{S}_k(\boldsymbol{f})} y_i,
            \]
            where $\mathcal{S}_k(\boldsymbol{f})$ is the index set of the selected instances according to the predicted scores.\footnote{
              For $\mathcal{S}_k(\boldsymbol{f})$ to be uniquely defined, one can choose a tie-breaking rule or restrict to $\boldsymbol{f}$ with unique entries.
            }
            This can be interpreted as receiving equal reward (normalized to $\tfrac{1}{k}$) for each correctly selected positive instance, and no reward for selecting negative instances.
            The Bayes act is to select the $k$ instances with the highest scores \cite{wydmuch_no-regret_2018}.
            \textbf{$\boldsymbol{U_k}$ is non-separable} (whether an instance contributes to $U_k$ depends on its rank relative to all other predictions, and therefore $U_k$ cannot be written as a sum of independent terms), so we cannot leverage \cref{lem:separability_barrier}.
            Nevertheless, we are able to prove that none of the metrics we consider are decision-aligned w.r.t.\ $\{U_k\}_{k\in [n-1]}$.

            \begin{restatable}{proposition}{bctopkallprop}\label{prop:bc_topk_all_prop}
                The NLL, BS, Acc, ECE, MCE, R-AUC, and E-Det are not decision-aligned w.r.t.~$\{U_k\}_{k\in [n-1]}$.
            \end{restatable}

            \begin{proof}[Proof sketch]
                The value of $U_k$ is only influenced by the subset $\mathcal{S}_k(\boldsymbol{f})$, and all metrics we consider evaluate on the entire dataset.
                One can find predictions $\boldsymbol{f}_1,\boldsymbol{f}_2$ that yield different metric values but leave the expected utility unchanged.
                We state the full proof in \cref{app:bc_topk_all_prop}.
            \end{proof}

            \begin{remark}[Case $k=n$]
                When defining $U_k$, we excluded the trivial case of $k=n$ (select all data points).
                Note, however, that Acc actually \emph{equals} $U_k$ for $k=n$.
                In other words, Acc is decision-aligned w.r.t.\ $\{U_k\}_{k\in [n]}$ under the (pathological) prior $\pi(k) = \delta_n(k)$.
            \end{remark}

    \subsection{Regression}

        In this section, we consider probabilistic univariate\footnote{
            We generalize our analysis to multivariate regression in \cref{app:multivariate}.
        }  regression models.
        We assume a Gaussian predictive distribution\footnote{
            We assume a Gaussian distribution for simplicity, but any distribution with a well-defined mean and variance is admissible.
        } and consider models of the form $f : \mathcal{X} \to \mathbb{R}\times\mathbb{R}_{> 0}$, mapping from the input space $\mathcal{X}$ to the predictive mean $\mu \in \mathbb{R}$ and predictive variance $\sigma^2 \in \mathbb{R}_{> 0}$, so $f=(\mu,\sigma^2)$.
        The regression target is the true label $y \in \mathcal{Y}:=\mathbb{R}$.
        As evaluation metrics, we consider the NLL, mean squared error (MSE), ECE, MCE, R-AUC, and E-Det.
        We state the metric definitions in \cref{app:metrics}, where we also show that the \textbf{NLL and MSE are coordinate-independent}, and the \textbf{ECE, MCE, R-AUC, and E-Det are coordinate-dependent}.
        We summarize the results on the right of \cref{tab:theory_summary}.

        \subsubsection{Selective prediction}\label{sect:reg_selpred}

            In selective prediction, the decision-maker chooses an action in $\{\text{predict}, \text{abstain}\}$, where selecting \emph{predict} corresponds
            to issuing a point prediction $\hat y$, while \emph{abstain} corresponds to deferring the decision
            (e.g., to a human expert).
            We assume squared loss for point predictions and a fixed abstention cost $\lambda \in \mathbb{R}_{>0}$.
            With this setup, the Bayes act
            is to issue the predictive mean $\mu$ whenever the predictive variance $\sigma^2$ is at most the abstention cost.
            This allows us to write the pointwise utility of a prediction $f = (\mu,\sigma^2)$
            with true outcome $y \in \mathbb{R}$ for a fixed abstention cost $\lambda$ as
            \[
                u_\lambda(f, y)
                =
                -(\mu - y)^2 \, \boldsymbol{1}_{\sigma^2 \le \lambda}
                -
                \lambda \, \boldsymbol{1}_{\sigma^2 > \lambda}.
            \]
            Once again, the dataset-level utility $U_\lambda(\boldsymbol f,\boldsymbol y)$
            is defined as the average pointwise utility, so $\boldsymbol{U_\lambda}$ \textbf{is pointwise-separable}, and therefore the ECE, MCE, R-AUC, and E-Det are not decision-aligned w.r.t.\ $\{U_\lambda\}_{\lambda \in \mathbb{R}_{>0}}$ by \cref{lem:separability_barrier}.
            We find that this also holds for the NLL and MSE.

            \begin{restatable}{proposition}{regselpredallprop}\label{prop:reg_selpred_all_prop}
                The NLL, MSE, ECE, MCE, R-AUC, and E-Det are not decision-aligned w.r.t.~$\{U_\lambda\}_{\lambda \in \mathbb{R}_{>0}}$.
            \end{restatable}

            \begin{proof}[Proof sketch]
                The MSE only evaluates the point prediction $\mu$ and is therefore insensitive to $\sigma^2$, so improvements in $U_\lambda$ through $\sigma^2$ cannot be measured.
                The NLL \emph{does} exhibit the same form as $U_\lambda$ after monotone transformation under a Pareto prior $\pi$, but the expected utility diverges because $\pi$ goes to infinity for $\lambda \to 0$ too quickly, violating our requirements for a valid prior.
                For the coordinate-dependent metrics, we can apply \cref{lem:separability_barrier}.
                We state the full proof in \cref{app:reg_selpred_all_prop}.
            \end{proof}

            Notably, the proof of \cref{prop:reg_selpred_all_prop} reveals that under minor additional assumptions, the NLL \emph{is} decision-aligned w.r.t.\ the $U_\lambda$-family under a Pareto prior.

            \begin{restatable}{proposition}{regselprednllprop}\label{prop:reg_selpred_nll_prop}
                Fix any $\varepsilon > 0$ and define $\mathcal{P}_\varepsilon := \{f=(\mu,\sigma^2)\mid \sigma^2 > \varepsilon\}$. Consider only predictions $\boldsymbol f \in \mathcal{P}_\varepsilon^n$.
                Then, the NLL is decision-aligned w.r.t.\ $\{U_\lambda\}_{\lambda \in \mathbb{R}_{>\varepsilon}}$ under $\pi(\lambda)=\varepsilon\lambda^{-2}$ and
                \[
                    h_{\boldsymbol y}(m)=2\varepsilon m-\varepsilon\log(2\pi \varepsilon)\quad \forall\boldsymbol y.
                \]
            \end{restatable}

            \begin{proof}[Proof sketch]
                When restricting to $\lambda,\sigma_i^2>\varepsilon$ for all $i$, we avoid the divergence issues from \cref{prop:reg_selpred_all_prop} and decision-alignment holds.
                We state the full proof in \cref{app:reg_selpred_nll_prop}.
            \end{proof}

            Note that this $\text{Pareto}(1,\varepsilon)$ prior is pathological, as it places more weight on $\lambda$ the smaller $\lambda$ is. In other words, it anticipates almost negligible abstention costs.
            We visualize the Pareto prior on the right of \cref{fig:path_priors}, alongside the Dirac prior for the case $\lambda = \infty$ discussed below.
    
            \begin{remark}[Case $\lambda=\infty$]
                When defining $U_\lambda$, we excluded the trivial case of $\lambda=\infty$ (always predict).
                Note, however, that the MSE actually \emph{equals} $U_\lambda$ for $\lambda=\infty$.
                In other words, the MSE is decision-aligned w.r.t.\ $\{U_\lambda\}_{\lambda\in \mathbb{R}_{>0}\cup \{\infty\}}$ under the (pathological) prior $\pi(\lambda) = \delta_\infty(\lambda)$.
            \end{remark}

    \subsubsection{Top-\texorpdfstring{$\boldsymbol k$}{k} selection}

        Given predictions $\boldsymbol f = (\boldsymbol \mu,\boldsymbol{\sigma^2})$ and labels $\boldsymbol y$, we consider a mean--variance trade-off of the form
        \[
            U_{\phi}(\boldsymbol f, \boldsymbol y)
            =
            \frac{1}{k}\sum_{i \in \mathcal{S}_{\phi}(\boldsymbol f)}
            (y_i - \gamma  \sigma_i^2),
        \]
        where $\phi=(k,\gamma) \in \Phi := [n-1] \times \mathbb{R}_{>0}$, and $\gamma$ controls the degree of risk aversion.
        The Bayes act is to select the top-$k$ instances according to the score $\mu_i - \gamma \sigma_i^2$.
        This yields a selection policy that trades off between predicted benefit and uncertainty.
        As in classification,
        we can
        show that none of the metrics we consider is decision-aligned w.r.t.\ $\{U_{\phi}\}_{\phi \in \Phi}$.

        \begin{restatable}{proposition}{regtopkallprop}\label{prop:reg_topk_all_prop}
            The NLL, MSE, ECE, MCE, R-AUC, and E-Det are not decision-aligned w.r.t.~$\{U_{\phi}\}_{\phi \in \Phi}$.
        \end{restatable}
        
        \begin{proof}[Proof sketch]
            As in classification, the value of the utility is only influenced by the top-$k$ subset, and all metrics we consider evaluate on the entire dataset.
            We state the full proof in \cref{app:reg_topk_all_prop}.
        \end{proof}

\begin{figure*}[t]
  \begin{center}
    \centerline{\includegraphics[width=.8\linewidth]{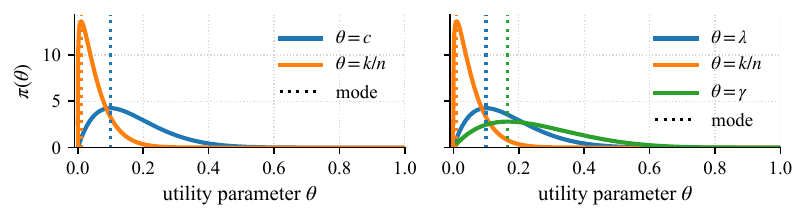}}
    \caption{
      Prior choices for our classification \emph{(left)} and regression \emph{(right)} prior-weighted utility metrics.
    }
    \label{fig:priors}
  \end{center}
  \vskip -0.2in
\end{figure*}

\section{Prior-weighted utility metrics}\label{sect:method}

    The key takeaway from \cref{sect:theory} is that common UQ evaluation metrics either encode implausible decision beliefs or do not reflect decision utility at all.
    In other words, \textbf{current UQ benchmarks silently optimize for the wrong downstream world}.
    In pursuit of a better UQ evaluation protocol, we propose \emph{prior-weighted utility (PWU)} metrics which are, by definition, \emph{decision-aligned} and encode \emph{plausible prior beliefs} about downstream decisions.\footnote{
        \cref{app:pwu_card} states a structured PWU documentation, similar to \emph{model cards} \cite{mitchell_model_2019}, alongside a broader impact statement.
    }

    \subsection{PWU construction}\label{sect:pwus}

        To define a PWU metric, we select a decision family $\{U_\theta\}_{\theta \in \Theta}$ and a plausible prior $\pi$ such that $EU_{\theta \sim \pi,\boldsymbol y}(\boldsymbol f)<\infty$ for all $(\boldsymbol f,\boldsymbol y)$.
        To align the PWU metric $M_{\pi}$ with $\{U_\theta\}_{\theta \in \Theta}$, we apply \cref{eq:decision_alignment} directly with $h_{\boldsymbol y}=\mathrm{id}$ for all $\boldsymbol y$:
        \begin{equation}
            M_{\pi}(\boldsymbol f,\boldsymbol y) := \int_{\Theta} - U_{\theta}(\boldsymbol f,\boldsymbol y) \, \pi(\theta) \, \mathrm{d}\theta.
            \label{eq:pwu_def}
        \end{equation}
        We propose a PWU metric for each decision problem analyzed in \cref{sect:theory}, obtaining
        \begin{alignat*}{2}
            M_{\pi_c}(\boldsymbol f,\boldsymbol y) &= \int_{0}^{1} - U_c(\boldsymbol f,\boldsymbol y) \, \pi_c(c) \, \mathrm{d}c, \qquad &
            M_{\pi_\lambda}(\boldsymbol f,\boldsymbol y) &= \int_{0}^{\infty} -U_\lambda(\boldsymbol f,\boldsymbol y) \, \pi_\lambda(\lambda) \, \mathrm{d}\lambda,\\
            M_{\pi_k}(\boldsymbol f,\boldsymbol y) &= \sum_{k=1}^{n-1} -U_k(\boldsymbol f,\boldsymbol y) \, \pi_k(k), \qquad &
            M_{\pi_\phi}(\boldsymbol f,\boldsymbol y) &= \int_{0}^{\infty}  \sum_{k=1}^{n-1} -U_{\phi}(\boldsymbol f,\boldsymbol y) \, \pi_{\phi}(\phi) \, \mathrm{d}\phi.
        \end{alignat*}
        In \cref{fig:priors}, we visualize the priors we choose for our experiments.
        In \cref{app:priors}, we justify them explicitly and provide some general guidance on prior elicitation.
        Note that constructing a PWU metric is a general method that can be applied to a broader set of decisions (and priors) than just those considered in this paper.
        Since each metric corresponds to a single decision family, it is advisable to use a \emph{variety} of PWU metrics in UQ benchmarking to test models across different decision contexts.

\begin{figure*}[ht]
  \begin{center}
    \centerline{\includegraphics[width=0.9\linewidth]{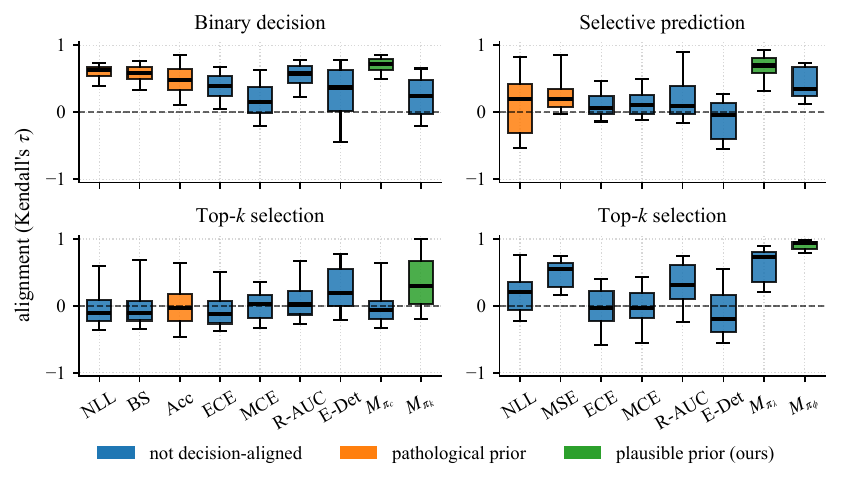}}
    \caption{
      Metric--utility alignment in classification \emph{(left)} and regression \emph{(right)}, averaged over our five datasets.
      The coloring corresponds to our theoretical findings from \cref{sect:theory}.
      The PWU metrics align best with respect to their corresponding utility families.
    }
    \label{fig:benchmark}
  \end{center}
  \vskip -0.2in
\end{figure*}

    \subsection{Metric properties}

        The purpose of our PWU metrics is to evaluate UQ in terms of downstream decision utility.
        Thus, we intentionally defined them to be decision-aligned with their corresponding utility family.

        \begin{proposition}[PWU metrics are decision-aligned]\label{prop:pwu_decisionaligned}
            Pick any decision family $\{U_\theta\}_{\theta \in \Theta}$ and a prior $\pi$ over $\Theta$.
            Then, the resulting PWU metric $M_{\pi}$ is decision-aligned w.r.t.\ $\{U_\theta\}_{\theta \in \Theta}$ under $\pi$.
        \end{proposition}

        \begin{proof}
            Follows by comparing the PWU definition \eqref{eq:pwu_def} to \cref{def:decision_alignment} with $h_{\boldsymbol y} = \mathrm{id}$ for all $\boldsymbol{y}$.
        \end{proof}

        Recall that decision-alignment is equivalent to strict order- and tie-preservation.
        Hence, PWU metrics offer a principled and decision-theoretic value for UQ benchmarking.
        Furthermore, using a PWU metric does not sacrifice properness---all PWU metrics are proper, and even strictly proper if the Bayes act for each $U_\theta$ is unique for all $\boldsymbol f$.  

        \begin{proposition}[PWU metrics are proper]\label{prop:pwu_proper}
            Pick any decision family $\{U_\theta\}_{\theta \in \Theta}$ and a corresponding prior $\pi$, and assume $\int_\Theta \mathbb E_{\boldsymbol Y\sim \boldsymbol p}\left[|U_{\theta,\boldsymbol Y}(\boldsymbol f)|\right] \, \pi(\theta) \, \mathrm d\theta < \infty$ for all $(\boldsymbol f,\boldsymbol y)$.
            Then, the resulting PWU metric $M_{\pi}$ is a PSR, and it is strictly proper if the Bayes acts for each $U_\theta$ are unique under all $\boldsymbol f$.
        \end{proposition}

        \begin{proof}
            $M_{\pi}$ is decision-aligned by \cref{prop:pwu_decisionaligned} and, by comparing \eqref{eq:pwu_def} to \cref{def:decision_alignment}, we have $h_{\boldsymbol y} = \mathrm{id}$ for all $\boldsymbol{y}$.
            Therefore, we can apply \cref{prop:da_properness} with $a=1,b=0$ and conclude that $M_{\pi}$ is proper (and strictly proper if the Bayes acts for each $U_\theta$ are unique under all $\boldsymbol f$).
        \end{proof}

        Our PWU metrics do not come without limitations (see also \cref{app:pwu_card}).
        No single PWU can cover all downstream objectives, so we recommend using a \emph{variety} of PWUs in UQ evaluation.
        Furthermore, PWU priors may be misspecified.
        However, in \cref{app:sens_analysis} we show empirically that \textbf{PWU metrics retain strong positive utility-alignment even under severe prior perturbation}, in stark contrast to the brittleness of conventional metrics under their pathological implicit priors.
        These results highlight PWUs' suitability for general-purpose UQ benchmarking.
        
\section{Experiments}\label{sect:experiments}

    
    \subsection{Controlled experiments on benchmark datasets}\label{sect:benchmarks}

        We implement ten binary classification and ten univariate regression models\footnote{
            We present results for the multiclass and multivariate settings in \cref{app:multi_exp}.
        } and evaluate these models on five datasets each.
        The experiments aim to measure how well the \emph{model ranking} of each metric $M$ aligns with the performance on downstream utilities.
        We use all metrics analyzed in \cref{sect:theory} in addition to our PWU metrics, and compare their rankings to the four decision utilities $U_\theta$, $\theta \in \{c,k,\lambda,\phi\}$ from \cref{sect:theory}.
        We then assess how well the model rankings of the metrics align with the model rankings of the utilities using \emph{Kendall's~$\tau$}.
        We repeat this procedure on 100 differently sampled test sets to quantify the results' variability.
        In \cref{fig:benchmark}, we report our results as boxplots\footnote{
            The boxplots show the median, 25th and 75th percentiles via boxes, and 5th, 95th percentiles via whiskers.
        } and state the exact numbers in Tables~\ref{tab:benchmark_bc_bd}, \ref{tab:benchmark_bc_topk}, \ref{tab:benchmark_reg_selpred}, \ref{tab:benchmark_reg_topk}.
        We detail our chosen models, datasets, experimental setup, and general reproducibility in \cref{app:experiments}.
        Overall, even though some conventional metrics are decision-aligned, we see that their pathological priors negatively impact their utility-alignment.
        Crucially, in top-$k$ selection and selective prediction, most of the common metrics exhibit little to no utility alignment, rendering them largely uninformative for downstream decision-making.
        In contrast, our PWU metrics align best with their corresponding utility family.

\begin{table}[t]
    \caption{
        Metric--utility alignment (Kendall's $\tau$ of model rankings) in the electricity market case study.
        We mark the highest median in \textbf{bold}.
        The bottom row states the 5th and 95th percentiles.
        The PWU metrics have the strongest and most stable positive bidding utility alignment.
    }
    \vspace{0.7em}
    \label{tab:bidding}
    \centering
                \begin{tabular}{cccccccc}
                    \toprule
                     NLL & MSE & ECE & MCE & R-AUC & E-Det & $M_{\pi_c}$ & $M_{\pi_k}$\\
                    \midrule
                    -0.16 & 0.09 & -0.24 & -0.24 & -0.07 & 0.07 & \textbf{0.16} & \textbf{0.16} \\
                    \scriptsize{{[}-0.47, 0.25{]}} & \scriptsize{{[}-0.47, 0.64{]}} & \scriptsize{{[}-0.51, 0.12{]}} & \scriptsize{{[}-0.56, 0.16{]}} & \scriptsize{{[}-0.42, 0.38{]}} & \scriptsize{{[}-0.38, 0.56{]}} & \scriptsize{{[}-0.16, 0.38{]}} & \scriptsize{{[}-0.11, 0.42{]}} \\
                    \bottomrule
                \end{tabular}
\end{table}
        
    \subsection{Applied case studies}\label{sect:realistic}

        While the above experiments consider the same utilities as our PWUs, we now deviate from this setup and implement three applied case studies.
        In our first case study, we consider a wind farm operator who wants to sell electricity in the \emph{day-ahead market}, where bids must be placed one day in advance.
        If the actual wind power production deviates from the bid, they risk a penalty for the resulting shortfall or surplus, as, in both cases, the electricity grid’s supply-demand balance needs to be restored.
        We follow the setup from \citet{bruninx_day-ahead_2025}, who derive an optimal bidding policy that resembles a selective prediction problem: bid only when the wind power forecast is sufficiently certain.
        The utility (payoff) is market-dependent for every bid/abstention and unknown in advance, deviating significantly from our controlled setup.
        We explain the case study in more detail in \cref{app:casestudies}
        and show the results in \cref{tab:bidding} and \cref{fig:bidding}.
        In this real-life setup, \emph{none} of the conventional UQ metrics have a stable positive bidding utility alignment, again showcasing how disconnected generic metrics are from real-world tasks.
        In contrast, our PWU metrics exhibit the strongest utility alignment.

        In our remaining two case studies (see \cref{app:casestudies}), we consider the tasks of credit approval and peer-to-peer lending, again using real-world economic payoffs as utility functions.
        We observe once more that our PWU metrics reliably align with the decision utility---despite the fact that some decision parameters differ strongly from our chosen PWU priors (see \cref{fig:cumsum_thresholds_ks}).
        This demonstrates that PWU metrics remain robust under prior misspecification, a finding further supported by our systematic sensitivity analysis (see \cref{app:sens_analysis}).

\section{Conclusion}

    We introduced decision-alignment as a criterion that connects uncertainty evaluation metrics to downstream decision utility and used it to diagnose that many standard UQ metrics are either misaligned with common decision problems or encode pathological implicit priors.
    Motivated by this insight, we proposed PWU metrics, a principled class of decision-aligned PSRs that enable utility-relevant benchmarking.
    Across benchmark datasets and real-world case studies, PWU metrics consistently produce model rankings that better reflect realized decision utility than conventional UQ metrics.
    Our results suggest that the dominant practice of evaluating UQ with generic metrics can be systematically misleading when the goal is to support downstream decisions.
    Instead, we claim that meaningful progress in probabilistic ML requires evaluating models with a variety of PWU metrics,\footnote{
        We provide our concrete recommendations for future UQ benchmarking that result from this paper in \cref{app:recomm}.
    }
    so that UQ benchmarks bring methods to light that enable reliable decision-making.


\begin{ack}
    AS was supported by the DAAD program Konrad Zuse Schools of Excellence in Artificial Intelligence, sponsored by the Federal Ministry of Research, Technology and Space.
    VF was supported by the Branco Weiss Fellowship.
\end{ack}

{\small
\bibliographystyle{unsrtnat}
\bibliography{Benchmarking}
}


\newpage
\appendix
\crefalias{section}{appendix}
\crefalias{subsection}{appendix}

\numberwithin{figure}{section}
\numberwithin{table}{section}

\numberwithin{definition}{section}
\numberwithin{theorem}{section}
\numberwithin{proposition}{section}
\numberwithin{lemma}{section}
\numberwithin{corollary}{section}
\numberwithin{remark}{section}
\numberwithin{assumption}{section}

\section{Scope of the paper}\label{app:intro}

    This appendix clarifies the scope of our work, which we term \emph{first-order UQ evaluation}.
    Throughout this paper, a probabilistic prediction refers to a single distribution over the label space; we call such an object a \emph{first-order probabilistic prediction}.
    In probabilistic binary classification, this is a single $f \in [0,1]$ encoding $\Pr(Y=1 \mid X)$; in regression, it is a single distribution over $\mathcal{Y}$ such as a Gaussian $\mathcal{N}(\mu, \sigma^2)$.
    This is the object almost universally produced and evaluated in probabilistic ML benchmarks~\cite{ovadia_can_2019,nado_uncertainty_2022,basora_benchmark_2025}, and the object for which all metrics we analyze (NLL, BS, MSE, Acc, ECE, MCE, R-AUC, E-Det) are defined.

    A separate strand of the literature studies \emph{second-order} probabilistic predictions: objects that take the form of a distribution over distributions.
    There are several reasons we restrict our scope to first-order evaluation.
    First, \textbf{our work is positioned as a response to the current mainstream UQ evaluation literature} (see \cref{sect:intro,sect:related_work}), which consistently uses metrics defined on first-order probabilistic predictions.
    Second, there is no broadly accepted evaluation protocol for second-order UQ to begin with~\cite{bengs_pitfalls_2022,bengs_second-order_2023}.
    Many studies even evaluate the marginal predictive (see, e.g.,~\cite{manchingal_random-set_2023,manchigal_epistemic_2025}), which is again first-order UQ and our analysis readily applies.
    Third, even if future research shifts more strongly toward second-order probabilistic predictions---a development we would welcome---we expect first-order evaluation to remain important for two reasons:
    \emph{interpretability}---a single predictive distribution is substantially easier to communicate, visualize, and reason about than a distribution over distributions; and
    \emph{decision-theoretic use}---Bayes-optimal actions under expected utility maximization are defined with respect to a single predictive distribution.
    Even models that maintain second-order beliefs internally are typically marginalized to a first-order posterior predictive distribution at decision time, so first-order quality remains the immediate driver of realized decision utility in most downstream pipelines.

    Nevertheless, extending decision-aligned evaluation to second-order predictions is an interesting and important direction which we leave for future work.

\section{Omitted related work}\label{app:related_work}

    \citet{ehm_quantiles_2016} study the evaluation and ranking of \emph{point forecasts} through consistent scoring functions for functionals (e.g., quantiles).
    Similar to our decision-alignment, they consider \emph{consistency} towards the functional of interest.
    They prove an order sensitivity result that resembles our \cref{prop:preservation}, however, they order point forecasts and associated scores, while we order scores and utilities of distributional reports.
    In addition, they show only one direction (consistency implies order sensitivity) and consider only inequalities, whereas we show that decision-alignment is \emph{equivalent} to strict order and tie preservation.
    Further, \citet{ehm_quantiles_2016} also discuss the decision-theoretic interpretation of binary decisions for PSRs and propose to use \emph{Murphy diagrams} to evaluate and rank forecasts when the downstream task is unknown.
    As in our approach, they do not commit to a single decision parameter $\theta$.
    Murphy diagrams plot the score value over varying $\theta$, whereas we set plausible prior $\pi(\theta)$ and report the corresponding (scalar) integral value.
    In addition, we do not restrict ourselves to an inherent downstream task like binary decisions in binary classification, but offer several decision-specific metrics for several tasks.

    \citet{bradley_summary_2011} also introduce a general framework tailored to ensemble weather forecasts, which is used to interpret known metrics.
    In their setup, probabilistic forecasts $f$ for a continuous variable $Y$ are mapped to probability forecasts for threshold events of the form $Y \le y$, and forecast quality is studied as a function $Q(y)$ of the threshold.
    Within this framework, several standard verification scores (e.g., the continuous ranked probability skill score) are shown to correspond to weighted averages of the forecast quality function over thresholds.
    Our decision-alignment framework is more general, allowing metric interpretation with respect to any given (integrable) utility family and is theoretically motivated through strict order and tie preservation in terms of expected downstream utility.
    
    \citet{shahroudi_aligning_2025} propose to learn a proxy evaluation function for a known downstream task, motivated by the fact that evaluating the downstream utility can be costly.
    While their framework is termed \emph{evaluation alignment}, similarly to our \emph{decision-alignment}, the concepts differ substantially: In evaluation alignment, the decision loss needs to be known and is approximated with a learned PSR, while we propose general-purpose metrics for general-purpose evaluation where no specific downstream task is anticipated.

    \citet{ferrer_evaluating_2025} advocate the use of \emph{expected PSRs (ESPRs)}, which are also defined as weighted integrals over a cost function (i.e., negative utility), but differ from our work in the sense that \citet{ferrer_evaluating_2025} only consider classification with pointwise classification cost.
    They neither analyze other metrics in terms of such an integral representation nor propose concrete new weight functions (i.e., priors), nor consider other utilities.
    The focus of the paper is rather to demonstrate that such ESPRs are superior to calibration metrics for UQ evaluation.

    \citet{pic_proper_2025} propose another metric framework for probabilistic weather forecasts, specific to PSRs, based on aggregation (combining PSRs through a finite weighted sum) and transformation (transform forecasts and labels before applying the PSR).
    This is a different framework than decision-alignment---we do not require a transformation and our aggregation step (prior-weighted integral) is more general.

    Also, \citet{derr_forecast_2025} introduce a general framework for expressing evaluation metrics that represents metrics as the average realized value of gambles drawn from a predefined set.
    Their framework aims to unify a wide range of existing evaluation metrics by elucidating their structural relationships.
    In contrast, our work is not primarily concerned with unification, but with analyzing when and how metrics are informative for downstream decision-making.

    \citet{flores_aligning_2025} also build on the Schervish representation \cite{schervish_general_1989} to propose evaluation metrics for binary classifiers that reflect clinical priorities.
    Similarly, \citet{flores_consequentialist_2026} propose to use a clipped variant of the Brier score to evaluate binary classifiers to better reflect decision-relevant thresholds.
    These works are conceptually a special case of our PWU metrics.
    
\section{Omitted proofs}

    \subsection{Proof of Proposition \texorpdfstring{\ref{prop:preservation}}{3.4}}\label{app:prop_preservation}

        \preservation*

        \begin{proof}
            Fix any $\boldsymbol y$.

            \textbf{(I)$\Rightarrow$(II):}
            Since (I) holds, there exists a strictly increasing $h_{\boldsymbol y}$ s.t.\ $h_{\boldsymbol y}(M_{\boldsymbol y}(\boldsymbol f))=EU_{\theta \sim \pi, \boldsymbol y}(\boldsymbol f)$ for all $\boldsymbol f$.
            Take any $\boldsymbol f_1,\boldsymbol f_2$.
            If $M_{\boldsymbol y}(\boldsymbol f_1)<M_{\boldsymbol y}(\boldsymbol f_2)$, then strict monotonicity of $h_{\boldsymbol y}$ implies $h_{\boldsymbol y}(M_{\boldsymbol y}(\boldsymbol f_1))<h_{\boldsymbol y}(M_{\boldsymbol y}(\boldsymbol f_2))$, hence $EU_{\theta \sim \pi, \boldsymbol y}(\boldsymbol f_1)<EU_{\theta \sim \pi, \boldsymbol y}(\boldsymbol f_2)$.
            If $M_{\boldsymbol y}(\boldsymbol f_1)=M_{\boldsymbol y}(\boldsymbol f_2)$, then, clearly, $h_{\boldsymbol y}(M_{\boldsymbol y}(\boldsymbol f_1))=h_{\boldsymbol y}(M_{\boldsymbol y}(\boldsymbol f_2))$, hence $EU_{\theta \sim \pi, \boldsymbol y}(\boldsymbol f_1)=EU_{\theta \sim \pi, \boldsymbol y}(\boldsymbol f_2)$.
            Altogether, (II) holds.

            \textbf{(II)$\Rightarrow$(I):}
            Assume (II) holds.
            Define $h_{\boldsymbol y}$ on $\mathrm{Im}(M_{\boldsymbol y})$ by
            \[
                h_{\boldsymbol y}(m) := EU_{\theta \sim \pi, \boldsymbol y}(\boldsymbol f) \quad \text{for any }\boldsymbol f\text{ with }M_{\boldsymbol y}(\boldsymbol f)=m.
            \]
            This is well-defined: if $M_{\boldsymbol y}(\boldsymbol f)=M_{\boldsymbol y}(\boldsymbol f')$, then tie preservation yields $EU_{\theta \sim \pi, \boldsymbol y}(\boldsymbol f)=EU_{\theta \sim \pi, \boldsymbol y}(\boldsymbol f')$.
            Moreover, for any $\boldsymbol f$, by construction $h_{\boldsymbol y}(M_{\boldsymbol y}(\boldsymbol f))=EU_{\theta \sim \pi, \boldsymbol y}(\boldsymbol f)$.
            It remains to show that $h_{\boldsymbol y}$ is strictly increasing.
            Let $m_1<m_2$ be in $\mathrm{Im}(M_{\boldsymbol y})$, and choose $\boldsymbol f_1,\boldsymbol f_2$ s.t.\ $M_{\boldsymbol y}(\boldsymbol f_1)=m_1$ and $M_{\boldsymbol y}(\boldsymbol f_2)=m_2$.
            Then, strict order preservation implies $EU_{\theta \sim \pi, \boldsymbol y}(\boldsymbol f_1)<EU_{\theta \sim \pi, \boldsymbol y}(\boldsymbol f_2)$, i.e., $h_{\boldsymbol y}(m_1)<h_{\boldsymbol y}(m_2)$.
            Hence, $h_{\boldsymbol y}$ is strictly increasing, completing the proof.
        \end{proof}

    \subsection{Proof of Lemma \texorpdfstring{\ref{lem:separability_barrier}}{3.6}}\label{app:separability_barrier}

        \separabilitybarrier*

        \begin{proof}
            Fix any $\boldsymbol y$ and let $M$ be decision-aligned w.r.t.\ $\{U_\theta\}_{\theta\in\Theta}$ for some prior weight $\pi$.
            Denote the instance function of $U_\theta$ by $u_\theta$.
            By the assumed integrability of $\theta\mapsto U_\theta(\boldsymbol f,\boldsymbol y)$ for all $(\boldsymbol f,\boldsymbol y)$ and the fact that the sum in \eqref{eq:pointwise_separability} is finite, we may interchange summation in \eqref{eq:pointwise_separability} and integration in \eqref{eq:decision_alignment} (Fubini/Tonelli):
            \[
                h_{\boldsymbol y}\bigl(M_{\boldsymbol y}(\boldsymbol f)\bigr)
                =\int_\Theta -\frac1n\sum_{i=1}^n u_\theta(f_i,y_i) \pi(\theta) \mathrm{d}\theta
                =\frac1n\sum_{i=1}^n \int_\Theta -u_\theta(f_i,y_i) \pi(\theta) \mathrm{d}\theta.
            \]
            Define the instance function $s(f_i,y_i):=\int_\Theta -u_\theta(f_i,y_i)\pi(\theta)\mathrm{d}\theta$, then the above equation reads
            \[
                h_{\boldsymbol y}\bigl(M_{\boldsymbol y}(\boldsymbol f)\bigr)
                = \frac1n\sum_{i=1}^n s(f_i,y_i).
            \]
            Define $S_{\boldsymbol y}(\boldsymbol f):=\frac1n\sum_{i=1}^n s(f_i,y_i)$, which is (by the form of its definition) pointwise-separable and hence also coordinate-independent, so, for $a,b \in \mathcal{P}$ and $\boldsymbol r,\boldsymbol r' \in \mathcal{P}^{n-1}$,
            \[
                S_{\boldsymbol y}((a,\boldsymbol r))
                <
                S_{\boldsymbol y}((b,\boldsymbol r))
                \Leftrightarrow
                S_{\boldsymbol y}((a,\boldsymbol r'))
                <
                S_{\boldsymbol y}((b,\boldsymbol r')),
            \]
            and because $h_{\boldsymbol y}\bigl(M_{\boldsymbol y}(\boldsymbol f)\bigr) = S_{\boldsymbol y}(\boldsymbol f)$,
            \[
                h_{\boldsymbol y}\bigl(M_{\boldsymbol y}((a,\boldsymbol r))\bigr)
                <
                h_{\boldsymbol y}\bigl(M_{\boldsymbol y}((b,\boldsymbol r))\bigr)
                \Leftrightarrow
                h_{\boldsymbol y}\bigl(M_{\boldsymbol y}((a,\boldsymbol r'))\bigr)
                <
                h_{\boldsymbol y}\bigl(M_{\boldsymbol y}((b,\boldsymbol r'))\bigr).
            \]
            Since $h_{\boldsymbol y}$ is strictly increasing, it preserves orderings, and the same coordinate-independence property must hold for $M_{\boldsymbol y}$.
            The same statements follow analogously for equality and strict inequality in the opposite direction.
        \end{proof}

    \subsection{Proof of Proposition \texorpdfstring{\ref{prop:bc_topk_all_prop}}{3.9}}\label{app:bc_topk_all_prop}

        \bctopkallprop*

        \begin{proof}
            For any metric $M \in \{\text{NLL},\text{BS},\text{Acc},\text{ECE},\text{MCE},\text{R-AUC},\text{E-Det}\}$, assume for contradiction that there exists a prior $\pi:[n-1]\to\mathbb{R}_{\ge 0}$ s.t.\ for any fixed labels $\boldsymbol y$, there exists a strictly increasing function $h_{\boldsymbol y}:\mathrm{Im}(M_{\boldsymbol y})\to\mathbb R$ s.t.\
            \begin{equation}
                h_{\boldsymbol y}\bigl(M_{\boldsymbol y}(\boldsymbol f)\bigr)=\sum_{k=1}^{n-1}-U_{k,\boldsymbol y}(\boldsymbol f)\pi(k) \quad \forall \boldsymbol f. \label{eq:bc_topk_all_prop_decalign}
            \end{equation}
            The value of $U_k$ is only influenced by the subset $\mathcal{S}_k(\boldsymbol{f})$, and we seek predictions $\boldsymbol f_1,\boldsymbol f_2$ that change our (monotonically transformed) metrics but not $U_k$.
            For this, consider the label vector $\boldsymbol y=\boldsymbol 0=(0,\dots,0)$,
            because then, for every $k\in[n-1]$ and every $\boldsymbol f$, $U_{k,\boldsymbol 0}(\boldsymbol f) =
                0$, and \eqref{eq:bc_topk_all_prop_decalign} would imply that
            \begin{equation}\label{eq:bc_topk_all_prop_Mzero}
                h_{\boldsymbol 0}\bigl(M_{\boldsymbol 0}(\boldsymbol f)\bigr)=0
                \quad\text{for all }\boldsymbol f.
            \end{equation}

            Because $h_{\boldsymbol 0}$ is strictly increasing, it is injective, and therefore \eqref{eq:bc_topk_all_prop_Mzero} implies that $M_{\boldsymbol 0}(\boldsymbol f)$ is \emph{constant} in $\boldsymbol f$:
            for all $\boldsymbol f_1,\boldsymbol f_2$,
            \begin{equation}\label{eq:bc_topk_all_prop_Mconst}
                M_{\boldsymbol 0}(\boldsymbol f_1)=M_{\boldsymbol 0}(\boldsymbol f_2).
            \end{equation}
            We now exhibit $\boldsymbol f_1,\boldsymbol f_2$ s.t.\ \eqref{eq:bc_topk_all_prop_Mconst} fails.

            \textbf{Case $\boldsymbol{M \in \{\text{NLL},\text{BS},\text{Acc}\}}$:}
            Let $\boldsymbol f_1=(\tfrac14,\dots,\tfrac14)$ and $\boldsymbol f_2=(\tfrac34,\dots,\tfrac34)$.
            Remember that we consider the negative conventional accuracy as Acc.
            Then,
            \begin{align*}
                \text{NLL}_{\boldsymbol 0}(\boldsymbol f_1)
                =
                -\log(\tfrac34) \quad
                &\neq \quad
                -\log(\tfrac14) = \text{NLL}_{\boldsymbol 0}(\boldsymbol f_2),\\
                \text{BS}_{\boldsymbol 0}(\boldsymbol f_1)
                =\tfrac1{16} \quad
                &\neq \quad
                \tfrac9{16} = \text{BS}_{\boldsymbol 0}(\boldsymbol f_2),\\
                \text{Acc}_{\boldsymbol 0}(\boldsymbol f_1)
                =-1 \quad
                &\neq \quad
                0 = \text{Acc}_{\boldsymbol 0}(\boldsymbol f_2),
            \end{align*}
            so \eqref{eq:bc_topk_all_prop_Mconst} is violated.

            \textbf{Case $\boldsymbol{M \in \{\text{ECE},\text{MCE}\}}$:}
            Fix any binning scheme $\mathcal{B}=\{B_1,\dots,B_m\}$ that partitions $[0,1]$ and let
            $\boldsymbol f_1,\boldsymbol f_2$ be as above.
            Then all entries of $\boldsymbol f_j$ fall into a single bin $B_j$ containing the constant value.
            Since $\boldsymbol y=\boldsymbol 0$, we have $\mathrm{freq}(B_j)=0$ and $\mathrm{conf}(B_j)$ equals the constant value.
            Therefore,
            \[
                \text{ECE}_{\boldsymbol 0}(\boldsymbol f_1)=|0-\tfrac12|=\tfrac14
                \quad
                \neq \quad
                \tfrac34 =
                \text{ECE}_{\boldsymbol 0}(\boldsymbol f_2),
            \]
            contradicting \eqref{eq:bc_topk_all_prop_Mconst}.
            Moreover, since only one bin is nonempty, the same holds for the MCE.

            \textbf{Case $\boldsymbol{M=\text{R-AUC}}$:}
            Let $\boldsymbol f_1,\boldsymbol f_2$ be as above.
            We have $\text{Acc}_{\boldsymbol 0}(\boldsymbol f_{1,i:i\in R})=-1$ for every retained set $R$,
            hence $\text{R-AUC}_{\boldsymbol 0}(\boldsymbol f_1)=0$.
            For $\boldsymbol f_2$, we have $\text{Acc}_{\boldsymbol 0}(\boldsymbol f_{2,i:i\in R})=0$ for every $R$,
            hence $\text{R-AUC}_{\boldsymbol 0}(\boldsymbol f_2)=1$, contradicting \eqref{eq:bc_topk_all_prop_Mconst}.

            \textbf{Case $\boldsymbol{M=\text{E-Det}}$:}
            Note that for the E-Det with $\boldsymbol y = 0$, we have $e_i=1$ iff $f_i>0.5$.
            Choose numbers $a,b\in(0,0.5)$ with $a<b$, and define $a':=1-a$, $b':=1-b$ so that
            $a',b'\in(0.5,1)$ and $s(a)=s(a')$, $s(b)=s(b')$ by symmetry of entropy.
            Define two prediction vectors:
            \[
                \boldsymbol f_1 = (\,a,\ b',\ b',\dots,b'\,),\quad
                \boldsymbol f_2 = (\,b,\ a',\ a',\dots,a'\,).
            \]
            Then, for $\boldsymbol f_1$, we have exactly one index with $f_i\le 0.5$ (hence $e_i=0$) and $n-1$ indices with $f_i>0.5$ (hence $e_i=1$).
            Moreover, the unique $e_i=0$ point has uncertainty score $s(a)$, while every $e_i=1$ point has uncertainty score $s(b')=s(b)$.
            Since $a<b$ and $s(\cdot)$ is strictly increasing on $(0,0.5)$ for entropy, we have $s(a)<s(b)$, so all positives ($e=1$) have strictly larger scores than all negatives ($e=0)$.
            Therefore, the AUROC of using $s_i$ to predict $e_i$ satisfies $\text{E-Det}_{\boldsymbol 0}(\boldsymbol f_1)=-1$.
            For $\boldsymbol f_2$, we again have exactly one index with $e_i=0$ and $n-1$ indices with $e_i=1$.
            Now the unique $e_i=0$ point has uncertainty score $s(b)$, while every $e_i=1$ point has uncertainty score $s(a')=s(a)$.
            Thus $s(b)>s(a)$ implies that the single negative has strictly larger score than all positives, so $\text{E-Det}_{\boldsymbol 0}(\boldsymbol f_2)=0$, contradicting \eqref{eq:bc_topk_all_prop_Mconst}.

            In all cases, we reach a contradiction. Therefore, no such prior $\pi$ can exist, which proves the claim.
        \end{proof}

    \subsection{Proof of Proposition \texorpdfstring{\ref{prop:reg_selpred_all_prop}}{3.11}}\label{app:reg_selpred_all_prop}

        \regselpredallprop*
        
        \begin{proof}
            For $M \in \{\text{ECE},\text{MCE},\text{R-AUC},\text{E-Det}\}$, the claim follow directly by \cref{lem:separability_barrier}.
            For $M \in \{\text{NLL},\text{MSE}\}$, assume for contradiction that there exists a prior $\pi:\mathbb{R}_{>0}\to\mathbb{R}_{\ge 0}$ s.t.\ for any fixed labels $\boldsymbol y$ there exists a strictly increasing function $h_{\boldsymbol y}:\mathrm{Im}(M_{\boldsymbol y})\to\mathbb R$ s.t.\
            \begin{equation}
                h_{\boldsymbol y}\bigl(M_{\boldsymbol y}(\boldsymbol f)\bigr)=\int_{\mathbb{R}_{>0}}-U_{\lambda,\boldsymbol y}(\boldsymbol f)\pi(\lambda)\mathrm{d}\lambda \quad \forall \boldsymbol f. \label{eq:reg_selpred_all_prop_decalign}
            \end{equation}

            \textbf{Case $\boldsymbol{M=\text{MSE}}$:}
            Fix any labels $\boldsymbol y$ and consider two prediction vectors $\boldsymbol f_1 = (\boldsymbol \mu, \boldsymbol{\sigma^2}_1), \boldsymbol f_2 = (\boldsymbol \mu, \boldsymbol{\sigma^2}_2)$
            with identical predictive means $\boldsymbol \mu = \boldsymbol y$.
            Take any interval $(a,b) \subset \mathbb{R}_{>0}$ with $ \int_{a}^{b}\pi(\lambda)\mathrm{d}\lambda > 0$ (must exist for a prior $\pi$)
            and consider $\sigma^2_{1,i}=a < b = \sigma^2_{2,i}$ for all $i$.
            Then $\text{MSE}(\boldsymbol f_1,\boldsymbol y) = \text{MSE}(\boldsymbol f_2,\boldsymbol y) = 0$,
            since the MSE ignores predictive uncertainty entirely, so $h_{\boldsymbol y}\bigl(M_{\boldsymbol y}(\boldsymbol f_1)\bigr) = h_{\boldsymbol y}\bigl(M_{\boldsymbol y}(\boldsymbol f_2)\bigr)$.
            For the difference on the right-hand side of \eqref{eq:reg_selpred_all_prop_decalign} when changing from $\boldsymbol f_1$ to $\boldsymbol f_2$, we have
            \begin{align*}
                &\int_{\mathbb{R}_{>0}}-U_{\lambda,\boldsymbol y}(\boldsymbol f_2)\pi(\lambda)\mathrm{d}\lambda - \int_{\mathbb{R}_{>0}}-U_{\lambda,\boldsymbol y}(\boldsymbol f_1)\pi(\lambda)\mathrm{d}\lambda\\
                &=
                \int_{0}^{b}\lambda\pi(\lambda)\mathrm{d}\lambda + 0 - \int_{0}^{a}\lambda\pi(\lambda)\mathrm{d}\lambda - 0
                =
                \lambda \int_{a}^{b}\pi(\lambda)\mathrm{d}\lambda \ne 0,
            \end{align*}
            contradicting \eqref{eq:reg_selpred_all_prop_decalign}.

            \textbf{Case $\boldsymbol{M=\text{NLL}}$:}
            Fix any labels $\boldsymbol y$.
            
            \emph{Step 1: Expected utility under $\pi$.}
            For each coordinate $i$, $-U_{\lambda,y_i}(f_i)
                =
                (\mu_i-y_i)^2 \boldsymbol{1}_{\sigma_i^2\le\lambda}
                +\lambda \boldsymbol{1}_{\sigma_i^2>\lambda}$.
            By linearity of integration,
            \begin{equation}
                \int -U_{\lambda,\boldsymbol y}(\boldsymbol f)\pi(\lambda)\mathrm{d}\lambda
                =\frac1n\sum_{i=1}^n
                \left(
                (\mu_i-y_i)^2
                \int_{\sigma_i^2}^{\infty}\pi(\lambda)\mathrm{d}\lambda
                +\int_{0}^{\sigma_i^2} \lambda\,\pi(d\lambda)
                \right).
                \label{eq:reg_selpred_all_prop_risk}
            \end{equation}
            Define $S(s):=\int_{s}^{\infty}\pi(\lambda)\mathrm{d}\lambda,\ s>0$.
            
            \emph{Step 2: One-dimensional reduction.}
            Fix an index $i$.  
            Fix all coordinates $(\mu_j,\sigma_j^2)$ for $j\neq i$ and fix
            $\sigma_i^2=s\in(0,\infty)$.  
            Let $x:=(\mu_i-y_i)^2\in[0,\infty)$.
            Under these choices,
            \begin{equation}\label{eq:reg_selpred_all_prop_M_affine}
                M_{\boldsymbol y}(\boldsymbol f)=A_s x+B_s,
                \quad
                A_s:=\frac{1}{2ns}>0,
            \end{equation}
            where $B_s$ is independent of $x$, and from \eqref{eq:reg_selpred_all_prop_risk},
            \begin{equation}\label{eq:reg_selpred_all_prop_R_affine}
                \int -U_{\lambda,\boldsymbol y}(\boldsymbol f)\pi(\lambda)\mathrm{d}\lambda
                =
                C_s x+D_s,
                \quad
                C_s:=\frac1n S(s)\ge 0,
            \end{equation}
            with $D_s$ independent of $x$.
            Substituting \eqref{eq:reg_selpred_all_prop_M_affine}--\eqref{eq:reg_selpred_all_prop_R_affine} into
            \eqref{eq:reg_selpred_all_prop_decalign} yields
            \begin{equation}\label{eq:reg_selpred_all_prop_functional}
                h(A_s x+B_s)=C_s x+D_s
                \quad\forall x\ge 0.
            \end{equation}

            \emph{Step 3: Form of $h$.}
            Note the following:
            For $A>0$, $B\in\mathbb R$, and $h$ a function defined on
            $[B,\infty)$, we have that if there exist constants $C,D$ s.t.\ $h(Ax+B)=Cx+D\quad\forall x\ge 0$,
            then
            \[
                h(z)=\frac{C}{A}z+\left(D-\frac{C}{A}B\right)
                \quad\forall z\in[B,\infty).
            \]
            This is because for any $z\ge B$, substituting $x=(z-B)/A\ge 0$ gives 
            \[
                h(z)=Cx+D
                =\frac{C}{A}z+\left(D-\frac{C}{A}B\right).
            \]
            Applying this to \eqref{eq:reg_selpred_all_prop_functional}, we conclude that
            \[
                h(z)=m_s z+b_s
                \quad\forall z\ge B_s,
            \]
            where
            \begin{equation}\label{eq:reg_selpred_all_prop_slope}
                m_s:=\frac{C_s}{A_s}=2s\,S(s)\ge 0.
            \end{equation}
            Because $h$ is strictly increasing on $\mathrm{im}(M_{\boldsymbol y})$ and
            $[B_s,\infty)\subseteq \mathrm{im}(M_y)$, we must have
            $m_s>0$, hence $S(s)>0$ for all $s>0$.

            \emph{Step 4: Consistency across different variances.}
            Let $s_1,s_2>0$.  
            Repeating the above construction yields representations
            \[
                h(z)=m_{s_1}z+b_{s_1}\quad(z\ge B_{s_1}),
                \quad
                h(z)=m_{s_2}z+b_{s_2}\quad(z\ge B_{s_2}).
            \]
            The half-lines $[B_{s_1},\infty)$ and $[B_{s_2},\infty)$ overlap, and both
            expressions describe the same function $h$ on the overlap.  
            Therefore, $m_{s_1}=m_{s_2}$.
            Since $s_1,s_2$ are arbitrary, there exists a constant $m>0$ s.t.\ $m_s=m\quad\forall s>0$.
            Combining with \eqref{eq:reg_selpred_all_prop_slope} gives
            \begin{equation}\label{eq:reg_selpred_all_prop_survival}
                S(s)=\frac{m}{2s}
                \quad\forall s>0.
            \end{equation}

            \emph{Step 5: Contradiction.}
            From \eqref{eq:reg_selpred_all_prop_survival}, we obtained
            \[
                S(s)=\int_s^\infty \pi(\lambda)\mathrm{d}\lambda=\frac{m}{2s}
                \quad \forall s>0,
            \]
            for some constant $m>0$.
            Since $S$ is absolutely continuous with derivative $S'(s)=-\pi(s)$ for a.e.\ $s>0$,
            differentiating yields
            \[
                \pi(s)=-S'(s)=\frac{m}{2s^2}
                \quad \text{for a.e.\ } s>0.
            \]
            But then, for any fixed $s>0$,
            \[
                \int_{0}^{s} \lambda\pi(\lambda)\mathrm{d}\lambda
                =\int_{0}^{s} \lambda \cdot \frac{m}{2\lambda^2}\mathrm{d}\lambda
                =\frac{m}{2}\int_{0}^{s} \frac{1}{\lambda}\mathrm{d}\lambda
                =+\infty.
            \]
            Consequently, the expected (negative) utility in \eqref{eq:reg_selpred_all_prop_risk} is infinite for any
            prediction vector with $\sigma_i^2=s$ for some coordinate $i$, and thus the right-hand
            side of \eqref{eq:reg_selpred_all_prop_decalign} is not well-defined, contradicting the definition of
            decision-alignment (which requires the integral to be well-defined for all admissible
            $\boldsymbol f$).
            Therefore, no such nonnegative measurable $\pi$ can exist, and the regression NLL is not decision-aligned w.r.t.\ $\{U_\lambda\}_{\lambda>0}$.
        \end{proof}

    \subsection{Proof of Proposition \texorpdfstring{\ref{prop:reg_selpred_nll_prop}}{3.12}}\label{app:reg_selpred_nll_prop}

        \regselprednllprop*

        \begin{proof}
            Restricting to $\lambda > \varepsilon$ and $f_i \in \mathcal{P}_\varepsilon$ for all $i$, the negative expected pointwise utility is
            \[
                \int_{\varepsilon}^{\sigma_i^2} \lambda \pi(\lambda)\mathrm{d}\lambda + (\mu_i - y_i)^2 \int_{\sigma_i^2}^{\infty} \pi(\lambda) \mathrm{d}\lambda.
            \]
            Choose $\pi(\lambda)$ as the probability density function of a $\text{Pareto}(1, \varepsilon)$ distribution, given by $\pi(\lambda) = \varepsilon \lambda^{-2}$ for $\lambda \ge \varepsilon$. Substituting this density into the integral gives
            \[
                \varepsilon \int_{\varepsilon}^{\sigma_i^2} \frac{1}{\lambda} \mathrm{d}\lambda + \varepsilon (\mu_i - y_i)^2 \int_{\sigma_i^2}^{\infty} \frac{1}{\lambda^2} \mathrm{d}\lambda
                = \varepsilon \left( \log \sigma_i^2 - \log \varepsilon \right) + \varepsilon \frac{(\mu_i - y_i)^2}{\sigma_i^2}.
            \]
            By linearity of expectation, the negative expected dataset-level utility is the average of the pointwise expected risks:
            \begin{align}
                \int_{\Theta} -U_{\lambda, \boldsymbol y}(\boldsymbol f) \pi(\lambda) \mathrm{d}\lambda &= \frac{1}{n} \sum_{i=1}^n \left[ \varepsilon \log \sigma_i^2 + \varepsilon \frac{(\mu_i - y_i)^2}{\sigma_i^2} - \varepsilon \log \varepsilon \right] \nonumber \\
                &= 2\varepsilon \left[ \frac{1}{n} \sum_{i=1}^n \left( \frac{1}{2} \log \sigma_i^2 + \frac{(\mu_i - y_i)^2}{2\sigma_i^2} \right) \right] - \varepsilon \log \varepsilon, \label{eq:reg_selpred_nll_prop_nll_derivation}
            \end{align}
            which closely resembles the NLL.
            We can express \eqref{eq:reg_selpred_nll_prop_nll_derivation} as an affine transformation of the NLL:
            \[
                \int_{\Theta} -U_{\lambda, \boldsymbol y}(\boldsymbol f) \pi(\lambda) \mathrm{d}\lambda = 2\varepsilon \cdot \text{NLL}_{\boldsymbol y}(\boldsymbol f) - \varepsilon\log(2\pi \varepsilon).
            \]
            Defining $h_{\boldsymbol y}(x) = 2\varepsilon x - \varepsilon\log(2\pi \varepsilon)$ for all $\boldsymbol y$, which is strictly increasing for $\varepsilon > 0$, we satisfy \eqref{eq:decision_alignment}.
            Thus, $M = \text{NLL}$ is decision-aligned with the selective prediction utility family under $\pi(\lambda)=\tfrac{\varepsilon}{\lambda^2}$ and $h_{\boldsymbol y}(m)=2\varepsilon m-\varepsilon\log(2\pi \varepsilon)$ for all $\boldsymbol y$.
        \end{proof}

    \subsection{Proof of Proposition \texorpdfstring{\ref{prop:reg_topk_all_prop}}{3.14}}\label{app:reg_topk_all_prop}

        \regtopkallprop*

        \begin{proof}
            For any $M \in \{\text{NLL},\text{MSE},\text{ECE},\text{MCE},\text{R-AUC},\text{E-Det}\}$, assume for contradiction that $M$ is decision-aligned w.r.t.\ $\{U_{\phi}\}_{\phi\in\Phi}$, i.e., that there exists a prior $\pi$ on $\Phi$ s.t.\ for any fixed labels $\boldsymbol y$ there exists a strictly increasing function
            $h_{\boldsymbol y}:\mathrm{Im}(M_{\boldsymbol y})\to\mathbb{R}$ with
            \[
                h_{\boldsymbol y}\bigl(M_{\boldsymbol y}(\boldsymbol f)\bigr)
                =
                \int_{\Phi} -U_{\phi,\boldsymbol y}(\boldsymbol f)\pi(\phi)\, \mathrm d\phi
                \quad \forall \boldsymbol f .
            \]
            Fix any strictly increasing sequence of variances $0 < v_1 < v_2 < \dots < v_n$ and define the label vector $\boldsymbol y=\boldsymbol 0=(0,\dots,0)$.
            Consider two prediction vectors with identical variances but different means:
            \[
                \boldsymbol f_1 := \bigl((0,\dots,0),(v_1,\dots,v_n)\bigr),
                \quad
                \boldsymbol f_2 := \bigl((1,\dots,1),(v_1,\dots,v_n)\bigr).
            \]
            For any $\phi=(k,\gamma)\in\Phi$ and any $\boldsymbol f=(\boldsymbol\mu,\boldsymbol\sigma^2)$, the selection rule ranks instances by the score
            $\mu_i-\gamma\sigma_i^2$.
            Since both $\boldsymbol f_1$ and $\boldsymbol f_2$ have \emph{constant} means and the \emph{same} strictly ordered variances,
            for every $\phi=(k,\gamma)$ the ordering of the scores is the ordering of $-\gamma v_i$, hence
            \[
                \mathcal{S}_{\phi}(\boldsymbol f_1)
                =
                \mathcal{S}_{\phi}(\boldsymbol f_2)
                =
                \{1,\dots,k\}
                \quad \forall (k,\gamma)\in\Phi.
            \]
            Therefore, for every $\phi=(k,\gamma)$,
            \[
                U_{\phi,\boldsymbol 0}(\boldsymbol f_1)
                =
                \frac{1}{k}\sum_{i=1}^k (0-\gamma v_i)
                =
                U_{\phi,\boldsymbol 0}(\boldsymbol f_2).
            \]
            Consequently, the expected utility under \emph{any} prior $\pi$ coincides:
            \[
                \int_{\Phi} -U_{\phi,\boldsymbol 0}(\boldsymbol f_1)\pi(\phi)\mathrm d\phi
                =
                \int_{\Phi} -U_{\phi,\boldsymbol 0}(\boldsymbol f_2)\pi(\phi)\mathrm d\phi.
            \]
            Since $h_{\boldsymbol 0}$ is injective, decision-alignment of $M$ would require that $M_{\boldsymbol 0}(\boldsymbol f_1)=M_{\boldsymbol 0}(\boldsymbol f_2)$.
            We now show that this fails.
        
            \textbf{Case $\boldsymbol{M=\text{MSE}}$:}
            Since $\boldsymbol y=\boldsymbol 0$,
            \[
                \text{MSE}_{\boldsymbol 0}(\boldsymbol f_1)
                =0 \quad
                \ne
                \quad
                1
                =
                \text{MSE}_{\boldsymbol 0}(\boldsymbol f_2).
            \]
        
            \textbf{Case $\boldsymbol{M=\text{NLL}}$:}
            For Gaussian NLL,
            \[
                \text{NLL}_{\boldsymbol 0}(\boldsymbol f_1)
                =
                \frac1n\sum_{i=1}^n \tfrac12\log(2\pi v_i)
                \quad
                \ne
                \quad
                \frac1n\sum_{i=1}^n\left(\tfrac12\log(2\pi v_i)+\frac{1}{2v_i}\right)
                =
                \text{NLL}_{\boldsymbol 0}(\boldsymbol f_2).
            \]

            \textbf{Case $\boldsymbol{M\in\{\text{ECE},\text{MCE}\}}$:}
            For $\boldsymbol y=\boldsymbol 0$ and $\boldsymbol f_1=((0,\dots,0),(v_1,\dots,v_n))$, each observation equals the predictive mean,
            so for every miscoverage level $\alpha_j$ and every index $i$ we have $y_i=\mu_i\in[\mu_i - z_{1-\alpha_j/2}\sqrt{v_i},\mu_i + z_{1-\alpha_j/2}\sqrt{v_i}]$,
            hence $\hat\alpha_j(\boldsymbol f_1,\boldsymbol 0)=0$.
            In contrast, for $\boldsymbol f_2=((1,\dots,1),(v_1,\dots,v_n))$ we have $|y_i-\mu_i|=1$ for all $i$, and therefore the empirical miscoverage
            $\hat\alpha_j(\boldsymbol f_2,\boldsymbol 0)$ depends on whether $1 \le z_{1-\alpha_j/2}\sqrt{v_i}$ holds.
            Choose the variance sequence s.t.\ $v_1$ is small enough s.t.\ $1 > z_{1-\alpha^*/2}\sqrt{v_1}$ for at least one level $\alpha^*$ and the remaining $v_i$ are large enough s.t. $1 \le z_{1-\alpha_j/2}\sqrt{v_i}$ for all $\alpha_j$.
            Then, $\hat\alpha^*(\boldsymbol f_2,\boldsymbol 0)=\tfrac 1n \ne 0 = \hat\alpha^*(\boldsymbol f_1,\boldsymbol 0)$, and all other empirical miscoverages are $0$ for both predictions, hence,
            \[
                \text{ECE}_{\boldsymbol 0}(\boldsymbol f_1)\neq \text{ECE}_{\boldsymbol 0}(\boldsymbol f_2)
                \quad\text{and}\quad
                \text{MCE}_{\boldsymbol 0}(\boldsymbol f_1)\neq \text{MCE}_{\boldsymbol 0}(\boldsymbol f_2).
            \]

            \textbf{Case $\boldsymbol{M=\text{R-AUC}}$:}
            In regression, R-AUC is computed from the retention sets $R_r$ obtained by sorting instances by Gaussian predictive entropy
            (which is a strictly increasing function of $\sigma_i^2$), and uses the MSE on the retained subset as the error rate.
            With our choice of $\boldsymbol f_1,\boldsymbol f_2$, the retained sets $R_r=\{\pi(r+1),\dots,\pi(n)\}$ are identical under $\boldsymbol f_1,\boldsymbol f_2$ (because $\boldsymbol\sigma^2$ is the same in both).
            Moreover, on any nonempty retained set $R_r$ we have
            \[
                \text{MSE}_{\boldsymbol 0}(\boldsymbol f_{1,i:i\in R_r})=0
                \quad
                \ne
                \quad
                1 =
                \text{MSE}_{\boldsymbol 0}(\boldsymbol f_{2,i:i\in R_r}).
            \]
            Consequently, $\text{R-AUC}_{\boldsymbol 0}(\boldsymbol f_1)\neq \text{R-AUC}_{\boldsymbol 0}(\boldsymbol f_2)$.

            \textbf{Case $\boldsymbol{M=\text{E-Det}}$:}
            In regression, we use the Gaussian predictive entropy as uncertainty score $s_i=s(\sigma_i^2)$ (strictly increasing in $\sigma_i^2$) and define error indicators by $e_i
                :=
                \boldsymbol{1}_{
                    \left|\frac{y_i-\mu_i}{|y_i|+\varepsilon}\right|>\tau
                }$
            for fixed $\varepsilon>0$ and $\tau>0$.
            Choose $a=0$ and $b=(\tau+1)\varepsilon$.
            Next, pick variances $0<v_1<\cdots<v_n$ and set $\boldsymbol\sigma^2=(v_1,v_n,\dots,v_n)$, so that $s(v_1)<s(v_n)$.
            Define two forecasts with identical variances but swapped mean offsets, $\boldsymbol f_1=(\boldsymbol\mu_1,\boldsymbol\sigma^2)$, $\boldsymbol f_2=(\boldsymbol\mu_2,\boldsymbol\sigma^2)$, where
            \[
                \boldsymbol\mu_1=(\,a,\ b,\ b,\dots,b\,),
                \qquad
                \boldsymbol\mu_2=(\,b,\ a,\ a,\dots,a\,).
            \]
            Note that the corresponding expected utilities still match, even for a random tie-breaking rule, since the utility only depends on $\boldsymbol 0$ and $\boldsymbol{\sigma^2}$.
            Under $\boldsymbol f_1$ we have $e_1=0$ and $e_i=1$ for $i\ge 2$, and the unique point with $e=0$ has uncertainty score $s(v_1)$ while every point with $e=1$ has uncertainty score $s(v_n)$.
            Since $s(v_1)<s(v_n)$, all positives have strictly larger scores than all negatives, hence $\text{E-Det}_{\boldsymbol 0}(\boldsymbol f_1)=-1$.
            Under $\boldsymbol f_2$ we have $e_1=1$ and $e_i=0$ for $i\ge 2$, so the unique positive has score $s(v_1)$ while all negatives have score $s(v_n)$, implying $\text{E-Det}_{\boldsymbol 0}(\boldsymbol f_2)=0$.

            In all cases, we obtain a contradiction. Therefore no such prior $\pi$ can exist, and $M$ is not decision-aligned w.r.t.\ $\{U_{\phi}\}_{\phi\in\Phi}$.
        \end{proof}

\section{Relationship between decision-alignment and properness}\label{app:properties}

    In this section, we prove that many decision-aligned metrics (those where the transformations $h_{\boldsymbol{y}}$ are affine) are proper.
    Note, however, that the converse is not true.
    While any PSR $M$ is decision-aligned to \emph{some} decision-family, which we can simply define as $U_\theta = \theta M$ (and considering $\pi(\theta)=\delta_1(\theta)$), a PSR is not decision-aligned to \emph{every} utility family.
    We discuss several such examples in this paper (for example, the NLL is not decision-aligned w.r.t.\ the top-$k$ utility family).

    \begin{proposition}[Decision-alignment with affine $h_{\boldsymbol y}$ $\Rightarrow$ properness]\label{prop:da_properness}
        Suppose that the metric $M$ is decision-aligned w.r.t.\ $\{U_\theta\}_{\theta\in\Theta}$ under $\pi$ with transformations $h_{\boldsymbol y}$, and assume $\int_\Theta \mathbb E_{\boldsymbol Y\sim \boldsymbol p}\left[|U_{\theta,\boldsymbol Y}(\boldsymbol f)|\right]\pi(\theta)\mathrm d\theta < \infty$ for all $(\boldsymbol f,\boldsymbol y)$.
        Further, assume there exists a constant $a>0$ and a function $b:\mathcal Y^n\to\mathbb R$ s.t.\
        \[
            h_{\boldsymbol y}(m) = am + b(\boldsymbol y) \quad \text{for all } \boldsymbol y.
        \]
        Then, $M$ is a PSR.
        If, in addition, the Bayes acts for each $U_\theta$ are unique under all $\boldsymbol f$, then $M$ is strictly proper.
    \end{proposition}

    \begin{proof}
        Define $S(\boldsymbol f,\boldsymbol y):=EU_{\theta \sim \pi,\boldsymbol y}(\boldsymbol f)$.
        By \citet{dawid_theory_2014}, we know that $-U_{\theta,\boldsymbol y}(\boldsymbol f)$ is a PSR for all $\theta$.
        We first show that this implies that $S(\boldsymbol f,\boldsymbol y)$ is a PSR as well.
        Denote by $\boldsymbol p$ the true distribution on $\mathcal Y^n$.
        By definition of properness, it holds for all predictions $\boldsymbol f$ that
        \[
            \mathbb E_{\boldsymbol Y\sim \boldsymbol p}\left[-U_{\theta,\boldsymbol Y}(\boldsymbol p)\right]
            \le
            \mathbb E_{\boldsymbol Y\sim \boldsymbol p}\left[-U_{\theta,\boldsymbol Y}(\boldsymbol f)\right]
            \quad\text{for all }\theta\in\Theta.
        \]
        Multiplying by the positive $\pi(\theta)$ and integrating both sides w.r.t.\ $\theta$ gives
        \[
            \int_\Theta \mathbb E_{\boldsymbol Y\sim \boldsymbol p}\left[-U_{\theta,\boldsymbol Y}(\boldsymbol p)\right]\pi(\theta)\mathrm d\theta
            \le
            \int_\Theta \mathbb E_{\boldsymbol Y\sim \boldsymbol p}\left[-U_{\theta,\boldsymbol Y}(\boldsymbol f)\right]\pi(\theta)\mathrm d\theta,
        \]
        where the integrals exist by assumption.
        By Fubini-Tonelli (again using the assumed finiteness of expectations),
        \[
            \int_\Theta \mathbb E_{\boldsymbol Y\sim \boldsymbol p}\left[-U_{\theta,\boldsymbol Y}(\boldsymbol f)\right]\pi(\theta)\mathrm d\theta
            =
            \mathbb E_{\boldsymbol Y\sim \boldsymbol p}\left[\int_\Theta -U_{\theta,\boldsymbol Y}(\boldsymbol f)\pi(\theta)\mathrm d\theta\right]
            =
            \mathbb E_{\boldsymbol Y\sim \boldsymbol p}\left[S(\boldsymbol f,\boldsymbol Y)\right].
        \]
        With the same equality on $\boldsymbol p$ instead of $\boldsymbol f$, we get
        \[
            \mathbb E_{\boldsymbol Y\sim \boldsymbol p}\left[S(\boldsymbol p,\boldsymbol Y)\right]
            \le
            \mathbb E_{\boldsymbol Y\sim \boldsymbol p}\left[S(\boldsymbol f,\boldsymbol Y)\right],
        \]
        which shows that $S(\boldsymbol f,\boldsymbol y)$ is itself a PSR.
        Now, since
        \[
            aM(\boldsymbol f,\boldsymbol y) + b(\boldsymbol y) = EU_{\theta \sim \pi,\boldsymbol y}\ \ \Leftrightarrow\ \  M(\boldsymbol f,\boldsymbol y) = \frac 1aEU_{\theta \sim \pi,\boldsymbol y} -  \frac 1a b(\boldsymbol y),
        \]
        we have for any prediction $\boldsymbol f$,
        \[
            \mathbb E_{\boldsymbol Y\sim \boldsymbol p}[M(\boldsymbol f,\boldsymbol Y)]
            = \frac 1a\mathbb E_{\boldsymbol Y\sim \boldsymbol p}[S(\boldsymbol f,\boldsymbol Y)] - \frac 1a\mathbb E_{\boldsymbol Y\sim \boldsymbol p}[b(\boldsymbol Y)],
        \]
        so the set of minimizers of $\mathbb E_{\boldsymbol Y\sim \boldsymbol p}[M(\boldsymbol f,\boldsymbol Y)]$ coincides with that of $\mathbb E_{\boldsymbol Y\sim \boldsymbol p}[S(\boldsymbol f,\boldsymbol Y)]$.
        Thus, $M$ is proper.
        If finally the Bayes acts for $U_\theta$ are unique under all $\boldsymbol f$, $-U_{\theta,\boldsymbol y}(\boldsymbol f)$ is a strict PSR, all inequalities become strict inequalities, and $M$ will be a strict PSR.
    \end{proof}

\section{Counterexample for the converse relation of Lemma \texorpdfstring{\ref{lem:separability_barrier}}{3.6}}\label{app:converse}

    Consider $\mathcal{P}=\mathcal{Y}=\mathbb{R}$.
    Define the pointwise-separable (and with this also coordinate-independent) metric
    \[
        M(\boldsymbol f,\boldsymbol y) := \frac1n\sum_{i=1}^n (f_i-y_i)^2.
    \]
    Let $\Theta=\{-1,+1\}$ and fix the prior $\pi(+1)=\pi(-1)=\tfrac12$.
    Define the non-separable functional $h(\boldsymbol f,\boldsymbol y) := \boldsymbol{1}_{f_1>f_2}$,
    which depends on the joint ordering of $(f_1,f_2)$ and hence is not pointwise-separable.
    Now define a utility family $\{U_\theta\}_{\theta\in\Theta}$ by
    \[
        U_{+1}(\boldsymbol f,\boldsymbol y) := -M(\boldsymbol f,\boldsymbol y) - h(\boldsymbol f,\boldsymbol y),
        \qquad
        U_{-1}(\boldsymbol f,\boldsymbol y) := -M(\boldsymbol f,\boldsymbol y) + h(\boldsymbol f,\boldsymbol y).
    \]
    Then both $U_{+1}$ and $U_{-1}$ are non-separable (because of the $\pm h$ term), yet
    \begin{align*}
        \int_{\Theta} -U_\theta(\boldsymbol f,\boldsymbol y)\mathrm{d}\pi(\theta)
        &=
        \frac12\bigl(-U_{+1}(\boldsymbol f,\boldsymbol y)\bigr)
        +
        \frac12\bigl(-U_{-1}(\boldsymbol f,\boldsymbol y)\bigr) \\
        &=
        \frac12\bigl(M(\boldsymbol f,\boldsymbol y)+h(\boldsymbol f,\boldsymbol y)\bigr)
        +
        \frac12\bigl(M(\boldsymbol f,\boldsymbol y)-h(\boldsymbol f,\boldsymbol y)\bigr) \\
        &=
        M(\boldsymbol f,\boldsymbol y).
    \end{align*}
    Hence, $M$ is decision-aligned w.r.t.\ the family $\{U_\theta\}_{\theta\in\Theta}$ under $\pi$, even though every $U_\theta$ in the family is non-separable.

\section{Omitted metric definitions}\label{app:metrics}

    In this section, we define all metrics used in this paper and prove for each metric whether it is coordinate-independent or not.
    We start by defining each metric and giving a short intuition of why they are coordinate-(in)dependent.
    At the bottom of this section, we state a proposition, where we formally prove these intuitions.

    \paragraph{Negative log-likelihood (NLL)}

        For binary classification, the pointwise NLL is defined as
        \[
            l_\text{NLL}(f,y)=-y\log(f)-(1-y)\log(1-f),
        \]
        and for regression ($f=(\mu,\sigma^2)$), we use the Gaussian NLL,
        \[
            l_{\text{NLL}}(f,y) = \tfrac12\log(2\pi\sigma^2)+\tfrac{(y-\mu)^2}{2\sigma^2}.
        \]
        The dataset-level loss is defined as the average,
        \[
            \text{NLL}(\boldsymbol f,\boldsymbol y)=\frac{1}{n}\sum_{i=1}^n l_{\text{NLL}}(f_i,y_i).
        \]
        Thus, by definition, the \textbf{NLL is pointwise-separable, hence also coordinate-independent}.
        Further, the NLL is a PSR \cite{gneiting_strictly_2007}.

    \paragraph{Brier score (BS)}

        The pointwise BS is defined as
        \[
            l_\text{BS}(f,y)=(f-y)^2,
        \]
        and the dataset-level loss is defined as the average,
        \[
            \text{BS}(\boldsymbol f,\boldsymbol y)=\frac{1}{n}\sum_{i=1}^n l_{\text{BS}}(f_i,y_i).
        \]
        Thus, by definition, the \textbf{BS is pointwise-separable, hence also coordinate-independent}.
        Further, the BS is a PSR \cite{gneiting_strictly_2007}.

    \paragraph{Accuracy (Acc)}

        Accuracy (Acc) is conventionally defined as the ratio of correctly classified instances when using a generic decision threshold of $0.5$.
        To conform with our notion that smaller metrics correspond to better performance, we consider the negative conventional accuracy as our Acc.
        With this, it is the sample average of the pointwise decision 0-1 loss, the indicator function of a false classification of any kind, minus $1$, i.e.,
        \begin{align*}
          &l_{\text{0-1}}(f,y) = y \boldsymbol{1}_{f\leq 0.5} + (1-y) \boldsymbol{1}_{f > 0.5},\\
          &\text{Acc}(\boldsymbol p, \boldsymbol y) = \frac 1n \sum_{i=1}^nl_{\text{0-1}}(p_i,y_i) - 1.
        \end{align*}
        Thus, by definition, \textbf{Acc is pointwise-separable, hence also coordinate-independent}, however, the Acc is not a PSR \cite{gneiting_strictly_2007}.

    \paragraph{Mean squared error (MSE)}

        The pointwise mean squared error (MSE) is the squared-error loss, defined as
        \[
            l_{\text{MSE}}(f,y) := (\mu-y)^2,
        \]
        where $f=(\mu,\sigma^2)$,
        and the dataset MSE is
        \[
            \text{MSE}(\boldsymbol f,\boldsymbol y) := \frac 1n \sum_{i=1}^nl_{\text{MSE}}(f_i,y_i).
        \]
        Thus, by definition, the \textbf{MSE is pointwise-separable, hence also coordinate-independent}.
        The MSE only evaluates the predictive mean, not the full Gaussian predictive distribution, and can therefore not be considered proper in terms of the full probabilistic report $\boldsymbol f \in \mathcal{P}$.

    \paragraph{Expected calibration error (ECE)}

        A widely used empirical calibration metric in binary classification is the ECE, which measures the discrepancy between empirical frequencies and stated probabilities (aka \emph{calibration}) within bins of the prediction space.
        Formally, fix a binning scheme $\mathcal{B}=\{B_1,\dots,B_m\}$ that partitions $[0,1]$.
        The ECE is defined as
        \[
            \text{ECE}(\boldsymbol f,\boldsymbol y)
            :=
            \sum_{j=1}^m \frac{n_j}{n}
            \bigl|
                \text{freq}(B_j)-\text{conf}(B_j)
            \bigr|,
        \]
        where $n_j$ is the number of indices $i$ s.t.\ $f_i\in B_j$,
        $\text{freq}(B_j)=\tfrac{1}{n_j}\sum_{i:\,f_i\in B_j} y_i$, and
        $\text{conf}(B_j)=\tfrac{1}{n_j}\sum_{i:\,f_i\in B_j} f_i$.
        We use $10$ equally spaced bins.
        The ECE is not a PSR \cite{ferrer_evaluating_2025}.
        The \textbf{ECE is coordinate-dependent}, since it couples coordinates through bin membership and bin statistics.

        To use the ECE in regression, we follow \citet{cui_calibrated_2020}, who assess calibration of regression models through nominal coverage levels
        $\alpha_1,\dots,\alpha_m\in(0,1)$ and the corresponding Gaussian prediction intervals.
        We use $\alpha \in \{0.5, 0.8, 0.9, 0.95\}$.
        Let $\hat\alpha_j((\boldsymbol\mu,\boldsymbol\sigma^2),\boldsymbol y)$
        denote the empirical miscoverage of the $(1-\alpha_j)$-level prediction interval.
        Then, the empirical regression ECE is defined as
        \[
            \text{ECE}((\boldsymbol\mu,\boldsymbol\sigma^2),\boldsymbol y)
            :=
            \frac{1}{m}\sum_{j=1}^{m}
            \bigl|\hat\alpha_j((\boldsymbol\mu,\boldsymbol\sigma^2),\boldsymbol y)-\alpha_j\bigr|.
        \]
        As in the binary classification case, the regression ECE is not a PSR and coordinate-dependent.
        
    \paragraph{Maximum calibration error (MCE)}

        The MCE measures the worst-case deviation between empirical frequencies and stated probabilities across bins of the prediction space.
        Using the same notation as for ECE, the MCE is defined as
        \[
            \text{MCE}(\boldsymbol f,\boldsymbol y)
            :=
            \max_{j\in\{1,\dots,m\}}
            \bigl|
                \text{freq}(B_j)-\text{conf}(B_j)
            \bigr|.
        \]
        Like the ECE, the MCE is not a PSR and \textbf{MCE is coordinate-dependent}.
        We adapt the MCE to regression in the same way as the ECE.

    \paragraph{Area under the retention curve (R-AUC)}

        The R-AUC \cite{malinin_shifts_2021} measures the usefulness of a model's uncertainty estimates for selecting correct predictions.
        For varying certainty thresholds $\tau$, the \emph{retention curve} tracks the error rate of a model when only predicting on the the subset with certainty at least $\tau$.
        When uncertainty is informative, incorrect predictions tend to be rejected first, resulting in a steep decline of the retention curve.
        By integrating this curve, R-AUC yields a single scalar that reflects both overall predictive accuracy and the ability of uncertainty estimates to rank errors.
        Formally, let $s_i = s(f_i)$ denote a scalar uncertainty score for instance $i$, with larger values indicating higher uncertainty.
        We use the predictive entropy,
        \[
            s_i = -f_i\log f_i - (1-f_i)\log (1-f_i).
        \]
        Let $\pi$ be a permutation of $\{1,\dots,n\}$ s.t.\ $s_{\pi(1)} \ge \cdots \ge s_{\pi(n)}$.
        For $r \in \{0,1,\dots,n\}$, define the retained index set
        \[
            R_r := \{\pi(r+1),\dots,\pi(n)\},
        \]
        corresponding to retaining a fraction $\tau=|R_r|/n$ of predictions and deferring the
        remaining $r$ most uncertain instances to an oracle.
        The area under the retention curve (R-AUC) is then defined as
        \[
            \text{R-AUC}(\boldsymbol f,\boldsymbol y)
            :=
            \frac{1}{n} \sum_{r=0}^{n-1} \bigl(1+\text{Acc}(\boldsymbol{f}_{i:i\in R_r}, \boldsymbol{y}_{i:i\in R_r})\bigr),
        \]
        where $1 + \text{Acc}$ is the error rate---remember that our Acc is the negative conventional accuracy.
        The R-AUC is not a PSR, since it can be improved by pushing uncertainty scores for likely-wrong predictions artificially high, so honest reporting is not necessarily optimal.
        Further, the \textbf{R-AUC is coordinate-dependent}:
        changing a single prediction $f_i$ affects not only the error term
        $\boldsymbol{1}_{f_i \neq y_i}$ associated with instance $i$, but also the ordering
        induced by the uncertainty scores $\{s_k\}_{k=1}^n$.

        For regression models, we use the Gaussian predictive entropy as uncertainty score and the MSE as error rate.
        Like in binary classification, the regression R-AUC is not a PSR and is coordinate-dependent.

    \paragraph{Error detection (E-Det)}

        E-Det \cite{bouvier_towards_2022} evaluates how well uncertainty can be used as a score to discriminate between correct and incorrect predictions.
        We define the hard predictions as $\hat{y}_i = \boldsymbol{1}_{f_i > 0.5}$ and the corresponding error indicators as $e_i = \boldsymbol{1}_{\hat{y}_i \neq y_i}$.
        Again writing $s_i$ for the uncertainty score (where we use predictive entropy, as above), E-Det is defined as the classical area under the receiver operating curve (AUROC), using $s_i$ as predictions for the labels $e_i$.
        To conform with our notion that smaller metrics correspond to better performance, we consider the negative conventional error detection as our E-Det.
        The E-Det is not a PSR and the \textbf{E-Det is coordinate-dependent} because of its ranking characteristic, similar to the R-AUC.

        In regression, we again use the Gaussian predictive entropy as uncertainty score and define errors $e_i$ as predictions whose relative errors $|\tfrac{y-\mu}{|y| + \varepsilon}|$ exceed $10\%$ (using a small $\varepsilon > 0$ to ensure the relative error is well-defined).

        \begin{proposition}[Coordinate-(in)dependence of common UQ evaluation metrics]\label{prop:coord_dependence}
            The NLL, BS, Acc, and MSE are \textbf{coordinate-independent}, and the ECE, MCE, R-AUC, and E-Det are \textbf{coordinate-dependent}.
        \end{proposition}

        \begin{proof}
            The coordinate-independence of the NLL, BS, Acc, and MSE follows directly from their pointwise-separability by definition.
            We now prove coordinate-dependence of the remaining metrics.

            \textbf{ECE \& MCE:}
            We prove the claim for $\text{ECE}$; the claim for $\text{MCE}$ then follows by the same construction (indeed, in our construction, only one bin is populated, so $\text{ECE}=\text{MCE}$).

            \emph{Step 1: choose a rational frequency inside a non-degenerate bin.}
            Let $B=(\ell,u]$ be a bin of nonzero width, i.e., $u-\ell>0$.
            Pick any rational number $q\in(\ell,u)$; write it as $q=k/t$ with integers $t\ge 2$ and $k\in\{0,1,\dots,t\}$.
            (Existence is guaranteed since every non-degenerate interval contains rationals.)
            Set $n:=t$ and fix labels $\boldsymbol y$ with exactly $k$ ones and $t-k$ zeros.
            Hence, for any prediction vector $\boldsymbol f$ for which \emph{all} coordinates lie in $B$, we have
            \[
                \mathrm{freq}(B)=\frac{1}{t}\sum_{i=1}^t y_i=\frac{k}{t}=q.
            \]

            \emph{Step 2: fix two values $a<b$ in the same bin.}
            Choose any $a,b\in B$ with $a<b$ and with $a$ and $b$ sufficiently close to $q$ so that the values defined below remain inside $B$.
            (Concretely: since $B$ is an open neighborhood of $q$ relative to $(\ell,u]$, there exists $\delta>0$ s.t.\ $(q-\delta,q+\delta)\subseteq B$; choose $a,b\in(q-\delta,q+\delta)$ with $a<b$.)
            Fix an index $i\in\{1,\dots,t\}$ (any choice works).

            \emph{Step 3: build two backgrounds that make the direction of change flip.}
            Define $c := \tfrac{tq-a}{t-1}$ and $c' := \tfrac{tq-b}{t-1}$.
            By construction, if we set all coordinates of $\boldsymbol f$ except $i$ equal to $c$ (respectively $c'$), then the bin confidence equals $q$ when $f_i=a$ (respectively $f_i=b$):
            \[
                \frac{1}{t}\Bigl(a+\sum_{r\neq i} c\Bigr)=\frac{a+(t-1)c}{t}=\frac{a+tq-a}{t}=q,
                \quad
                \frac{1}{t}\Bigl(b+\sum_{r\neq i} c'\Bigr)=\frac{b+(t-1)c'}{t}=q.
            \]
            Since $a$ and $b$ were chosen sufficiently close to $q$, the values $c$ and $c'$ are also close to $q$ and thus lie in $B$ as well (note that we can rewrite $c,c'$ as perturbations of $q$: $c=q+\tfrac{q-a}{t-1}$, $c'=q+\tfrac{q-b}{t-1}$.
            Therefore, in both constructions, all predictions fall into the \emph{same} bin $B$.
            Now define two prediction vectors:
            \[
                \boldsymbol f_1 := (a,\underbrace{c,\dots,c}_{t-1 \text{ times}}),
                \qquad
                \tilde{\boldsymbol f}_1 := (b,\underbrace{c,\dots,c}_{t-1 \text{ times}}),
            \]
            and
            \[
                \boldsymbol f_2 := (a,\underbrace{c',\dots,c'}_{t-1 \text{ times}}),
                \qquad
                \tilde{\boldsymbol f}_2 := (b,\underbrace{c',\dots,c'}_{t-1 \text{ times}}),
            \]
            where the distinguished coordinate is $i$ (we suppress the explicit placement notation for readability).
            Since all coordinates lie in $B$, the ECE reduces to the single-bin absolute deviation:
            \[
                \text{ECE}_{\boldsymbol y}(\boldsymbol f)=\bigl|\mathrm{freq}(B)-\mathrm{conf}(B)\bigr|.
            \]
            For $\boldsymbol f_1$ we have $\mathrm{conf}(B)=q=\mathrm{freq}(B)$, hence $\text{ECE}_{\boldsymbol y}(\boldsymbol f_1)=0$.
            But for $\tilde{\boldsymbol f}_1$ we have
            \[
                \mathrm{conf}(B)=\frac{b+(t-1)c}{t}
                =\frac{b+(tq-a)}{t}
                =q+\frac{b-a}{t},
            \]
            so
            \[
                \text{ECE}_{\boldsymbol y}(\tilde{\boldsymbol f}_1)
                =\Bigl|q-\Bigl(q+\frac{b-a}{t}\Bigr)\Bigr|
                =\frac{b-a}{t}
                >0.
            \]
            Thus, under the background $\boldsymbol r=(c,\dots,c)$, changing $f_i$ from $a$ to $b$ \emph{increases} the ECE: $\text{ECE}_{\boldsymbol y}\bigl((a,\boldsymbol r)\bigr) < \text{ECE}_{\boldsymbol y}\bigl((b,\boldsymbol r)\bigr)$.
            Conversely, for $\tilde{\boldsymbol f}_2$ we again have perfect calibration in $B$ by construction, hence $\text{ECE}_{\boldsymbol y}(\tilde{\boldsymbol f}_2)=0$,
            whereas for $\boldsymbol f_2$ we compute
            \[
                \mathrm{conf}(B)=\frac{a+(t-1)c'}{t}
                =\frac{a+(tq-b)}{t}
                =q-\frac{b-a}{t},
            \]
            so
            \[
                \text{ECE}_{\boldsymbol y}(\boldsymbol f_2)
                =\Bigl|q-\Bigl(q-\frac{b-a}{t}\Bigr)\Bigr|
                =\frac{b-a}{t}
                >0.
            \]
            Hence, under the background $\boldsymbol r'=(c',\dots,c')$, changing $f_i$ from $a$ to $b$ \emph{decreases} the ECE: $\text{ECE}_{\boldsymbol y}\bigl((a,\boldsymbol r')\bigr) > \text{ECE}_{\boldsymbol y}\bigl((b,\boldsymbol r')\bigr)$.

            \emph{Transfer to regression.}
            Fix a nominal level $\alpha\in(0,1)$.
            For Gaussian predictive distributions, let
            \[
                \hat\alpha((\boldsymbol\mu,\boldsymbol\sigma^2),\boldsymbol y)
                := \frac{1}{n}\sum_{i=1}^n
                \boldsymbol{1}_{\{y_i \notin I_\alpha(\mu_i,\sigma_i^2)\}}
            \]
            denote the empirical miscoverage of the $(1-\alpha)$-level prediction interval $I_\alpha(\mu_i,\sigma_i^2)$.
            Since the regression ECE (resp.\ MCE) is obtained by averaging (resp.\ maximizing) the absolute deviations
            $\lvert \hat\alpha-\alpha\rvert$ over the finitely many levels $\alpha_j$, coordinate-independence would in particular require coordinate-independence of the contribution of each fixed level $\alpha$ (as we can easily pick labels and predictions such that slightly perturbing a prediction only affects one miscoverage term).
            We therefore fix $\alpha$ and construct a counterexample for $\lvert \hat\alpha-\alpha\rvert$.
            As in the classification proof, fix an index $i$ and two alternative predictions
            $a$ and $b$ for the $i$th coordinate (i.e., two Gaussian reports $(\mu_i,\sigma_i^2)$),
            such that the $i$th prediction interval covers $y_i$ under $a$ but does not cover $y_i$ under $b$.
            For the remaining coordinates $k\neq i$, construct two background prediction vectors
            $\boldsymbol r$ and $\boldsymbol r'$ as follows:
            in $\boldsymbol r$, choose the $(\mu_k,\sigma_k^2)$ so that exactly $\alpha t$ of the $t-1$ background
            prediction intervals fail to cover their labels when the $i$th coordinate is set to $a$,
            whereas in $\boldsymbol r'$, choose them so that exactly $\alpha t$ background intervals fail to
            cover their labels when the $i$th coordinate is set to $b$
            (such constructions are always possible by choosing the remaining intervals sufficiently wide or narrow.).
            By construction, we then have
            \[
                \hat\alpha((a,\boldsymbol r),\boldsymbol y)=\alpha,
                \quad
                \hat\alpha((b,\boldsymbol r'),\boldsymbol y)=\alpha,
            \]
            so that $\lvert\hat\alpha-\alpha\rvert=0$ in both cases.
            However, switching $a$ to $b$ under the background $\boldsymbol r$ increases the number of miscovered
            instances by one and thus increases $\lvert\hat\alpha-\alpha\rvert$,
            whereas switching $a$ to $b$ under the background $\boldsymbol r'$ decreases the number of miscovered
            instances by one and thus decreases $\lvert\hat\alpha-\alpha\rvert$.
            Hence,
            \[
                \lvert\hat\alpha((a,\boldsymbol r),\boldsymbol y)-\alpha\rvert
                <
                \lvert\hat\alpha((b,\boldsymbol r),\boldsymbol y)-\alpha\rvert,
                \quad
                \lvert\hat\alpha((a,\boldsymbol r'),\boldsymbol y)-\alpha\rvert
                >
                \lvert\hat\alpha((b,\boldsymbol r'),\boldsymbol y)-\alpha\rvert.
            \]
            This flip of inequalities shows that the direction of change of the calibration error induced by
            modifying a single coordinate depends on the remaining coordinates.
            Therefore, the regression ECE is coordinate-dependent.
            Moreover, by choosing a nominal level $\alpha$ whose absolute miscoverage deviation strictly
            dominates all others, the same construction ensures that this level uniquely attains the maximum in
            the definition of the regression MCE.
            Hence, the same argument applies to the regression MCE, which is therefore
            also coordinate-dependent.

            \textbf{R-AUC:}
            We give explicit counterexamples that violate coordinate-independence.
            We start with the binary classification case.
            Let $n=2$ and fix labels $\boldsymbol y=(1,0)$.
            For a probability $f\in(0,1)$, write $s(f) := -f\log f-(1-f)\log(1-f)$ for the predictive entropy, and let the hard prediction be $\hat y(p):=\boldsymbol{1}_{p>0.5}$.
            For $n=2$ the R-AUC reduces to
            \[
                \text{R-AUC}_{\boldsymbol y}(\boldsymbol f)
                =\frac12\Bigl(\mathrm{err}(\{1,2\})+\mathrm{err}(R_1)\Bigr),
            \]
            where $\mathrm{err}(\cdot)$ is the classification error rate, and $R_1$ is obtained by rejecting the single most uncertain instance.
            Fix the index $i=1$ and choose $a:=0.9,b:=0.51$.
            Note that $\hat y(a)=\hat y(b)=1$, hence the first instance is correct under both $a$ and $b$, but $s(b)>s(a)$.

            \emph{Background 1.}
            Let the second prediction coordinates be $r:=0.8$.
            Then $\hat y(r)=1\neq y_2$, so the second instance is wrong, and moreover $s(b)>s(r)>s(a)$.
            Hence, with $\boldsymbol f=(a,r)$ the most uncertain instance is $2$, so $R_1=\{1\}$ and $\mathrm{err}(R_1)=0$; while with $\boldsymbol f=(b,r)$ the most uncertain instance is $1$, so $R_1=\{2\}$ and $\mathrm{err}(R_1)=1$.
            In both cases the full-set error rate is $\mathrm{err}(\{1,2\})=1/2$ (exactly one error).
            Therefore,
            \[
                \text{R-AUC}_{\boldsymbol y}((a,r))=\frac12\Bigl(\frac12+0\Bigr)=\frac14
                <
                \frac34 = \frac12\Bigl(\frac12+1\Bigr) = \text{R-AUC}_{\boldsymbol y}((b,r)),
            \]
            so changing $f_1$ from $a$ to $b$ \emph{increases} the R-AUC.

            \emph{Background 2.} Let the second prediction coordinate be $r':=0.01$.
            Then $\hat y(r')=0=y_2$, so the second instance is correct, and $s(r')<s(a)<s(b)$.
            Hence, regardless of whether $f_1=a$ or $f_1=b$, the most uncertain instance is $1$ and thus $R_1=\{2\}$, yielding $\mathrm{err}(R_1)=0$.
            Moreover, both instances are correct under $(a,r')$ and $(b,r')$, so $\mathrm{err}(\{1,2\})=0$ in both cases.
            Thus,
            \[
                \text{R-AUC}_{\boldsymbol y}((a,r'))=\text{R-AUC}_{\boldsymbol y}((b,r'))=0.
            \]
            Combining the two backgrounds, we have exhibited fixed labels $\boldsymbol y$, an index $i=1$, and values $a,b$ s.t.\
            \[
                \text{R-AUC}_{\boldsymbol y}((a,r))<\text{R-AUC}_{\boldsymbol y}((b,r)),
                \ \ \text{while}\ \
                \text{R-AUC}_{\boldsymbol y}((a,r'))=\text{R-AUC}_{\boldsymbol y}((b,r')),
            \]
            violating coordinate-independence.

            Now we move on to regression.
            Let $n=2$ and fix labels $\boldsymbol y=(0,0)$.
            Reports are $f_i=(\mu_i,\sigma_i^2)$, the uncertainty score is the Gaussian predictive entropy $s(\mu,\sigma^2):=\tfrac12\log(2\pi e\,\sigma^2)$,
            which is strictly increasing in $\sigma^2$, and the error rate on a retained set is the MSE computed from the predictive means.
            Fix $i=1$ and choose the two alternative reports
            \[
                a:=(\mu_1,\sigma_1^2)=(1,0.1^2),
                \quad
                b:=(\mu_1,\sigma_1^2)=(1,10^2),
            \]
            so that the mean-squared error of the first instance equals $(\mu_1-y_1)^2=1$ under both $a$ and $b$, but the uncertainty score satisfies $s(b)>s(a)$.

            \emph{Background 1.} Let the remaining coordinate be $r:=(\mu_2,\sigma_2^2)=(0,1^2)$,
            so the second instance has squared error $(\mu_2-y_2)^2=0$ and uncertainty score $s(r)$ with $s(a)<s(r)<s(b)$.
            Hence, with $(a,r)$ the most uncertain instance is $2$ and $R_1=\{1\}$, whereas with $(b,r)$ the most uncertain instance is $1$ and $R_1=\{2\}$.
            The full-set MSE equals
            \[
                \mathrm{MSE}(\{1,2\})=\frac{1+0}{2}=\frac12
            \]
            for both $(a,r)$ and $(b,r)$, while the retained-set MSE equals $1$ under $(a,r)$ and $0$ under $(b,r)$.
            Therefore,
            \[
                \text{R-AUC}_{\boldsymbol y}((a,r))=\frac12\Bigl(\frac12+1\Bigr)=\frac34
                >
                \frac14 = \frac12\Bigl(\frac12+0\Bigr) = \text{R-AUC}_{\boldsymbol y}((b,r)),
            \]
            so changing $f_1$ from $a$ to $b$ \emph{decreases} the regression R-AUC.
    
            \emph{Background 2.} Let the remaining coordinate be $r':=(\mu_2,\sigma_2^2)=(0,0.01^2)$,
            so $s(r')<s(a)<s(b)$.
            Then the most uncertain instance is $1$ under both $(a,r')$ and $(b,r')$, hence $R_1=\{2\}$ in both cases.
            Again $\mathrm{MSE}(\{1,2\})=1/2$, and the retained-set MSE is $0$ in both cases, so
            \[
                \text{R-AUC}_{\boldsymbol y}((a,r'))=\text{R-AUC}_{\boldsymbol y}((b,r'))=\frac14.
            \]
            Thus, the strict comparison between $\text{R-AUC}_{\boldsymbol y}((a,\cdot))$ and $\text{R-AUC}_{\boldsymbol y}((b,\cdot))$ depends on the remaining coordinate, proving that the regression R-AUC is also coordinate-dependent.

            \textbf{E-Det:}
            Since $\text{E-Det}$ is the negative conventional AUROC, it suffices to show that the conventional E-Det is coordinate-dependent; multiplying by $-1$ preserves coordinate-(in)dependence.
            We start with binary classification.
            Let $n=3$, fix labels $\boldsymbol y=(1,0,1)$, and let $\hat y(f):=\boldsymbol{1}_{f>0.5}$.
            Define the error indicators $e_i=\boldsymbol{1}_{\hat y(f_i)\neq y_i}$, and let $\text{AUROC}(\boldsymbol s,\boldsymbol e)$
            be the standard AUROC treating larger $s_i$ as stronger evidence for $e_i=1$.
            Fix $i=1$ and choose two values $a:=0.9,b:=0.49$.
            Then $\hat y(a)=1$ (hence $e_1=0$ under $a$) while $\hat y(b)=0$ (hence $e_1=1$ under $b$), and moreover
            $s(b)>s(a)$ since predictive entropy is maximized at $0.5$.
            We construct two backgrounds $\boldsymbol r,\boldsymbol r'\in(0,1)^{2}$ (for coordinates $2,3$) that yield different directions of change.

            \emph{Background 1.} Set $(f_2,f_3)=(0.999,0.6)$.
            Then $\hat y(f_2)=1\neq y_2=0$, so $e_2=1$, and $\hat y(f_3)=1=y_3$, so $e_3=0$.
            Moreover, $s(0.999) < s(0.9) < s(0.6) < s(0.49)$.
            Hence, under $(a,\boldsymbol r)$ the positive set is $\{2\}$ and the negative set is $\{1,3\}$, and
            the unique positive score $s_2=s(0.999)$ is smaller than both negative scores, so $\text{AUROC}_{\boldsymbol y}((a,\boldsymbol r))=0$.
            Under $(b,\boldsymbol r)$ the positive set is $\{1,2\}$ and the negative set is $\{3\}$, and we have
            $s(b)>s(0.6)>s(0.999)$, hence exactly one of the two positive-negative comparisons is correct, so $\text{AUROC}_{\boldsymbol y}((b,\boldsymbol r))=\frac{1}{2}$.
            Thus, changing $f_1$ from $a$ to $b$ \emph{increases} AUROC under background $\boldsymbol r$.

            \emph{Background 2.} Set $(f_2,f_3)=(0.51,0.99)$.
            Then $\hat y(f_2)=1\neq y_2=0$, so $e_2=1$, and $\hat y(f_3)=1=y_3$, so $e_3=0$.
            Moreover, $s(0.51)$ is close to maximal and therefore larger than both $s(0.9)$ and $s(0.99)$.
            Hence, under $(a,\boldsymbol r')$ the positive set is $\{2\}$ and negatives are $\{1,3\}$, and
            $s_2$ exceeds both negative scores, so $\text{AUROC}_{\boldsymbol y}((a,\boldsymbol r'))=1$.
            Under $(b,\boldsymbol r')$ the positives are $\{1,2\}$ and the negative is $\{3\}$, and both positive
            scores $s(b)$ and $s(0.51)$ exceed the negative score $s(0.99)$, so again $\text{AUROC}_{\boldsymbol y}((b,\boldsymbol r'))=1$.
            Thus, under background $\boldsymbol r'$, the same change from $a$ to $b$ leaves AUROC unchanged, violating coordinate-independence.

            We now continue with the regression case.
            Let $n=3$ and fix labels $\boldsymbol y=(1,1,1)$.
            Reports are $f_i=(\mu_i,\sigma_i^2)$.
            Fix $\tau>0$ and $\varepsilon>0$, and define error indicators $e_i := \boldsymbol{1}_{\left\{\left|\frac{y_i-\mu_i}{|y_i|+\varepsilon}\right|\ge \tau\right\}}$.
            Define the conventional AUROC using scores $s_i=s(f_i)$ to predict labels $e_i$.
            Fix $i=1$ and choose two reports
            \[
                a := (\mu_1,\sigma_1^2)=(1,\;0.1^2),
                \quad
                b := (\mu_1,\sigma_1^2)=(1+(|1|+\varepsilon)\tau,\;10^2).
            \]
            Then under $a$ we have $e_1=0$, while under $b$ we have $e_1=1$ (by equality in the threshold),
            and $s(b)>s(a)$.

            \emph{Background 1.} Set
            \[
                f_2=(1+(|1|+\varepsilon)\tau,\;0.01^2),
                \quad
                f_3=(1,\;1^2).
            \]
            Then $e_2=1$ and $e_3=0$.
            Moreover, by monotonicity in $\sigma^2$, $s(f_2) < s(a) < s(f_3) < s(b)$.
            Hence, under $(a,\boldsymbol r)$ the positive set is $\{2\}$ and negatives are $\{1,3\}$, and the
            unique positive score is smaller than both negative scores, giving $\text{AUROC}_{\boldsymbol y}((a,\boldsymbol r))=0$.
            Under $(b,\boldsymbol r)$ the positives are $\{1,2\}$ and the negative is $\{3\}$, and exactly one
            of the two positive-negative comparisons is correct (since $s(b)>s(f_3)>s(f_2)$), giving $\text{AUROC}_{\boldsymbol y}((b,\boldsymbol r))=\frac12$.
            Thus, AUROC increases under background $\boldsymbol r$ when changing $f_1$ from $a$ to $b$.

            \emph{Background 2.} Set
            \[
                f_2=(1+(|1|+\varepsilon)\tau,\;10^2),
                \quad
                f_3=(1,\;0.01^2).
            \]
            Then $e_2=1$ and $e_3=0$, and $s(f_2)$ is large while $s(f_3)$ is small.
            Consequently, under $(a,\boldsymbol r')$ the positive set is $\{2\}$ and negatives are $\{1,3\}$, and
            the positive score exceeds both negative scores, so $\text{AUROC}_{\boldsymbol y}((a,\boldsymbol r'))=1$.
            Under $(b,\boldsymbol r')$ the positives are $\{1,2\}$ and the negative is $\{3\}$, and both positive
            scores exceed the negative score, so again $\text{AUROC}_{\boldsymbol y}((b,\boldsymbol r'))=1$.
            Thus, under background $\boldsymbol r'$, the same change from $a$ to $b$ leaves AUROC unchanged, violating coordinate-independence.
        \end{proof}

\section{Generalization to the multiclass and multivariate setting}\label{app:multi}

    The purpose of this appendix is to demonstrate that our framework extends beyond the binary and univariate setting. We perform the analysis on three metrics and one utility each. Analogous results for additional metrics and decision families can be derived using the same template.

    \subsection{Multiclass classification}\label{app:multiclass}

        
        We consider probabilistic $K$-class classification models, $K > 2$, $f : \mathcal{X} \to \Delta^{K-1}$, mapping from the input space $\mathcal{X}$ to the $(K{-}1)$-simplex $\Delta^{K-1} := \{p \in [0,1]^K : \sum_{j=1}^K p^j = 1\}$, so the prediction for an instance is $f = (f^1, \ldots, f^K)$ with $f^j$ denoting the predicted probability that the true label $y \in \mathcal{Y} := [K] := \{1,\ldots,K\}$ equals $j$.
    
        As evaluation metrics, we consider the \emph{multiclass} NLL, BS, and ECE.
        We define
        \[
            \mathrm{NLL}(\boldsymbol{f}, \boldsymbol{y}) = \frac{1}{n}\sum_{i=1}^n \bigl(-\log f_i^{y_i}\bigr), \quad
            \mathrm{BS}(\boldsymbol{f}, \boldsymbol{y}) = \frac{1}{n}\sum_{i=1}^n \sum_{j=1}^K \bigl(f_i^j - \boldsymbol{1}_{y_i = j}\bigr)^2.
        \]
        The multiclass ECE uses the standard extension to vector-valued predictions via one-versus-all binning (see, e.g.,~\cite{guo_calibration_2017}).
        The NLL and BS are \textbf{pointwise-separable} (hence coordinate-independent) by definition.
        The multiclass ECE is \textbf{coordinate-dependent} by the same bin-coupling argument as in the binary case (Proposition~\ref{prop:coord_dependence}).
    
        As downstream task, we generalize the binary decision from Section~\ref{sect:bc_bd} to multiclass via a one-vs-rest decomposition.
        For each class $j \in [K]$, we consider the binary decision ``flag instance as class $j$ or not,'' with cost threshold $c \in (0,1)$.
        Defining the binary label $\tilde{y}^j := \boldsymbol{1}_{y = j}$, the Bayes act is to flag class $j$ whenever $f^j > c$.
        The pointwise utility of a prediction $f$ with true label $y$ for class $j$ at threshold $c$ is
        \[
            u_{j,c}(f, y) = -\tilde{y}^j (1-c) \boldsymbol{1}_{f^j \leq c} - (1-\tilde{y}^j) c \boldsymbol{1}_{f^j > c}.
        \]
        This is exactly the binary utility $u_c$ from Section~\ref{sect:bc_bd}, applied to the binary prediction--label pair $(f^j, \tilde{y}^j)$.
        The combined decision family is indexed by $\theta = (j,c) \in [K] \times (0,1)$, with dataset-level utility $U_{j,c}(\boldsymbol{f}, \boldsymbol{y}) = \frac{1}{n}\sum_{i=1}^n u_{j,c}(f_i, y_i)$. In particular, $U_{j,c}$ is \textbf{pointwise-separable}.
    
        \begin{proposition}
            The multiclass NLL and BS are decision-aligned w.r.t.\ $\{U_{j,c}\}_{j \in [K], c \in (0,1)}$ with
            $\pi_{\mathrm{NLL}}(j,c) = 1/c$ and $\pi_{\mathrm{BS}}(j,c) = 2$,
            and $h_{\boldsymbol{y}} = \mathrm{id}$ for all $\boldsymbol{y}$.
        \end{proposition}
    
        \begin{proof}
    
            \textbf{NLL:}
            Evaluating the per-instance expected negative utility under $\pi(j,c) = 1/c$ gives
            \begin{align*}
                \sum_{j=1}^K \int_0^1 -u_{j,c}(f_i, y_i)\,\frac{1}{c}\;\mathrm{d}c
                &= \int_{f_i^{y_i}}^{1} \frac{1-c}{c}\;\mathrm{d}c + \sum_{j \neq y_i} \int_0^{f_i^j} 1\;\mathrm{d}c \\
                &= \bigl[-\log f_i^{y_i} - 1 + f_i^{y_i}\bigr] + \sum_{j \neq y_i} f_i^j \\
                &= -\log f_i^{y_i} - 1 + f_i^{y_i} + (1 - f_i^{y_i}) = -\log f_i^{y_i}.
            \end{align*}
            In the last step we used the simplex constraint $\sum_{j \neq y_i} f_i^j = 1 - f_i^{y_i}$.
            Averaging over instances, $\mathrm{EU}_{\boldsymbol{y}}(\boldsymbol{f}) = \mathrm{NLL}(\boldsymbol{f}, \boldsymbol{y})$,
            so the NLL is decision-aligned with $h_{\boldsymbol{y}} = \mathrm{id}$.
    
            \textbf{BS:}
            The multiclass BS decomposes as
            $\mathrm{BS}(\boldsymbol{f}, \boldsymbol{y}) = \sum_{j=1}^K \mathrm{BS}^{(j)}(\boldsymbol{f}, \boldsymbol{y})$,
            where $\mathrm{BS}^{(j)}(\boldsymbol{f}, \boldsymbol{y})$
            is the binary BS for the $j$-th one-vs-rest problem.
            By Proposition~\ref{prop:buja}, each binary BS admits the integral representation~\eqref{eq:psr_bd_integral}
            with prior $\pi(c) = 2$, giving
            $\mathrm{BS}(\boldsymbol{f}, \boldsymbol{y}) = \sum_{j=1}^K \int_0^1 -U_{j,c}(\boldsymbol{f}, \boldsymbol{y}) \cdot 2\;\mathrm{d}c$,
            which matches Definition~\ref{def:decision_alignment} with $h_{\boldsymbol{y}} = \mathrm{id}$.
    
        \end{proof}
    
        As in the binary case, the NLL and BS implicit priors are pathological.
        The ECE is \emph{not} decision-aligned as a direct consequence of \cref{lem:separability_barrier}.
    
        \begin{corollary}
            The multiclass ECE is not decision-aligned w.r.t.\ $\{U_{j,c}\}_{j \in [K], c \in (0,1)}$.
        \end{corollary}
    
        
    

    \subsection{Multivariate regression}\label{app:multivariate}
    
        We consider probabilistic multivariate regression models with output dimension $D > 1$.
        We assume a Gaussian predictive distribution and consider models of the form $f : \mathcal{X} \to \mathbb{R}^D \times \mathbb{S}^D_{++}$, mapping from the input space $\mathcal{X}$ to the predictive mean $\mu \in \mathbb{R}^D$ and predictive covariance $\Sigma \in \mathbb{S}^D_{++}$ (the cone of $D \times D$ symmetric positive definite matrices), so $f = (\mu, \Sigma)$.
        The regression target is $y \in \mathcal{Y} := \mathbb{R}^D$.
    
        As evaluation metrics, we consider the \emph{multivariate} NLL, MSE, and the energy score (ES) \citep{gneiting_strictly_2007}.\footnote{
            To the best of our knowledge, there is no established notion of the ECE for multivariate regression.
            Therefore, we consider the ES, which is a more established metric in this context.
        }
        We define
        \begin{align*}
            \mathrm{NLL}(\boldsymbol{f}, \boldsymbol{y}) &= \frac{1}{n}\sum_{i=1}^n \left[\frac{D}{2}\log(2\pi) + \frac{1}{2}\log\det(\Sigma_i) + \frac{1}{2}(y_i - \mu_i)^\top \Sigma_i^{-1}(y_i - \mu_i)\right],\\
            \mathrm{MSE}(\boldsymbol{f}, \boldsymbol{y}) &= \frac{1}{n}\sum_{i=1}^n \|y_i - \mu_i\|^2.
        \end{align*}
        For a predictive distribution $f$ and observation $y$, the per-instance ES is defined as
        \[
            l_\mathrm{es}(f, y) = \mathbb{E}_{X \sim f}\|X - y\| - \tfrac{1}{2}\,\mathbb{E}_{X, X' \sim f}\|X - X'\|,
        \]
        where $X, X'$ are independent draws from $f$ and $\|\cdot\|$ denotes the Euclidean norm, and the dataset-level ES is the average.
        All three metrics are \textbf{pointwise-separable} by definition.
    
        As downstream task, we generalize selective prediction from Section~\ref{sect:reg_selpred} to multivariate output by applying the univariate selective-prediction rule \emph{coordinate-wise} to each standardized output dimension and averaging over dimensions.
        Concretely, for each output dimension $d \in [D]$ the decision-maker chooses an action in $\{\text{predict}, \text{abstain}\}$.
        To put dimensions of different scale on equal footing, costs are standardized by the marginal variance $v_d := \mathrm{Var}(y_d)$, so the standardized squared-error loss for a point prediction in dimension $d$ is $(y_d - \mu_d)^2 / v_d$, and the abstention cost is a dimensionless $\lambda \in \mathbb{R}_{>0}$.
        Equivalently, this is the univariate selective-prediction problem of Section~\ref{sect:reg_selpred} applied independently to each of the standardized targets $y_d / \sqrt{v_d}$.\footnote{
            Note that we apply the same scaling in the univariate case (see \cref{app:priors}), but we omitted this notation in the main text for brevity.
        }
        The Bayes act in dimension $d$ is to predict $\mu_d$ whenever the standardized predictive variance $\Sigma_{dd}/v_d$ is at most $\lambda$.
        The pointwise utility of a prediction $f = (\mu, \Sigma)$ with true outcome $y \in \mathbb{R}^D$ is
        \[
            u_\lambda(f, y) = -\frac{1}{D}\sum_{d=1}^D \left[\frac{(y_d - \mu_d)^2}{v_d}\, \boldsymbol{1}_{\Sigma_{dd} \leq \lambda v_d} + \lambda\, \boldsymbol{1}_{\Sigma_{dd} > \lambda v_d}\right].
        \]
        The dataset-level utility $U_\lambda(\boldsymbol{f}, \boldsymbol{y})$ is defined as the average pointwise utility.
        In particular, $U_\lambda$ is \textbf{pointwise-separable}.
    
        \begin{proposition}
            The multivariate NLL, MSE, and ES are not decision-aligned w.r.t.\ $\{U_\lambda\}_{\lambda \in \mathbb{R}_{>0}}$.
        \end{proposition}
    
        \begin{proof}
            Write $r_{i,d} := \Sigma_{i,dd}/v_d$ for the standardized marginal predictive variance and $e_{i,d} := (y_{i,d} - \mu_{i,d})^2 / v_d$ for the standardized squared error.
            Consider any prior $\pi$.
            The per-instance expected negative utility takes the form
            \begin{equation}
                s(f_i, y_i) = \frac{1}{D}\sum_{d=1}^D \bigl[e_{i,d}\, S(r_{i,d}) + G(r_{i,d})\bigr],
                \label{eq:mv_eu}
            \end{equation}
            where $S(t) := \int_t^\infty \pi(\lambda)\,\mathrm{d}\lambda$ and $G(t) := \int_0^t \lambda\,\pi(\lambda)\,\mathrm{d}\lambda$.
            In particular, $s(f_i, y_i)$ depends on $\Sigma_i$ only through its (standardized) diagonal $(r_{i,1}, \ldots, r_{i,D})$.
            We now show, for each metric, that a strictly monotone link to the expected utility cannot exist.
    
            \textbf{MSE:}
            The MSE does not depend on $\Sigma_i$ at all.
            Fix labels $\boldsymbol{y}$ and set $\mu_i = y_i$ for all $i$, so $\mathrm{MSE}(\boldsymbol{f}, \boldsymbol{y}) = 0$ regardless of the covariances.
            Then $e_{i,d} = 0$ for all $i, d$, and~\eqref{eq:mv_eu} reduces to $s(f_i, y_i) = \frac{1}{D}\sum_d G(r_{i,d})$.
            Since $\pi$ is nonzero, there exists an interval $(a, b) \subset \mathbb{R}_{>0}$ with $\int_a^b \pi(\lambda)\,\mathrm{d}\lambda > 0$, and for $t_1, t_2 \in (a, b)$ with $t_1 < t_2$ and $t_1, t_2 \ne 1$ we have $G(t_2) > G(t_1)$.
            Setting $\Sigma_i = t v_1\, e_1 e_1^\top + \sum_{d \geq 2} v_d\, e_d e_d^\top$ (where $e_d$ is the $d$-th standard basis vector), we obtain $r_{i,1} = t$ and $r_{i,d} = 1$ for $d \geq 2$.
            Choosing $t \in \{t_1, t_2\}$ yields two prediction vectors with identical MSE (both zero) but different expected utilities, contradicting the existence of a strictly increasing $h_{\boldsymbol{y}}$.
    
            For the NLL and ES, we use the following shared construction.
            Since all three quantities (NLL, ES, and expected utility) are pointwise-separable, fixing $n = 2$, any labels $\boldsymbol{y}$, and $f_2$ arbitrarily, the strictly increasing function $h_{\boldsymbol{y}}$ would establish a strictly monotone link between the per-instance metric $l(f_1, y_1)$ and the per-instance expected utility $s(f_1, y_1)$.
            In particular, $l(f_1, y_1) \neq l(f_1', y_1)$ must imply $s(f_1, y_1) \neq s(f_1', y_1)$.
            Fix $\mu_1 = y_1$ and consider two covariance matrices that share the same diagonal but differ in their off-diagonal structure:
            \[
                \Sigma_A = \mathrm{diag}(v_1, v_2, \ldots, v_D),
                \qquad
                \Sigma_B = \Sigma_A + \rho\sqrt{v_1 v_2}\,\bigl(e_1 e_2^\top + e_2 e_1^\top\bigr),
            \]
            for some $\rho \in (0, 1)$ and $v_1 \ne v_2$ where $e_d$ denotes the $d$-th standard basis vector.
            Both matrices are symmetric positive definite (the leading $2 \times 2$ block of $\Sigma_B$ has determinant $v_1 v_2 (1 - \rho^2) > 0$).
            They share the same diagonal entries, hence the same standardized diagonal $(1, 1, \ldots, 1)$, and since $\mu_1 = y_1$ also $e_{1,d} = 0$ for all $d$.
            By~\eqref{eq:mv_eu}, $s((y_1, \Sigma_A), y_1) = s((y_1, \Sigma_B), y_1) = \frac{1}{D}\sum_d G(1)$.
    
            \textbf{NLL:}
            Since $\mu_1 = y_1$, the per-instance NLL reduces to $\frac{D}{2}\log(2\pi) + \frac{1}{2}\log\det(\Sigma_1)$.
            We have $\det(\Sigma_A) = \prod_d v_d$ and $\det(\Sigma_B) = v_1 v_2 (1 - \rho^2) \prod_{d \geq 3} v_d < \det(\Sigma_A)$,
            so the NLL takes different values for $\Sigma_A$ and $\Sigma_B$ while the expected utility is the same.
            This contradicts the strictly monotone link.
    
            \textbf{ES:}
            Since $\mu_1 = y_1$, the predictive distribution is $\mathcal{N}(y_1, \Sigma_1)$ and for $W := X - y_1$,
            \[
                l_\mathrm{es}\bigl((y_1,\Sigma_1), y_1\bigr) = (1 - 1/\sqrt{2})\,\mathbb{E}_{W \sim \mathcal{N}(0, \Sigma_1)}\|W\|.
            \]
            Writing $\eta_1, \ldots, \eta_D$ for the eigenvalues of $\Sigma_1$, by its eigendecomposition $\|W\|^2 \stackrel{d}{=} \sum_{d=1}^D \eta_d Z_d^2$ where $Z_1, \ldots, Z_D$ are independent standard normal variables.
            The map $\Phi = (\eta_1, \ldots, \eta_D) \mapsto \mathbb{E}\sqrt{\sum_{d=1}^D \eta_d Z_d^2}$ is strictly concave on $\mathbb{R}_{>0}^D$ (as the expectation of the strictly concave map $\eta \mapsto \sqrt{\sum_d \eta_d z_d^2}$, which is concave for every realization $z$ as the composition of the concave, non-decreasing function $\sqrt{\cdot}$ with a linear function of $\eta$).
            The eigenvalues of $\Sigma_A$ are $(v_1, v_2, \ldots, v_D)$, while those of $\Sigma_B$ are
            $\bigl(\tfrac{v_1+v_2}{2} + \Delta,\, \tfrac{v_1+v_2}{2} - \Delta,\, v_3, \ldots, v_D\bigr)$
            with $\Delta := \sqrt{(\tfrac{v_1-v_2}{2})^2 + \rho^2 v_1 v_2} > |\tfrac{v_1-v_2}{2}|$.
            Hence $\Sigma_B$'s eigenvalues are a strict mean-preserving spread of $\Sigma_A$'s within the first two coordinates.
            By symmetry of $\Phi$ in its first two arguments, $\Phi(a) = \Phi(a')$ where $a' := (v_2, v_1, v_3, \ldots, v_D)$, and similarly $\Phi(b) = \Phi(b')$ where $b'$ swaps the first two entries of $b$.
            Both pairs share the same midpoint $\bar a := \tfrac{1}{2}(a + a') = \tfrac{1}{2}(b + b') = (\tfrac{v_1+v_2}{2}, \tfrac{v_1+v_2}{2}, v_3, \ldots, v_D)$.
            Since $v_1 \neq v_2$ implies $a \neq a'$, and $\Delta > 0$ implies $b \neq b'$, strict concavity of $\Phi$ gives $\Phi(\bar a) > \Phi(a)$ and $\Phi(\bar a) > \Phi(b)$.
            Moreover, $a$ and $b$ lie on the line through $\bar a$ in the direction $(1, -1, 0, \ldots, 0)$, with $a$ at distance $|\tfrac{v_1-v_2}{2}|$ from $\bar a$ and $b$ at distance $\Delta > |\tfrac{v_1-v_2}{2}|$, so $a = (1-t)\bar a + t b$ for $t := |\tfrac{v_1-v_2}{2}|/\Delta \in (0, 1)$ (up to swapping $b$ and $b'$, which by symmetry does not affect $\Phi$).
            By strict concavity of $\Phi$ along this segment,
            \[
                \Phi(a) > (1-t)\Phi(\bar a) + t \Phi(b) > (1-t)\Phi(b) + t \Phi(b) = \Phi(b),
            \]
            where the second inequality uses $\Phi(\bar a) > \Phi(b)$.
            Hence $\mathbb{E}_{W \sim \mathcal{N}(0, \Sigma_A)}\|W\| = \Phi(a) > \Phi(b) = \mathbb{E}_{W \sim \mathcal{N}(0, \Sigma_B)}\|W\|$, so the ES takes different values for $\Sigma_A$ and $\Sigma_B$ while the expected utility is the same.
            This contradicts the strictly monotone link.
    
        \end{proof}


\begin{figure}[ht]
    \begin{center}
    \centerline{\includegraphics[width=.9\textwidth]{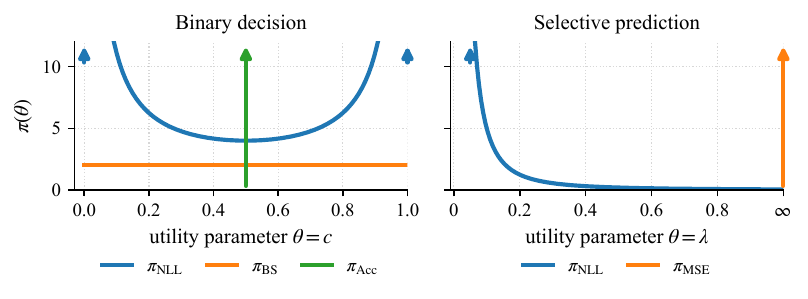}}
        \caption{
          Illustrations of the pathological priors identified in \cref{sect:theory} for the binary decision \emph{(left)} and selective prediction \emph{(right)} utility.
          The arrows indicate where the curves diverge.
        }
        \label{fig:path_priors}
    \end{center}
    \vskip -0.2in
\end{figure}

\section{PWU documentation}\label{app:pwu_card}

    We now document our PWU metrics similarly to the model card framework \cite{mitchell_model_2019}.

    \paragraph{Metric details}
    
        The PWU metrics are a family of proper scoring
        rules for evaluating probabilistic predictions, introduced in this paper
        (Section~\ref{sect:method}).
        A PWU metric $M_\pi$ is parameterized by a
        decision family $\{U_\theta\}_{\theta \in \Theta}$ and a prior
        $\pi$ over $\Theta$, and is defined as
        $M_\pi(f, y) := \int_\Theta -U_\theta(f, y)\, \pi(\theta)\, \mathrm{d}\theta$
        (Equation~\eqref{eq:pwu_def}).
        The construction
        applies to any decision family with finite expected utility under $\pi$.
        We instantiate four PWU metrics in the binary
        and univariate setting: $M_{\pi_c}$ (binary decisions),
        $M_{\pi_\lambda}$ (selective prediction), $M_{\pi_k}$ (top-$k$ selection
        in classification), and $M_{\pi_\phi}$ (top-$k$ selection in regression
        with risk aversion).
        We further introduce two analogous PWU instantiations
        for the multiclass and multivariate setting in
        Appendix~\ref{app:multi}.
    
    \paragraph{Inputs and outputs}
    
        A PWU metric takes a first-order probabilistic prediction $f$ and a label
        $y$, and returns a real-valued score (lower is better).
        
    \paragraph{Intended use}
    
        PWU metrics are intended for general-purpose uncertainty quantification benchmarking when
        no specific downstream task is given.
        They are
        designed as a complement to, not a replacement for, conventional UQ
        metrics: each PWU encodes a single decision family, so we recommend
        reporting a \emph{set} of PWUs alongside standard scoring rules to cover
        a range of plausible downstream scenarios.
        
    \paragraph{Prior elicitation}
    
        The prior $\pi$ controls which regions of the decision space the metric
        emphasizes. We justify our specific prior choices and provide general
        elicitation guidance in Appendix~\ref{app:priors}; sensitivity to prior
        misspecification is analyzed in Appendix~\ref{app:sens_analysis}, where we
        find that PWU metrics retain robust positive utility alignment even under
        extreme prior perturbations.

    \paragraph{Limitations}
    
        No single PWU metric covers all downstream objectives, therefore, \emph{several} PWU metrics are required to broadly represent a model's decision-making ability.
        Prior elicitation introduces a modeling step that generic metrics avoid.
        Our PWUs’ priors may be misspecified in certain decision problems, which may degrade performance (even though we observed in Appendix~\ref{app:sens_analysis} that PWU metrics exhibit robust positive alignment even under severe prior misspecification).
        The scope of PWU metrics is currently restricted to first-order probabilistic predictions
        (Appendix~\ref{app:intro}); extension to second-order UQ evaluation is
        left for future work.
        Computationally, evaluating $M_\pi$ requires
        integrating the negative utility against $\pi$, which requires numerical integration for
        prior choices that do not admit a closed-form $M_\pi$.
        Mathematically, only utility--prior pairs are admitted that allow for a finite expected utility and hence a well-defined PWU.

    \paragraph{Out-of-scope use}

        PWU metrics are not intended as training objectives---some are non-separable, making mini-batching ill-defined, some are non-differentiable (top-$k$ selection, thresholding), and the integral adds computational cost.
        They are also not intended as a fairness or
        robustness diagnostic: alignment with a chosen utility says nothing about
        performance on subpopulations not captured by that utility.

    \paragraph{Broader impact}

        Our work proposes evaluation metrics rather than new models, datasets,
        or deployment systems, so its societal impact is indirect and mediated
        by how UQ benchmarks influence model selection.
        We see two main
        directions of impact.
        On the positive side, PWU metrics encourage UQ benchmarking that is
        explicit about the downstream decision context and the assumed
        distribution over decision parameters.
        Standard UQ metrics encode
        implicit and often pathological priors over downstream costs; making
        these priors explicit can lead to model rankings that better track
        realized utility in deployed systems---particularly in cost-sensitive
        settings such as medical triage, credit approval, and selective
        prediction, where mismatches between the evaluation metric and the
        actual decision cost can have material consequences.
        On the negative side, PWU metrics depend on a chosen prior, and a
        poorly elicited or strategically chosen prior can shift model rankings
        in ways that favor certain models without that choice being apparent
        to downstream users. While our sensitivity analysis
        (Appendix~\ref{app:sens_analysis}) shows that PWU metrics remain robust
        to substantial prior misspecification, transparent reporting of the
        prior used---ideally alongside conventional metrics and a range of
        plausible alternative priors, as we recommend in
        Section~\ref{sect:method}---is important to prevent the metric from being
        used to justify a predetermined conclusion. PWU metrics also do not
        address fairness or distributional concerns: alignment with a chosen
        utility says nothing about performance on subpopulations not captured
        by that utility, and a benchmark that selects models well on average
        may still select models that perform poorly on minority groups.
        We do not foresee a direct path from this work to malicious use:
        PWU metrics evaluate existing predictions and do not enable new
        capabilities such as generation, surveillance, or large-scale
        inference attacks.

\section{Choosing PWU priors}\label{app:priors}

    Priors should be designed through consultation of the domain literature
    and expert knowledge.
    For each decision task, one should ask what range of decision parameters
    is plausible, and where the probability mass should lie.
    This process is neither straightforward nor unique---prior elicitation is
    a complex research area in it's own right (see, e.g.,~\cite{mikkola_prior_2024}).
    Fortunately, PWU metrics are robust to misspecification; even PWUs
    corresponding to misspecified priors perform better than metrics
    corresponding to pathological priors, as demonstrated in our sensitivity
    analysis in \cref{app:sens_analysis}.
    
    In the remainder of this section, we provide concrete prior-elicitation guidance for benchmark designers who wish to construct their own PWU
    metrics.
    \cref{app:priors_protocol} describes a four-step procedure for constructing a PWU metric in general,
    \cref{app:priors_defaults} lists default prior families,
    \cref{app:priors_diagnostics} provides a diagnostic prior-choice checklist,
    and \cref{app:priors_ours} documents the prior choices used in our
    experiments as worked examples of the procedure.

    \subsection{A four-step PWU construction procedure}\label{app:priors_protocol}

        We recommend the following procedure for constructing a PWU prior.

        \paragraph{Step 1: Identify the decision families}
            Specify the action sets for potential downstream use
            and the corresponding utility (or loss) structure.
            Express this as parameterized families $\{U_\theta\}_{\theta \in \Theta}$.

        \paragraph{Step 2: Identify the parameter spaces \texorpdfstring{$\boldsymbol \Theta$}{Theta}}
            Determine the units, range, and natural scale of $\theta$.
            Is it bounded (cost ratios, selection fractions) or unbounded
            (abstention costs, risk-aversion coefficients)?
            For unbounded parameters, consider whether a natural rescaling exists
            (e.g., normalizing abstention costs by the empirical variance of
            $\boldsymbol{y}$ as we do in \cref{sect:reg_selpred}).

        \paragraph{Step 3: Specify prior families}
            Select parametric families for $\pi_\theta$ using the defaults in
            \cref{app:priors_defaults}.
            For most decision parameters that arise in practice, a Beta
            distribution on a bounded or rescaled domain suffices.

        \paragraph{Step 4: Set the prior's mode and concentration}
            Pin down the mode from a single domain anchor (e.g., ``a typical
            false negative is roughly $10\times$ as costly as a false positive''
            anchors $c$ at $\approx 0.1$).
            Set the concentration according to confidence in that
            anchor: a tight prior reflects strong belief in the modal value;
            a diffuse prior reflects broader uncertainty about the deployment
            context.
            We recommend verifying the choice using the diagnostics in
            \cref{app:priors_diagnostics}.

    \subsection{Default prior families}\label{app:priors_defaults}

        \cref{tab:prior_defaults} summarizes default prior families for
        decision parameters that commonly arise in UQ benchmarking.

        \begin{table}[h]
            \caption{
                Default prior families for common decision parameter types.
                For unbounded scale parameters, we recommend rescaling to a
                bounded domain (e.g., dividing by the empirical variance of
                $\boldsymbol{y}$) and using a Beta prior, as we do for the
                abstention cost $\lambda$ in \cref{sect:reg_selpred}.
            }
            \label{tab:prior_defaults}
            \vspace{0.7em}
            \centering
            \begin{tabular}{lll}
                \toprule
                Parameter type & Domain & Default family \\
                \midrule
                Cost ratio / mixing weight & $(0,1)$ & $\text{Beta}(\alpha, \beta)$ \\
                Selection fraction $k/n$ & $(0,1)$ & $\text{Beta}(\alpha, \beta)$ \\
                Abstention cost (rescaled) & $(0,1)$ & $\text{Beta}(\alpha, \beta)$ \\
                Risk-aversion $\gamma$ (rescaled) & $(0,1)$ & $\text{Beta}(\alpha, \beta)$ \\
                Unbounded scale parameter & $(0, \infty)$ & $\text{LogNormal}(\mu, \sigma^2)$ or $\text{Gamma}(\alpha, \beta)$ \\
                \bottomrule
            \end{tabular}
        \end{table}

        For a Beta$(\alpha, \beta)$ prior with $\alpha, \beta > 1$, the mode is
        $(\alpha - 1) / (\alpha + \beta - 2)$, and the concentration grows with
        $\alpha + \beta$.
        The mode can therefore be fixed at a domain-motivated value and
        the concentration can be controlled independently by scaling $\alpha$ and
        $\beta$ together.

    \subsection{Diagnostics}\label{app:priors_diagnostics}

        Before reporting PWU results, we recommend the following checks.

        \paragraph{Visualization}
            Plot the density of $\pi_\theta$ and verify that it concentrates
            mass where expected.
            We do this in \cref{fig:priors}.

        \paragraph{Quantile check}
            Compute the 5th, 50th, and 95th percentiles of $\pi_\theta$.
            Each should correspond to a plausible deployment scenario.
            For example, under our $\pi_c = \text{Beta}(2, 10)$, the 5th, 50th,
            and 95th percentiles of $c$ are approximately $0.033$, $0.148$, and
            $0.364$, corresponding to false-negative-to-false-positive cost
            ratios of roughly $29{:}1$, $5.8{:}1$, and $1.7{:}1$. If any of
            these quantiles seems implausible for the target deployment
            context, the prior should be reconsidered.

        \paragraph{Degeneracy check}
            Verify that $\pi_\theta$ does not place substantial mass on regions
            where the decision becomes trivial or where UQ is least needed
            (e.g., $c \to 0$ or $c \to 1$ for binary decisions, $\lambda \to 0$
            for selective prediction).
            This is exactly the pathology we identify in \cref{sect:theory} for
            the implicit priors of conventional metrics.
            If the prior is degenerate in this sense, truncate or otherwise reshape it.

        \paragraph{Sensitivity check}
            Re-run the benchmark with at least one perturbed prior (e.g.,
            shifting the mode by $50\%$) and verify that conclusions are
            qualitatively stable.
            This makes the dependence of results on the prior transparent and
            guards against overfitting the prior to a particular outcome.
            Our sensitivity analysis in \cref{app:sens_analysis} provides a
            template for such checks.

    \subsection{Worked examples: prior choices used in our experiments}\label{app:priors_ours}

        We document our specific prior choices below as worked examples of
        the procedure in \cref{app:priors_protocol}.
        All priors were specified \emph{a priori} from domain heuristics and
        held fixed across all experiments; they were not selected, tuned, or
        adjusted based on test-set performance, model rankings, or alignment
        results.
        The sensitivity analysis in \cref{app:sens_analysis} characterizes
        robustness to misspecification of these choices.

        \paragraph{Classification}

            For the binary decision problem, our PWU metric coincides with the
            well-known PSR integral representation \eqref{eq:psr_bd_integral}
            discussed in \cref{sect:bc_bd}.
            For the form \eqref{eq:psr_bd_integral},
            \citet{buja_loss_2005} proposed PSRs that use
            $\text{Beta}(\alpha,\beta)$ priors over $c$.
            We follow this approach and set $\alpha=2$ and $\beta=10$.
            These are reasonable choices as the result is a distribution with
            modal value $c=0.1$ and the majority of probability mass on small
            values of $c$.
            This reflects the cost asymmetry in many binary decisions where
            the cost of a false negative $1-c$ is much higher than the cost
            of a false positive $c$.
            In the case of top-$k$ selection, we place a prior over the
            fraction $k/n$.
            For simplicity, we denote this as $\pi_k(k)$, but it should be
            understood as $\pi_{k/n}(k/n)$.
            We again propose a Beta distribution, this time targeting a lower
            mode of $0.01$ (pick top $1\%$) as many top-$k$ selection studies
            are about picking only a small number of items from a large
            dataset.
            Therefore, we set $\alpha=1.2$ and $\beta=20.8$.

        \paragraph{Regression}

            For selective prediction, we consider the abstention cost to be
            roughly $10\%$ of the data variance ($\alpha=2,\beta=10$), so
            that abstention is often cheaper than the typical squared error
            but not negligible.
            For top-$k$ selection, we use the same prior as in classification
            for $k/n$, and say that the distribution for $\gamma$ is
            parameterized by $\alpha=2$ and $\beta=6$.\footnote{
                Note that although the natural domain of $\lambda$ and
                $\gamma$ is $\mathbb{R}_{>0}$, we place priors on $[0,1]$;
                in practice, samples are rescaled by multiplying with the
                empirical variance of $\boldsymbol y$.
            }
            This corresponds to a moderate risk-aversion ($\gamma=0$
            corresponds to risk-neutrality).
            We write $\pi_c,\pi_\lambda$ for $\text{Beta}(2,10)$, $\pi_{k}$
            for $\text{Beta}(1.2,20.8)$, and $\pi_\phi$ for the product
            density of $\text{Beta}(1.2,20.8)$ and $\text{Beta}(2,6)$
            (assuming independence of $k/n$ and $\gamma$).

        \paragraph{On cross-validating the prior}

            We do not cross-validate or otherwise tune the priors on labeled
            data.
            PWU priors are not hyperparameters in the conventional sense---they
            encode a belief about the downstream decision context, which is
            exogenous to the data.
            Selecting a prior to maximize alignment on a validation split
            would conflate ``the prior best representing beliefs about a decision scenario '' with
            ``the prior under which a particular model looks best,''
            reintroducing precisely the misalignment problem PWU metrics are
            designed to avoid.
            The appropriate analog of cross-validation is the sensitivity
            analysis in \cref{app:sens_analysis}, which characterizes
            robustness to misspecification rather than tuning the prior.
   
\section{Details on our experiments}\label{app:experiments}

    In both the controlled experiments and the case studies, we aim to measure the \emph{alignment} between model rankings by a metric $M$ and by a utility $U_\theta$.
    Throughout all experiments, we fit the same models, use the same metrics $M$ (all metrics from \cref{sect:theory} and our PWU metrics from \cref{sect:method}), and apply the same alignment score computation.
    
    \paragraph{Model implementations}

        For binary classification, we implement logistic regression (LogReg), random forests (RFs), gradient boosting (GB), sparse variational Gaussian processes (GPs)~\cite{hensman_scalable_2015}, multilayer perceptrons (MLPs), TabPFN~\cite{hollmann_tabpfn_2023}, the FT-Transformer~\cite{gorishniy_revisiting_2021}, SAINT~\cite{somepalli_saint_2022}, the ResNet-MLP~\cite{gorishniy_revisiting_2021}, and a deep ensemble~\cite{lakshminarayanan_simple_2017}.
        For regression, we implement the same suite of models, replacing LogReg with linear regression (LinReg), GB with natural gradient boosting (NGB)~\cite{duan_ngboost_2020}, and MLPs with heteroscedastic MLPs.
        Our codebase is primarily based on the following packages: \texttt{scikit-learn} \cite{scikit-learn} (BSD 3-Clause License), \texttt{gpytorch} \cite{gardner_gpytorch_2018} (MIT license), \texttt{ngboost} \cite{duan_ngboost_2020} (Apache-2.0 license), and \texttt{tabpfn} \cite{hollmann_tabpfn_2023} (Prior Labs License).

        \emph{Linear and Logistic Regression}
        We use standard L2-regularized logistic regression and ordinary least squares, where the predictive variance is analytically available and can be estimated from the data.

        \emph{Random Forests}
        Each ensemble contains 200 trees.
        Predictive uncertainty combines inter-tree variability (epistemic) and out-of-bag residual error (aleatoric).
        
        \emph{(Natural) Gradient Boosting}
        We employ gradient boosting for classification and \texttt{NGBoost}~\cite{duan_ngboost_2020} for regression, using a Gaussian likelihood.
        
        \emph{Sparse Variational Gaussian Processes}
        We use the formulation of~\citet{hensman_scalable_2015}, employing an RBF kernel with automatic relevance determination.
        The model is trained by maximizing the variational evidence lower bound (ELBO).
        
        \emph{Multilayer Perceptrons}
        We train two-layer ReLU networks (128 units per layer) with early stopping.
        For regression, a heteroscedastic variant outputs both mean and variance, and is trained by maximum likelihood.
        
        \emph{TabPFN}
        We subsample the training sets to $10{,}000$ instances when necessary, as required by the model.
        For regression, predictive uncertainty is obtained from the $15.87\%$ and $84.13\%$ quantiles of the predictive distribution, corresponding to $\pm 1$ standard deviation under a Gaussian assumption.
        
        \emph{FT-Transformer}
        We use a $[\mathrm{CLS}]$ token, $d_{\mathrm{token}} = 64$, $4$ attention heads, $3$ layers, and dropout $0.1$.
        For regression, the model outputs both a mean $\mu$ and a variance $\sigma^2$ (via a softplus activation), trained by minimizing the Gaussian negative log-likelihood.
        
        \emph{SAINT}
        Architecture and training details match those of the FT-Transformer.
        For regression, a heteroscedastic head is used, analogous to the FT-Transformer.
        
        \emph{ResNet-MLP}
        We use $3$ residual blocks with hidden dimension $128$, batch normalization, GELU activations, and dropout $0.1$.
        Each block contains two linear layers with a skip connection.
        For regression, a heteroscedastic variant with separate mean and variance heads is trained by Gaussian negative log-likelihood.
        
        \emph{Deep Ensemble}
        Following~\citet{lakshminarayanan_simple_2017}, we train an ensemble of $5$ MLPs with distinct random initializations.
        For classification, the predictive probability is the average of the member probabilities.
        For regression, each member is a heteroscedastic MLP outputting mean and variance; the ensemble predictive mean and variance are obtained by the standard mixture-of-Gaussians formulas.

    \paragraph{Alignment score computation}

        We now explain how we compute the alignment score for one metric--utility pair $(M,U_\theta)$ on one dataset $\mathcal{D}$, and this follows equally for all considered metric--utility pairs on each dataset.
        We train and test each model $\mathcal{M}_j$, $j=1,\dots,5$, on $\mathcal{D}$ using $5$-fold cross-validation (80/20 train--test partition) with fixed random seed, and repeat this train/test loop for $100$ different seeds (we have a slightly different setup in the electricity market case study, see below).
        So for each model $\mathcal{M}_j$, we obtain $100 \times 5$ metric/utility values:
        $\{m_{\text{rep},\text{fold}}\}_{\mathcal{M}_j}, \{u_{\text{rep},\text{fold}}\}_{\mathcal{M}_j}$, $\text{rep}=1,\dots,100$, $\text{fold}=1,\dots,5$.
        In each of the $100$ repeats, we average the metric/utility value over the $5$ folds to obtain a stable metric/utility value for each repeat:
        $\{\bar{m}_{\text{rep}}\}_{\mathcal{M}_j}, \{\bar{u}_{\text{rep}}\}_{\mathcal{M}_j}$.
        Using these $100$ metric/utility values, we can determine the corresponding model rankings
        $\boldsymbol r_{M,\text{rep}}=(r_{M,\text{rep},\mathcal{M}_1},\dots,r_{M,\text{rep},\mathcal{M}_5})$, $\boldsymbol r_{U_\theta,\text{rep}}=(r_{U_\theta,\text{rep},\mathcal{M}_1},\dots,r_{U_\theta,\text{rep},\mathcal{M}_5})$, $r_{M,\text{rep},\mathcal{M}_j},r_{U_\theta,\text{rep},\mathcal{M}_j} \in \{0,\dots,5\}$.
        We then measure the \emph{alignment} of $\boldsymbol r_{M,\text{rep}},\boldsymbol r_{U_\theta,\text{rep}}$ by computing Kendall's $\tau$.
        Through this procedure, we obtain $100$ metric-utility alignment scores $\{\tau_{(M,U_\theta),\text{rep}}\}$.

    \paragraph{Choice of Kendall's \texorpdfstring{$\boldsymbol \tau$}{tau}}

        We aim to evaluate decision-alignment (\cref{def:decision_alignment}) between metrics and downstream utilities.
        \cref{prop:preservation} shows that decision-alignment is equivalent to strict order- and tie-preservation between the rankings induced by a metric  and a utility.
        Therefore, evaluating decision-alignment reduces to assessing agreement between the two rankings.
        Standard measures for ranking similarity include Spearman’s $\rho$ and Kendall’s $\tau$.
        We choose Kendall's $\tau$ because it is robust and interpretable for small rankings, as in our case.
        For example, $\tau_{(M,U_\theta)} = 0.8$ means that $80\%$ of all pairwise model comparisons induced by the metric $M$ agree with the downstream utility $U_\theta$.
        Other possibilities to compare two ranking vectors are top-$k$ agreement or cardinal comparisons (e.g., Pearson's $\rho$).
        As discussed above, these do not evaluate decision-alignment in the sense of \cref{prop:preservation}.
        That said, these other metrics can be valid for different evaluation goals: top-$k$ if the objective is identifying the best models, and cardinal correlations (e.g., Pearson’s $\rho$) if magnitudes matter.
        So as an insightful ablation, we also report the top-$1$ and top-$3$ agreement and Pearson's $\rho$ for our benchmark experiments in Figures \ref{fig:benchmark_top1}, \ref{fig:benchmark_top3}, and \ref{fig:benchmark_pearson}.
        In each case, our PWU metrics perform better overall compared to common UQ metrics.

\begin{figure}[ht]
    \begin{center}
    \centerline{\includegraphics[width=.9\textwidth]{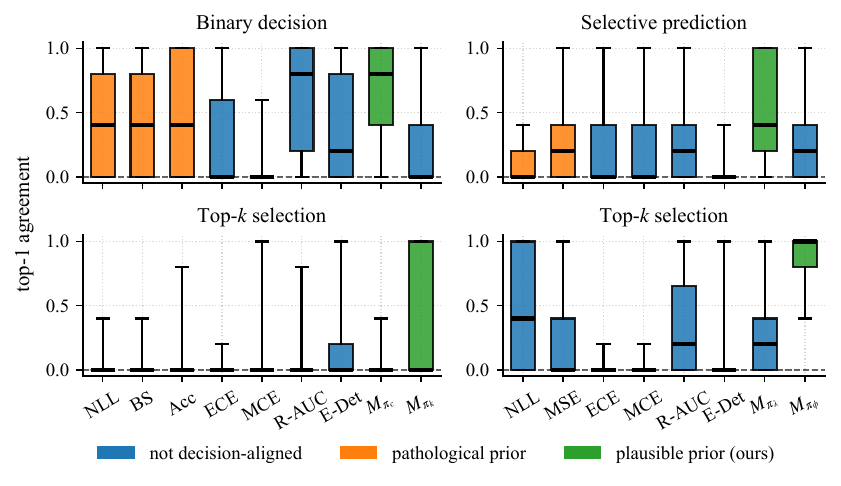}}
        \caption{
          Top-$1$ agreement in classification \emph{(left)} and regression \emph{(right)}, averaged over our five datasets.
          The coloring corresponds to our theoretical findings from \cref{sect:theory}.
          The PWU metrics perform best in identifying the best model for their corresponding utility families.
        }
        \label{fig:benchmark_top1}
    \end{center}
    \vskip -0.2in
\end{figure}
\begin{figure}[ht]
    \begin{center}
    \centerline{\includegraphics[width=.9\textwidth]{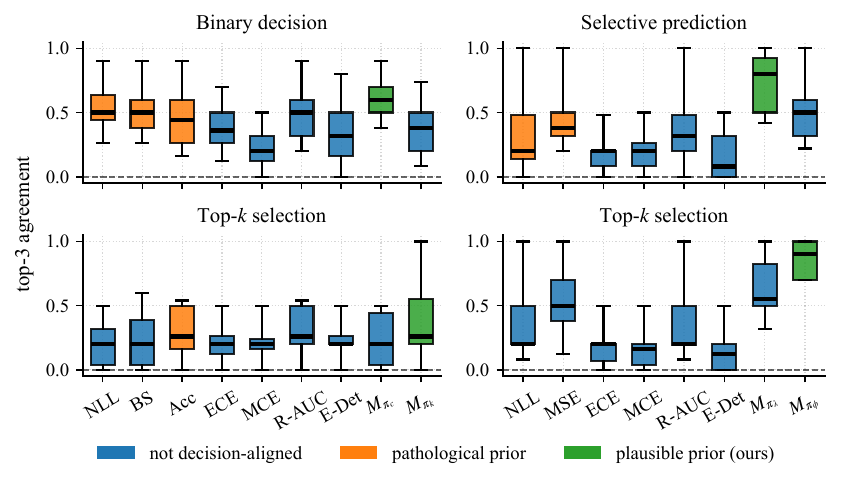}}
        \caption{
          Top-$3$ agreement in classification \emph{(left)} and regression \emph{(right)}, averaged over our five datasets.
          The coloring corresponds to our theoretical findings from \cref{sect:theory}.
          The PWU metrics perform best in identifying the top-$3$ models for their corresponding utility families.
        }
        \label{fig:benchmark_top3}
    \end{center}
    \vskip -0.2in
\end{figure}
\begin{figure}[ht]
    \begin{center}
    \centerline{\includegraphics[width=.9\textwidth]{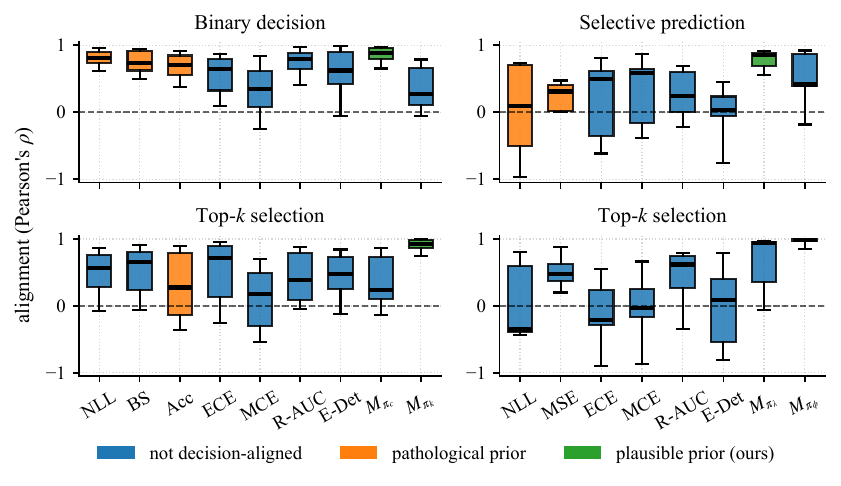}}
        \caption{
          Metric--utility correlation (Pearson's $\rho$) in classification \emph{(left)} and regression \emph{(right)}, averaged over our five datasets.
          The coloring corresponds to our theoretical findings from \cref{sect:theory}.
          The PWU metrics exhibit the strongest cardinal metric correlation with respect to their corresponding utility families.
        }
        \label{fig:benchmark_pearson}
    \end{center}
    \vskip -0.2in
\end{figure}

    \subsection{Experiments on benchmark datasets}

        \paragraph{Datasets}

            We use the following benchmark datasets from the UCI repository~\cite{kelly_uci_2025}:
            bank marketing~\cite{moro_bank_2014}, heart disease~\cite{janosi_heart_1989}, ionosphere~\cite{sigillito_ionosphere_1989}, mushroom~\cite{schlimmer_mushroom_1981}, and sonar~\cite{sejnowski_connectionist_1988} for classification, and air quality~\cite{vito_air_2008}, auto MPG~\cite{quinlan_auto_1993}, energy efficiency~\cite{tsanas_energy_2012}, power plant~\cite{tfekci_combined_2014}, and wine quality~\cite{cortez_wine_2009} for regression.
            All datasets use a CC BY 4.0 license.

        \paragraph{Utilities}

            We measure the alignment of our metrics $M$ to the four utilities $U_\theta$, $\theta \in \{c,k,\lambda,\phi\}$, considered in \cref{sect:theory}.
            We do not settle on one specific value for $\theta$, but sample $5$ different $\theta_j$ from the corresponding priors $\pi_\theta$, $\theta \in \{c,k,\lambda,\phi\}$, that we define in \cref{app:priors}.
            We then average the alignment scores of a metric $M$ with $U_{\theta_j}$, $\tau_{(M,U_{\theta_j})}$ across the respective decision family to obtain $\tau_{(M,U_{\theta})}=\tfrac 15 \sum_{j=1}^5 \tau_{(M,U_{\theta_j})}$.
        
            We report the median and 5th, 95th percentiles for each dataset in Tables~\ref{tab:benchmark_bc_bd}, \ref{tab:benchmark_bc_topk}, \ref{tab:benchmark_reg_selpred}, \ref{tab:benchmark_reg_topk}, and visualize the results, averaged over all datasets, in \cref{fig:benchmark}.

    \subsection{Applied case studies}\label{app:casestudies}

\begin{figure}[ht]
    \begin{center}
    \centerline{\includegraphics[width=.9\textwidth]{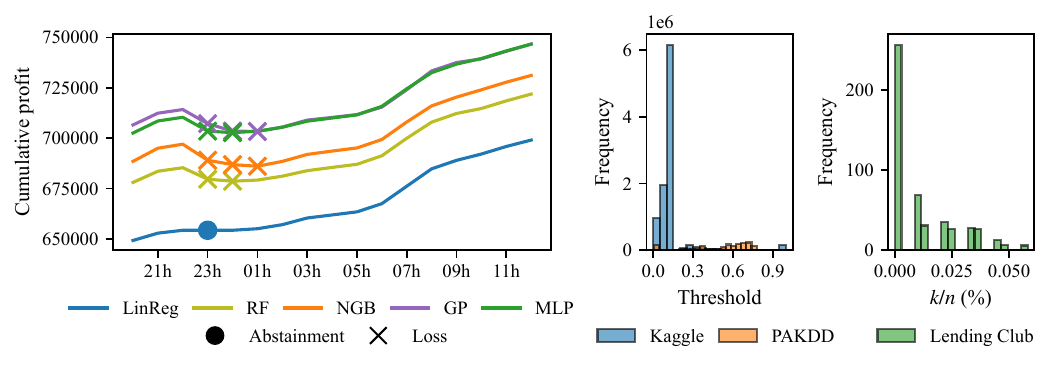}}
        \caption{
          Illustrations of our three case studies:
          cumulative electricity market trading profit over the night of Dec 03 to Dec 04, 2024, showing when which models led to bid abstention or loss \emph{(left)}, Bayes-optimal decision thresholds for the credit customers from our two datasets \emph{(center)}, and the share of accepted loans with a fixed budget $B$ for our 500 different test sets \emph{(right)}.
          We can see that the PAKDD thresholds and the $k/n$ distribution differ significantly from our chosen priors (see \cref{fig:priors})---note that $k/n$ is given in percent.
          Nevertheless, our PWU metrics perform well (see \cref{fig:credit_p2p}), demonstrating their robustness under prior misspecification.
        }
        \label{fig:cumsum_thresholds_ks}
    \end{center}
    \vskip -0.2in
\end{figure}

        \paragraph{Electricity market bidding}

            We consider a wind farm operator who wants to sell electricity in the \emph{day-ahead market}.
            In the day-ahead market, bids must be placed one day in advance, at noon, for each hour $h$ of the next day.
            Hence, the wind farm operator needs a \emph{forecasting model} for the next-day electricity generation $E_h$ to decide on the bidding amounts $y_h, h=0,\dots,23$.
            If a forecast is imprecise and the operator bids too much or too little electricity, they risk a penalty for the resulting shortfall or surplus, as, in both cases, the electricity grid’s supply-demand balance needs to be restored.
            Consequently, when forecast uncertainty is high, it may be better to abstain from bidding, resulting in a typical \emph{selective prediction} decision problem.
            In~\citet{bruninx_day-ahead_2025}, the optimization problem for the bidding amount $y_h^*$ for hour $h$ is defined as follows:
            \begin{align*}
                \max_{y_h} \quad & \mathbb{E}\left[ R_h^\text{DA} y_h + R_h^\text{B} (E_h - y_h) \right] \\
                \text{s.t.} \quad & \mathbb{E}\left[ (E_h - y_h)^2 \right] \leq \alpha, \\
                & 0 \leq y_h \leq \beta,
            \end{align*}
            where $\beta$ is the wind farm's capacity,
            $E_h$ is the realized electricity generation,
            $R_h^\text{DA}$ is the day-ahead market price (the unit payoff for the bid quantity $y_h$),
            and $R_h^\text{B}$ is the balancing market price at hour $h$ (the unit cost or reward for the deviation $E_h - y_h$).
            $\alpha$ is a pre-specified imbalance penalty termed \emph{risk certificate}.
            It can range from $\alpha_\text{min} = \text{Var}\left[E_h\right]$ to $\alpha_\text{max} = \text{Var}\left[E_h\right] + (\beta-\mathbb{E}\left[E_h\right])^2$.
            In their paper, the authors run experiments on several values of the \emph{relative} $\tilde{\alpha} = (\alpha - \alpha_\text{min})/(\alpha_\text{max} - \alpha_\text{min})$ and we settle on $\tilde{\alpha} = 1\%$ for our experiments.
            \citet{bruninx_day-ahead_2025} derive the optimal bidding amount $y_h^*$ for hour $h$ analytically. It is given as
            \begin{align*}
                y_h^* =
                \begin{cases}
                    \min\{ \mathbb{E}\left[ E_h \right] + \Delta_h; \beta \} & \text{if } \mathbb{E}\left[ R_h^\text{DA} \right] > \mathbb{E}\left[ R_h^\text{B} \right], \\
                    \max\{ \mathbb{E}\left[ E_h \right] - \Delta_h; 0 \} & \text{else }, \\
                \end{cases}
            \end{align*}
            where $\Delta_h = \sqrt{\alpha - \text{Var}\left[ E_h \right]}$.
            Hence, when the expected balancing price is larger than the expected day-ahead price and the uncertainty (variance) is too high, the optimal decision is abstainment.
            As in~\citet{bruninx_day-ahead_2025}, we use Belgian balancing price data from~\cite{elia_elia_2025} and day-ahead price and wind power data (\emph{Belwind Phase 1} wind farm) from~\cite{entso-e_2025}.
            The data from~\cite{elia_elia_2025} is licensed under CC BY 4.0.
            The data from~\cite{entso-e_2025} has no explicit license, but our use complies with their terms and conditions.
            We train on data from June to November 2024 to predict wind power generation, using the same suite of regression models as in our benchmark experiments.
            We then employ the trained models to estimate $\mathbb{E}[E_h]$ and $\mathrm{Var}[E_h]$ for each hour $h$ in December 2024, based on which we compute $y_h^*$ and the payoff.
            To quantify the variability of the metric--utility alignment scores, we bootstrap daily blocks (24 hours) of the 31 days of December, using $B = 100$ resamples.
            For each bootstrap replicate, we compute the metrics and utilities on the resampled set of days; the \emph{bidding utility} is the sum of payoffs over all hours in the replicate,
            \[
                \text{Bid-Util} = \sum_h \left( R_h^{\text{DA}}\, y_h^* + R_h^{\text{B}} (E_h - y_h^*) \right).
            \]
            We show the results in \cref{fig:bidding} and \cref{tab:bidding} and see that our PWU metrics exhibit the strongest utility-alignment.
            On the left of Figure~\ref{fig:cumsum_thresholds_ks}, we show the \emph{cumulative payoff} on the night of Dec 03 to Dec 04, 2024.
            Whenever a loss occurred (so payoff $<0$), we mark the hour with a cross.
            Abstainments are marked as circles.
            At 11 pm, for example, we can see that most models led to a risky trade, resulting in imbalance penalties, while LinReg correctly quantified too high uncertainty and abstained from trading.

\begin{figure}[ht]
 \vskip 0.2in
  \begin{center}
    \centerline{\includegraphics[width=.45\textwidth]{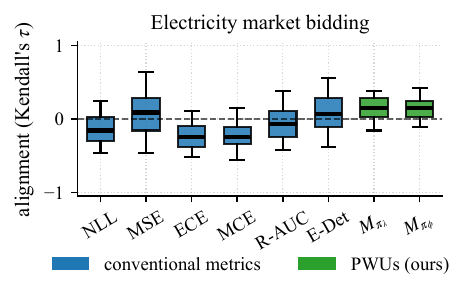}}
    \caption{
      Metric--utility alignment in the electricity market case study.
      The PWU metrics are the only metrics with stable positive bidding utility alignment.
    }
    \label{fig:bidding}
  \end{center}
  \vskip -0.2in
\end{figure}

        \paragraph{Credit approval}

            In our second case study, we consider a dataset of credit requests.
            The decision-maker (the financial institution) needs to decide for each customer whether to grant or decline a requested credit.
            The modeling task is probabilistic binary classification, predicting the probability of default ($y=1$).
            We follow the setup from \citet{bahnsen_example-dependent_2014}, where the cost of a false negative (credit granted that defaults) is 75\% of the credit amount, and the cost of a false positive (credit declined that would have been paid back) is the lost profit:
            \[
                C_{\text{FN},i} = Cl_i\cdot 0.75, \quad C_{\text{FP},i}=r_i +  C_{\text{FP}}^a,
            \]
            where $Cl_i$ is the customer's credit line, $r_i$ is the loss in profit by rejecting a good customer, and $-C_{\text{FP}}^a$ is the gain through giving out the loan to an alternative customer.
            $Cl_i$ and $r_i$ are not readily available from the data, but \citet{bahnsen_example-dependent_2014} provide a way to calculate these quantities in Appendix A of their paper.
            The resulting pointwise utility is
            \[
                -\sum_{i=1}^n y_i (1-c_{\text{FN},i})\boldsymbol{1}_{f_i \le \tau_i} + (1-y_i)c_{\text{FP},i}\boldsymbol{1}_{f_i>\tau_i},
            \]
            where we choose the Bayes-optimal $\tau_i = \tfrac{c_{\text{FP},i}}{c_{\text{FP},i} + c_{\text{FN},i}}$ for each customer separately.
            We show a histogram of the $\tau_i$ in the center of \cref{fig:cumsum_thresholds_ks}.
            The problem is similar to a binary decision, but the false negative and false positive costs are customer-dependent and highly variable.
            We use the same two datasets as in \citet{bahnsen_example-dependent_2014}: the PAKDD and Kaggle credit datasets \cite{pakdd_credit_2009,cukierski_givemesomecredit_2011}.
            The Kaggle dataset has no associated license, but our use respects corresponding competition rules.
            The PAKDD data is openly available through the Python package \texttt{costcla} \cite{bahnsen_costcla_2016} that uses a BSD 3-Clause License.
            We show the results on the left of \cref{fig:credit_p2p}, averaged over the two datasets, and see that the binary decision PWU has the strongest alignment with the credit approval utility.
            This is particularly remarkable, because we can see in the center of \cref{fig:cumsum_thresholds_ks} that the prior $\pi_c$ we chose for our PWU metric is misspecified for the PAKDD dataset.
            Nevertheless, the PWU metrics robustly achieve strong utility alignment.

        \paragraph{Peer-to-peer lending}

            In peer-to-peer (P2P) lending, the task is to lend money to a few candidates from a large pool of borrowers, given a fixed budget $B$.
            We follow the setup from \citet{byanjankar_datadriven_2021}, who use $B=15,\!000$ and the Lending Club dataset \cite{nathan_lending_2020} that uses a CC0 1.0 Universal license.
            The modeling task is probabilistic binary classification, predicting the probability of a good loan ($y=1$).
            In our implementation, we greedily select the loans with the highest predicted probabilities until the budget $B$ is exhausted.
            The utility is then computed as the return through the borrower for those loans that are paid back, minus the lost lent money for those loans that have not been paid back.
            To make the task nontrivial for our models, we considered only loans with grade $6$ or higher, corresponding to worse gradings.
            The task resembles top-$k$ selection, but the payoff is customer-dependent and highly variable.
            In addition, there is no fixed-$k$, but $k$ is implicitly determined through the fixed budget and the loan amounts of the chosen borrowers---we visualize the empirical $k/n$-distribution on our 500 different test sets on the right of \cref{fig:cumsum_thresholds_ks}.
            We show the results on the right of \cref{fig:credit_p2p}.
            We can see that the top-$k$ PWU has the second strongest alignment with the P2P lending utility---even though its prior is misspecified in light of the case study.

\begin{figure}[ht]
    \begin{center}
        \centerline{\includegraphics[width=.9\textwidth]{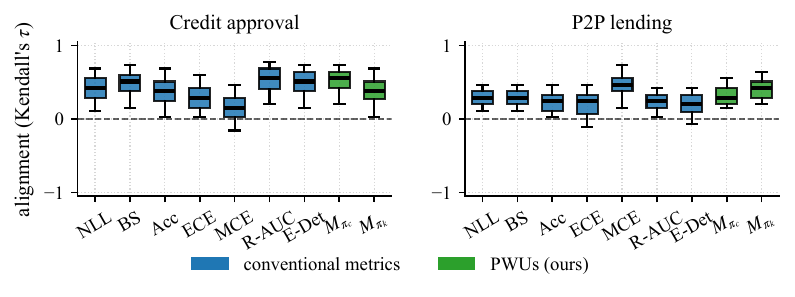}}
        \caption{
          Metric--utility alignment in the credit approval case study \emph{(left)} and the P2P lending case study \emph{(right)}.
          We can see that the PWU metrics exhibit the strongest \emph{(left)} and second strongest \emph{(right)} alignment with the utilities in these two case studies.
        }
        \label{fig:credit_p2p}
    \end{center}
    \vskip -0.2in
\end{figure}

    \subsection{Sensitivity analysis}\label{app:sens_analysis}

\begin{figure}[ht]
    \begin{center}
        \centerline{\includegraphics[width=.9\textwidth]{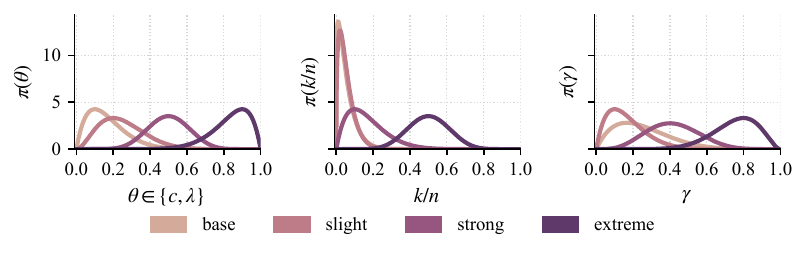}}
        \caption{
          Prior perturbations considered in our sensitivity analysis.
          The base priors anticipated by our PWU metrics are shown in the lightest color, while darker shades correspond to increasing levels of prior misspecification.
        }
        \label{fig:priors_perturbed}
    \end{center}
    \vskip -0.2in
\end{figure}

        When constructing a PWU metric for a given decision family, one must commit to a specific prior.
        In this section, we analyze the sensitivity of PWU metrics to \emph{prior misspecification}.
        We rerun our benchmark dataset experiments, but instead of measuring alignment to utilities $U_\theta$ originating from the same prior assumption $\pi_\theta$, we measure the alignment to utilities stemming from perturbed priors $\tilde{\pi}_\theta$, so that our PWU metric priors are misspecified.
        We show the choice of our perturbed priors $\tilde{\pi}_\theta$ in \cref{fig:priors_perturbed}, alongside the \emph{base prior} our PWU anticipates.
        We consider a \emph{slight}, \emph{strong}, and \emph{extreme} perturbation.
        In \cref{fig:sens_analysis_bc} and \ref{fig:sens_analysis_reg}, we show how the metric--utility alignments change across these perturbation levels in the binary classification and regression experiments respectively.

\begin{figure}[ht]
    \begin{center}
        \centerline{\includegraphics[width=.8\textwidth]{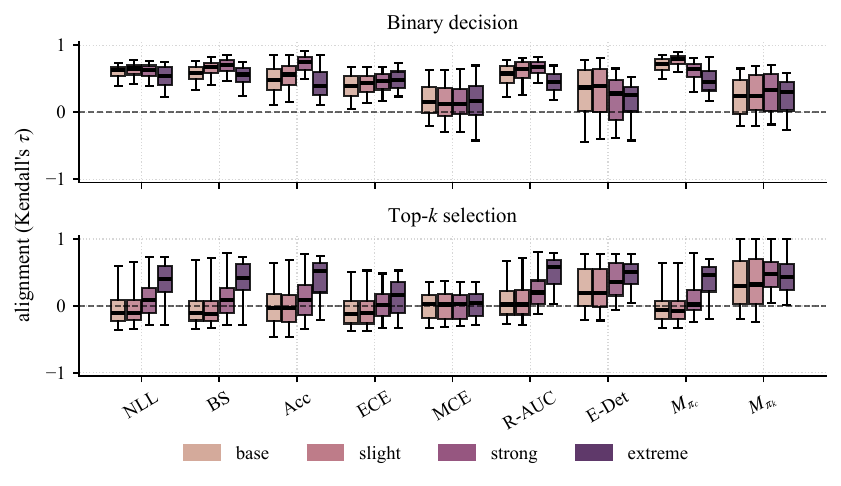}}
        \caption{
            Metric--utility alignment in binary classification under increasing levels of prior perturbation.
            The boxes corresponding to the base prior coincide with those shown in \cref{fig:benchmark}.
            Our PWU metrics exhibit strong robustness even under extreme prior misspecification.
        }
        \label{fig:sens_analysis_bc}
    \end{center}
    \vskip -0.2in
\end{figure}
        
        \paragraph{Binary decision}

            Our base prior for the binary decision problem has a mode at $10\%$, reflecting asymmetric costs in which false negatives are substantially more costly than false positives.
            The perturbed priors shift the mode to $20\%$ (slight), $50\%$ (strong), and $90\%$ (extreme).
            The metric $M_{\pi_c}$ is the PWU metric that is decision-aligned with respect to the binary decision problem under the base prior.
            We can see that a slight perturbation does not affect its utility-alignment, and the strong and extreme perturbation slightly worsens it.
            However, even after extreme perturbation, our PWU metric is still among the most aligned ones, demonstrating its robustness against prior-misspecification.
            Notably, even $M_{\pi_k}$ exhibits fair alignment that is entirely robust against different prior choices for $c$.
            Another interesting observation is the development of Acc.
            Remember that Acc is decision-aligned with respect to the binary decision problem for $\pi_c=\delta_{0.5}$, so we would expect a strong alignment for the Beta prior with mode at $0.5$---which is exactly what we can observe.

        \paragraph{Top-\texorpdfstring{$\boldsymbol k$}{k} selection (classification)}

            Our base prior for the top-$k$ selection problem has a mode at $1\%$, reflecting the task of selecting a small number of items from a large pool.
            The perturbed priors shift the mode to $2\%$ (slight), $10\%$ (strong), and $50\%$ (extreme).
            We first observe that increasing levels of perturbation improve the alignment of metrics that evaluate performance over the entire dataset.
            This behavior is expected, as larger values of $k$ imply that a greater fraction of instances contributes to the utility $U_k$.
            In particular, we expect this trend to be most pronounced for accuracy (Acc), since Acc is decision-aligned w.r.t.\ $U_k$ when $k=n$ (see \cref{sect:theory})---which is exactly what is observed.
            Turning to $M_{\pi_k}$, the PWU metric that is decision-aligned with respect to the top-$k$ selection problem under the base prior, we find that slight and strong perturbations even lead to a modest improvement in alignment.
            A plausible explanation is that for $k/n=1\%$, the selection task may be trivial on many benchmark datasets, with all evaluated models successfully selecting only positive instances.\footnote{
                We nevertheless argue that the base prior remains well motivated, as it reflects realistic application scenarios.
                Moreover, our case studies demonstrate that even smaller values of $k/n$ can yield non-trivial decision problems depending on the dataset and task.
            }
            As $k$ increases, the task becomes less trivial, enabling a more meaningful assessment of alignment.
            Even under extreme perturbation, however, the alignment of $M_{\pi_k}$ degrades only moderately, demonstrating the robustness of our PWU metric to prior misspecification.
            
\begin{figure}[ht]
    \begin{center}
        \centerline{\includegraphics[width=.8\textwidth]{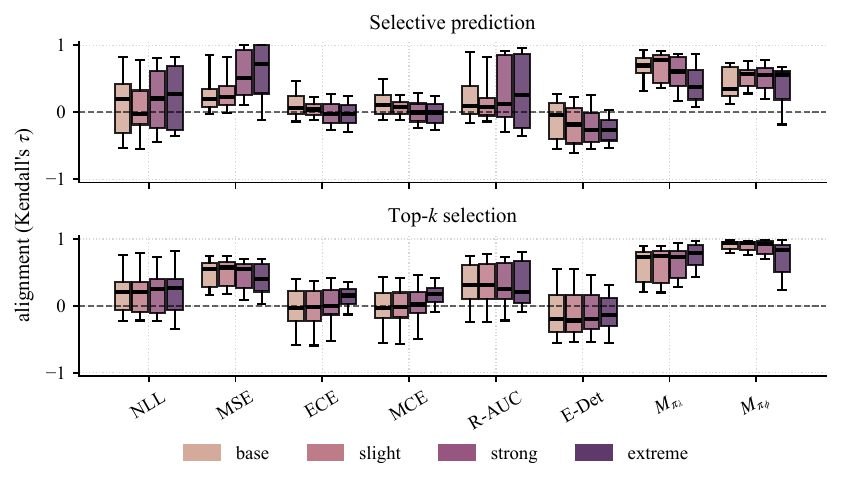}}
        \caption{
            Metric--utility alignment in regression under increasing levels of prior perturbation.
            The boxes corresponding to the base prior coincide with those shown in \cref{fig:benchmark}.
            Our PWU metrics exhibit strong robustness even under extreme prior misspecification.
        }
        \label{fig:sens_analysis_reg}
    \end{center}
    \vskip -0.2in
\end{figure}

        \paragraph{Selective prediction}

            Our base prior for selective prediction has a mode at 10\% (abstention is cheaper about 10\% of the time), while the perturbations consider the modes 20\% (slight), 50\% (strong), and 90\% (extreme).
            $M_{\pi_\lambda}$ is our PWU that is decision-aligned with respect to the selective prediction problem under the base prior.
            We can see that a slight perturbation does not affect its utility-alignment, and the strong and extreme perturbation slightly worsens it.
            However, even after extreme perturbation, our PWU metric is still among the most aligned ones, demonstrating its robustness against prior-misspecification.
            Notably, even $M_{\pi_\phi}$ exhibits fair alignment that is robust against different prior choices for $\lambda$.
            Another interesting observation is the development of the MSE.
            Remember that the MSE is decision-aligned for $\pi_\lambda=\delta_{\infty}$, so we would expect a strong alignment for the Beta prior with a high mode---which is exactly what we can observe.

        \paragraph{Top-\texorpdfstring{$\boldsymbol k$}{k} selection (regression)}

            Our base prior for the top-$k$ selection problem has a mode at $k/n = 1\%$ and a moderate risk-aversion parameter~$\gamma$.
            The perturbed priors consider $k/n = 2\%$ with low risk aversion (slight), $k/n = 10\%$ with strong risk aversion (strong), and $k/n = 50\%$ with extreme risk aversion (extreme).
            The metric $M_{\pi_\phi}$ is the PWU metric that is decision-aligned with respect to the top-$k$ selection problem under the base prior.
            Slight and strong perturbations have essentially no effect on its utility alignment, while extreme perturbation leads to a moderate degradation.
            Nevertheless, even under extreme perturbation, $M_{\pi_\phi}$ remains among the most utility-aligned metrics, demonstrating robustness to prior misspecification.
            Notably, even $M_{\pi_\lambda}$ exhibits strong alignment that is robust against different prior choices for $\phi$.

            Overall, these results show that PWU metrics are robust to prior misspecification, further underscoring their suitability in general-purpose UQ benchmarking.

            \begin{remark}[Utility misspecification robustness]
                The sensitivity analysis above studies the robustness of a PWU metric $M_\pi$ when the prior $\pi$ is misspecified \emph{within} its target utility family.
                A complementary form of misspecification, however, occurs when the entire utility family is misspecified.
                Our experiments already provide empirical evidence on this scenario, since we evaluate every applicable PWU against every utility family considered.
                Inspecting Figure~\ref{fig:benchmark}, we observe that most PWUs remain among the most decision-aligned metrics even when applied across utility families: $M_{\pi_k}$ exhibits strong alignment with the binary decision utility, and analogously $M_{\pi_\lambda}$ and $M_{\pi_\phi}$ transfer between selective prediction and top-$k$ selection in the regression setting.
                A plausible explanation is that all of our PWUs are proper scoring rules and thus a fair heuristic in any application considered.
                Combined with the prior-level sensitivity analysis, this suggests that PWU metrics degrade gracefully under both forms of misspecification, further supporting their use as default metrics in general-purpose UQ benchmarking.
            \end{remark}

    \subsection{Reproducibility}\label{app:reprod}

        Our codebase is provided at \url{https://github.com/fortuinlab/prior-weighted-utilities}.
        Essentials on the experiments are described in this manuscript.
        All further details can be directly extracted from the codebase.
        The code is readily executable and documented in its \texttt{README.md}.
        All experiments are \emph{deterministic} through fixed seeds.
        We ran all experiments on an internal high-performance compute cluster.
        For PWU evaluation we used CPU nodes (Intel Xeon Gold 6248R, 2 cores
        and 60~GB of memory per job), with each evaluation taking less than
        one minute.
        For model training we used GPU nodes, where each job was
        allocated one NVIDIA A100 (40~GB MIG slice), 16 CPU cores, and 60~GB
        of memory.
        Training all models in a single repeat--fold combination
        takes approximately 14--17 minutes for the binary classification,
        univariate regression, multiclass classification, and multivariate
        regression benchmarks, and 6--16 minutes for the credit approval and
        P2P lending case studies.
        The electricity market case study has no repeat--fold structure and takes 59 minutes for a single full run.
        Reproducing the full benchmark suite (10 models for binary
        classification and univariate regression, 5 models for the multiclass
        and multivariate settings, each across 5 datasets and 500
        repeat--fold combinations) and the case studies therefore requires
        on the order of 2{,}500 GPU-hours in total.

\begin{table}[ht]
    \caption{
        Metric--utility alignment (Kendall's $\tau$ of model rankings) over 100 repeats in the binary decision problem.
        We mark the highest median in \textbf{bold}.
        $M_{\pi_c}$ aligns best with respect to the binary decision utility family on all datasets.
    }
    \vspace{0.7em}
    \label{tab:benchmark_bc_bd}
    \centering
                \begin{tabular}{ccccccccc}
                    \toprule
                    NLL & BS & Acc & ECE & MCE & R-AUC & E-Det & $M_{\pi_c}$ & $M_{\pi_k}$\\
                    \midrule
                    \multicolumn{9}{c}{bank marketing}\\
                    \midrule
                    0.69 & 0.72 & 0.62 & 0.40 & 0.14 & 0.74 & 0.76 & \textbf{0.79} & 0.58 \\
                    \tiny{{[}0.61, 0.78{]}} & \tiny{{[}0.64, 0.81{]}} & \tiny{{[}0.46, 0.73{]}} & \tiny{{[}0.17, 0.57{]}} & \tiny{{[}-0.11, 0.39{]}} & \tiny{{[}0.64, 0.80{]}} & \tiny{{[}0.63, 0.85{]}} & \tiny{{[}0.69, 0.85{]}} & \tiny{{[}0.41, 0.72{]}} \\
                    \midrule
                    \multicolumn{9}{c}{heart disease}\\
                    \midrule
                    0.50 & 0.45 & 0.24 & 0.37 & 0.18 & 0.34 & 0.20 & \textbf{0.58} & 0.09 \\
                    \tiny{{[}0.32, 0.64{]}} & \tiny{{[}0.23, 0.62{]}} & \tiny{{[}0.02, 0.49{]}} & \tiny{{[}0.12, 0.57{]}} & \tiny{{[}-0.09, 0.56{]}} & \tiny{{[}0.11, 0.52{]}} & \tiny{{[}-0.17, 0.40{]}} & \tiny{{[}0.41, 0.70{]}} & \tiny{{[}-0.16, 0.34{]}} \\
                    \midrule
                    \multicolumn{9}{c}{ionosphere}\\
                    \midrule
                    0.64 & 0.57 & 0.42 & 0.16 & -0.03 & 0.65 & 0.54 & \textbf{0.73} & 0.46 \\
                    \tiny{{[}0.45, 0.73{]}} & \tiny{{[}0.45, 0.70{]}} & \tiny{{[}0.23, 0.59{]}} & \tiny{{[}-0.03, 0.34{]}} & \tiny{{[}-0.30, 0.29{]}} & \tiny{{[}0.56, 0.76{]}} & \tiny{{[}0.36, 0.72{]}} & \tiny{{[}0.66, 0.84{]}} & \tiny{{[}0.28, 0.63{]}} \\
                    \midrule
                    \multicolumn{9}{c}{mushroom}\\
                    \midrule
                    0.65 & 0.65 & 0.71 & 0.61 & 0.57 & 0.56 & -0.38 & \textbf{0.80} & -0.13 \\
                    \tiny{{[}0.54, 0.75{]}} & \tiny{{[}0.55, 0.79{]}} & \tiny{{[}0.61, 0.94{]}} & \tiny{{[}0.52, 0.74{]}} & \tiny{{[}0.46, 0.72{]}} & \tiny{{[}0.47, 0.72{]}} & \tiny{{[}-0.62, -0.24{]}} & \tiny{{[}0.66, 0.86{]}} & \tiny{{[}-0.43, -0.02{]}} \\
                    \midrule
                    \multicolumn{9}{c}{sonar}\\
                    \midrule
                    0.57 & 0.53 & 0.40 & 0.40 & 0.04 & 0.44 & 0.33 & \textbf{0.64} & 0.21 \\
                    \tiny{{[}0.41, 0.70{]}} & \tiny{{[}0.32, 0.65{]}} & \tiny{{[}0.18, 0.61{]}} & \tiny{{[}0.15, 0.65{]}} & \tiny{{[}-0.18, 0.31{]}} & \tiny{{[}0.24, 0.60{]}} & \tiny{{[}0.03, 0.61{]}} & \tiny{{[}0.43, 0.79{]}} & \tiny{{[}-0.02, 0.48{]}} \\
                    \bottomrule
                \end{tabular}
\end{table}

\begin{table}[ht]
    \caption{
        Metric--utility alignment (Kendall's $\tau$ of model rankings) over 100 repeats in the top-$k$ decision problem (binary classification).
        We mark the highest median in \textbf{bold}.
        $M_{\pi_k}$ aligns best with respect to the top-$k$ utility family on most datasets.
    }
    \vspace{0.7em}
    \label{tab:benchmark_bc_topk}
    \centering
                \begin{tabular}{ccccccccc}
                    \toprule
                    NLL & BS & Acc & ECE & MCE & R-AUC & E-Det & $M_{\pi_c}$ & $M_{\pi_k}$\\
                    \midrule
                    \multicolumn{9}{c}{bank}\\
                    \midrule
                    0.55 & \textbf{0.63} & 0.59 & 0.43 & 0.16 & \textbf{0.63} & 0.53 & 0.57 & 0.62 \\
                    \tiny{{[}0.40, 0.68{]}} & \tiny{{[}0.50, 0.75{]}} & \tiny{{[}0.44, 0.72{]}} & \tiny{{[}0.16, 0.57{]}} & \tiny{{[}-0.08, 0.41{]}} & \tiny{{[}0.49, 0.76{]}} & \tiny{{[}0.35, 0.70{]}} & \tiny{{[}0.44, 0.71{]}} & \tiny{{[}0.50, 0.72{]}} \\
                    \midrule
                    \multicolumn{9}{c}{heartdisease}\\
                    \midrule
                    -0.17 & -0.16 & -0.01 & -0.16 & 0.15 & -0.02 & 0.00 & -0.15 & \textbf{0.24} \\
                    \tiny{{[}-0.38, 0.13{]}} & \tiny{{[}-0.39, 0.10{]}} & \tiny{{[}-0.29, 0.34{]}} & \tiny{{[}-0.38, 0.08{]}} & \tiny{{[}-0.15, 0.46{]}} & \tiny{{[}-0.29, 0.32{]}} & \tiny{{[}-0.33, 0.38{]}} & \tiny{{[}-0.36, 0.12{]}} & \tiny{{[}-0.07, 0.49{]}} \\
                    \midrule
                    \multicolumn{9}{c}{ionosphere}\\
                    \midrule
                    -0.12 & -0.15 & -0.33 & -0.10 & -0.02 & -0.16 & \textbf{0.02} & -0.16 & -0.11 \\
                    \tiny{{[}-0.29, 0.07{]}} & \tiny{{[}-0.33, 0.07{]}} & \tiny{{[}-0.51, -0.03{]}} & \tiny{{[}-0.36, 0.09{]}} & \tiny{{[}-0.33, 0.25{]}} & \tiny{{[}-0.29, 0.01{]}} & \tiny{{[}-0.25, 0.32{]}} & \tiny{{[}-0.29, 0.03{]}} & \tiny{{[}-0.31, 0.18{]}} \\
                    \midrule
                    \multicolumn{9}{c}{mushroom}\\
                    \midrule
                    -0.24 & -0.24 & -0.02 & -0.29 & -0.24 & 0.09 & 0.69 & -0.11 & \textbf{1.00} \\
                    \tiny{{[}-0.42, -0.11{]}} & \tiny{{[}-0.42, -0.16{]}} & \tiny{{[}-0.42, 0.11{]}} & \tiny{{[}-0.42, -0.16{]}} & \tiny{{[}-0.42, -0.16{]}} & \tiny{{[}-0.29, 0.24{]}} & \tiny{{[}0.51, 0.78{]}} & \tiny{{[}-0.38, -0.02{]}} & \tiny{{[}1.00, 1.00{]}} \\
                    \midrule
                    \multicolumn{9}{c}{sonar}\\
                    \midrule
                    -0.07 & -0.06 & -0.07 & -0.14 & 0.05 & -0.04 & 0.04 & -0.08 & \textbf{0.09} \\
                    \tiny{{[}-0.29, 0.12{]}} & \tiny{{[}-0.22, 0.14{]}} & \tiny{{[}-0.24, 0.11{]}} & \tiny{{[}-0.39, 0.11{]}} & \tiny{{[}-0.20, 0.26{]}} & \tiny{{[}-0.20, 0.14{]}} & \tiny{{[}-0.22, 0.25{]}} & \tiny{{[}-0.31, 0.19{]}} & \tiny{{[}-0.18, 0.31{]}} \\
                    \bottomrule
                \end{tabular}
\end{table}

\begin{table}[ht]
    \caption{
        Metric--utility alignment (Kendall's $\tau$ of model rankings) over 100 repeats in the selective decision problem.
        We mark the highest median in \textbf{bold}.
        $M_{\pi_\lambda}$ aligns best with respect to the selective decision utility family on all datasets.
    }
    \vspace{0.7em}
    \label{tab:benchmark_reg_selpred}
    \centering
                \begin{tabular}{cccccccc}
                    \toprule
                    NLL & MSE & ECE & MCE & R-AUC & E-Det & $M_{\pi_\lambda}$ & $M_{\pi_\phi}$\\
                    \midrule
                    \multicolumn{8}{c}{air quality}\\
                    \midrule
                    0.39 & 0.16 & 0.44 & 0.47 & 0.36 & 0.21 & \textbf{0.67} & 0.16 \\
                    \tiny{{[}0.30, 0.47{]}} & \tiny{{[}0.02, 0.26{]}} & \tiny{{[}0.36, 0.51{]}} & \tiny{{[}0.39, 0.56{]}} & \tiny{{[}0.27, 0.43{]}} & \tiny{{[}0.10, 0.31{]}} & \tiny{{[}0.63, 0.72{]}} & \tiny{{[}0.07, 0.24{]}} \\
                    \midrule
                    \multicolumn{8}{c}{auto mpg}\\
                    \midrule
                    0.20 & 0.27 & 0.04 & 0.08 & 0.08 & -0.42 & \textbf{0.79} & 0.69 \\
                    \tiny{{[}0.00, 0.31{]}} & \tiny{{[}0.11, 0.44{]}} & \tiny{{[}-0.12, 0.19{]}} & \tiny{{[}-0.06, 0.24{]}} & \tiny{{[}-0.04, 0.24{]}} & \tiny{{[}-0.63, -0.25{]}} & \tiny{{[}0.72, 0.84{]}} & \tiny{{[}0.57, 0.78{]}} \\
                    \midrule
                    \multicolumn{8}{c}{energy efficiency}\\
                    \midrule
                    0.81 & 0.84 & 0.19 & 0.22 & 0.88 & -0.43 & \textbf{0.92} & 0.69 \\
                    \tiny{{[}0.72, 0.88{]}} & \tiny{{[}0.79, 0.87{]}} & \tiny{{[}0.08, 0.31{]}} & \tiny{{[}0.13, 0.33{]}} & \tiny{{[}0.82, 0.92{]}} & \tiny{{[}-0.56, -0.28{]}} & \tiny{{[}0.87, 0.93{]}} & \tiny{{[}0.64, 0.75{]}} \\
                    \midrule
                    \multicolumn{8}{c}{power plant}\\
                    \midrule
                    -0.27 & 0.03 & -0.02 & -0.02 & -0.01 & -0.03 & \textbf{0.33} & 0.28 \\
                    \tiny{{[}-0.36, -0.16{]}} & \tiny{{[}-0.08, 0.10{]}} & \tiny{{[}-0.09, 0.10{]}} & \tiny{{[}-0.08, 0.11{]}} & \tiny{{[}-0.10, 0.06{]}} & \tiny{{[}-0.17, 0.13{]}} & \tiny{{[}0.27, 0.42{]}} & \tiny{{[}0.22, 0.38{]}} \\
                    \midrule
                    \multicolumn{8}{c}{wine quality}\\
                    \midrule
                    -0.48 & 0.20 & -0.08 & -0.04 & -0.10 & 0.09 & \textbf{0.66} & 0.32 \\
                    \tiny{{[}-0.61, -0.35{]}} & \tiny{{[}0.02, 0.34{]}} & \tiny{{[}-0.20, 0.06{]}} & \tiny{{[}-0.16, 0.09{]}} & \tiny{{[}-0.25, 0.05{]}} & \tiny{{[}-0.05, 0.27{]}} & \tiny{{[}0.51, 0.76{]}} & \tiny{{[}0.18, 0.44{]}} \\
                    \bottomrule
                \end{tabular}
\end{table}

\begin{table}[ht]
    \caption{
        Metric--utility alignment (Kendall's $\tau$ of model rankings) over 100 repeats in the top-$k$ decision problem (regression).
        We mark the highest median in \textbf{bold}.
        $M_{\pi_\phi}$ aligns best with respect to the selective decision utility family on all datasets.
    }
    \vspace{0.7em}
    \label{tab:benchmark_reg_topk}
    \centering
                \begin{tabular}{cccccccc}
                    \toprule
                    NLL & MSE & ECE & MCE & R-AUC & E-Det & $M_{\pi_\lambda}$ & $M_{\pi_\phi}$\\
                    \midrule
                    \multicolumn{8}{c}{air quality}\\
                    \midrule
                    -0.18 & 0.24 & -0.19 & -0.14 & -0.19 & -0.42 & 0.24 & \textbf{0.81} \\
                    \tiny{{[}-0.29, -0.09{]}} & \tiny{{[}0.14, 0.35{]}} & \tiny{{[}-0.26, -0.11{]}} & \tiny{{[}-0.23, -0.06{]}} & \tiny{{[}-0.32, -0.09{]}} & \tiny{{[}-0.50, -0.26{]}} & \tiny{{[}0.13, 0.32{]}} & \tiny{{[}0.78, 0.85{]}} \\
                    \midrule
                    \multicolumn{8}{c}{auto mpg}\\
                    \midrule
                    0.24 & 0.28 & -0.03 & -0.01 & 0.16 & -0.47 & 0.76 & \textbf{0.96} \\
                    \tiny{{[}0.09, 0.41{]}} & \tiny{{[}0.14, 0.48{]}} & \tiny{{[}-0.19, 0.19{]}} & \tiny{{[}-0.18, 0.22{]}} & \tiny{{[}0.01, 0.29{]}} & \tiny{{[}-0.68, -0.28{]}} & \tiny{{[}0.63, 0.85{]}} & \tiny{{[}0.94, 0.99{]}} \\
                    \midrule
                    \multicolumn{8}{c}{energy efficiency}\\
                    \midrule
                    0.72 & 0.62 & 0.34 & 0.40 & 0.72 & -0.21 & 0.78 & \textbf{0.94} \\
                    \tiny{{[}0.62, 0.82{]}} & \tiny{{[}0.55, 0.70{]}} & \tiny{{[}0.24, 0.45{]}} & \tiny{{[}0.29, 0.50{]}} & \tiny{{[}0.66, 0.79{]}} & \tiny{{[}-0.32, -0.04{]}} & \tiny{{[}0.72, 0.83{]}} & \tiny{{[}0.92, 0.96{]}} \\
                    \midrule
                    \multicolumn{8}{c}{power plant}\\
                    \midrule
                    0.31 & 0.57 & 0.20 & 0.13 & 0.58 & 0.08 & 0.87 & \textbf{0.96} \\
                    \tiny{{[}0.16, 0.39{]}} & \tiny{{[}0.45, 0.66{]}} & \tiny{{[}0.08, 0.29{]}} & \tiny{{[}0.00, 0.24{]}} & \tiny{{[}0.48, 0.66{]}} & \tiny{{[}-0.15, 0.30{]}} & \tiny{{[}0.76, 0.92{]}} & \tiny{{[}0.93, 0.97{]}} \\
                    \midrule
                    \multicolumn{8}{c}{wine quality}\\
                    \midrule
                    -0.01 & 0.70 & -0.53 & -0.52 & 0.32 & 0.49 & 0.41 & \textbf{0.87} \\
                    \tiny{{[}-0.20, 0.13{]}} & \tiny{{[}0.53, 0.83{]}} & \tiny{{[}-0.65, -0.35{]}} & \tiny{{[}-0.62, -0.33{]}} & \tiny{{[}0.16, 0.44{]}} & \tiny{{[}0.32, 0.65{]}} & \tiny{{[}0.26, 0.52{]}} & \tiny{{[}0.76, 0.94{]}} \\
                    \bottomrule
                \end{tabular}
\end{table}

\section{Multiclass and Multivariate Experiments}\label{app:multi_exp}

    We now extend our binary and univariate experiments from \cref{sect:benchmarks} to the multiclass and multivariate setting.
    We consider the metrics and utilities defined in \cref{app:multi}.
    For the experimental evaluation of the multiclass-decision PWU, we use the same $\text{Beta}(2, 10)$ prior on $c$ as in the binary case (Appendix~\ref{app:priors}), with $j$ sampled uniformly over $[K]$.
    For the experimental evaluation of the multivariate selective-prediction PWU, we use the same $\text{Beta}(2, 10)$ prior on the dimensionless abstention-cost factor $\lambda$ as in the univariate case (Appendix~\ref{app:priors}), with $\lambda$ rescaled per dimension by the empirical variance inside the utility.
    As probabilistic ML models, we again use LogReg, RFs, GB, MLPs, and TabPFN for classification and LinReg, RFs, NGBoost, MLPs, and deep ensembles for regression.
    We use covertype~\cite{blackard_covertype_1998}, dry bean~\cite{koklu_dry_2020}, iris~\cite{fisher_iris_1936}, pen-based digits~\cite{alpaydin_pen-based_1996}, and wine~\cite{aeberhard_wine_1992} as classification datasets and air quality~\cite{vito_air_2008}, energy efficiency~\cite{tsanas_energy_2012}, solar flare~\cite{bradshaw_solar_1989}, SGEMM~\cite{paredes_sgemm_2017}, and parkinsons telemonitoring~\cite{tsanas_parkinsons_2009} as regression datasets.
    All datasets use a CC BY 4.0 license.
    We can see the results in \cref{fig:benchmark_multi}.
    In classification, the PWU metric exhibits the strongest utility alignment alongside the BS.
    In regression, the ES has a stronger median alignment, but a significantly lower, even negative, 5\% quantile, while the PWU metric is the only metric with consistent positive utility alignment.

\begin{figure*}[ht]
  \begin{center}
    \centerline{\includegraphics[width=.9\linewidth]{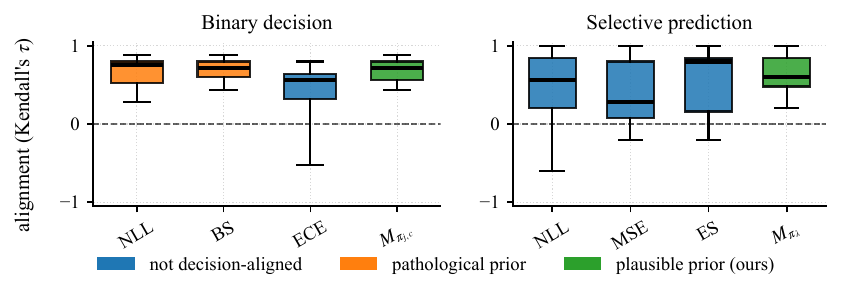}}
    \caption{
      Metric--utility alignment in classification \emph{(left)} and regression \emph{(right)}, averaged over our five datasets.
      The coloring corresponds to our theoretical findings from Appendix \ref{app:multi}.
      In classification, the PWU metric exhibits the strongest utility alignment alongside the BS.
      In regression, the PWU metric is the only metric with consistent positive utility alignment.
    }
    \label{fig:benchmark_multi}
  \end{center}
  \vskip -0.2in
\end{figure*}

\section{Recommendations for UQ benchmarking}\label{app:recomm}

    Our findings suggest a concrete revision to the standard UQ evaluation
    protocol.
    We do not advocate replacing NLL, BS, or ECE---these metrics
    carry useful auxiliary information about likelihood and calibration---but we do advocate against treating them as the
    primary evidence that a UQ method is ``good.'' Specifically, we
    recommend that papers proposing new UQ methods:
    \begin{enumerate}
        \item \textbf{Report PWU metrics for at least two distinct decision
        families} that plausibly reflect deployment use cases, e.g., a
        cost-sensitive binary decision PWU and a selective-prediction PWU.
        \item \textbf{State the prior $\pi$ explicitly}, following our recommendations in Appendix~\ref{app:priors}:
        Priors should be elicited from the deployment context where
        possible and from generic plausible defaults (such as the Beta
        priors we use). Hidden
        priors---those implicit in NLL, BS, or accuracy---are exactly what
        our analysis shows to be problematic.
        \item \textbf{Report conventional metrics alongside, but interpret
        them in light of their implicit priors} (Table~\ref{tab:theory_summary}).
        NLL improvements at the tails of $c$ may be irrelevant for a
        deployment in which the cost ratio is plausibly bounded away from
        $0$ and $1$; ECE improvements may not translate into utility gains
        at all.
    \end{enumerate}
    This is a modest amount of additional reporting---a few extra columns
    in a results table---but it shifts UQ benchmarking from implicitly
    assuming a single pathological prior to explicitly evaluating against a
    small panel of plausible decision contexts.



\end{document}